\documentclass[mnsc,nonblindrev]{informs3-hide-journal} 

\OneAndAHalfSpacedXI

\usepackage{natbib}
 \bibpunct[, ]{(}{)}{,}{a}{}{,}%
 \def\bibfont{\small}%
\TheoremsNumberedThrough     

\EquationsNumberedThrough    

\MANUSCRIPTNO{} 

\usepackage{microtype}
\usepackage{graphicx}
\usepackage{subcaption}
\usepackage{booktabs} 

\usepackage{hyperref}

\usepackage{amsmath}
\usepackage{amssymb}
\usepackage{mathtools}
\usepackage{graphicx}    
\usepackage{microtype}
\usepackage{booktabs}
\usepackage{algorithm}
\usepackage{algorithmic}
\usepackage{natbib}
\usepackage{xcolor}
\usepackage{enumitem}
\usepackage{multirow}
\usepackage{tabularx}
\usepackage{tikz}
\usetikzlibrary{arrows.meta,positioning,fit,calc,shapes.misc}
\usepackage{bm}

\usepackage[capitalize,noabbrev]{cleveref}

\setlist{nosep,leftmargin=*}

\hypersetup{
    colorlinks=true,
    linkcolor=blue!70!black,
    citecolor=blue!70!black,
    urlcolor=blue!70!black
}

\newcommand{\Cov}{\mathrm{Cov}}

\newcommand{\bU}{\mathbf{U}}
\newcommand{\bx}{\mathbf{x}}

\newcommand{\R}{\mathbb{R}}
\newcommand{\E}{\mathbb{E}}
\newcommand{\Pe}{\mathbb{P}}

\newcommand{\Ivec}{\mathbf{I}}
\newcommand{\Xvec}{\mathbf{X}}
\newcommand{\xvec}{\mathbf{x}}

\newcommand{\fvec}{\mathbf{f}}
\newcommand{\zvec}{\mathbf{z}}
\newcommand{\yvec}{\mathbf{y}}
\newcommand{\gvec}{\mathbf{g}}
\newcommand{\Jvec}{\mathbf{J}}
\newcommand{\betavec}{\bm{\beta}}
\newcommand{\Sigmavec}{\bm{\Sigma}}

\newcommand{\Psivec}{\bm{\Psi}}
\newcommand{\tauvec}{\bm{\tau}}
\newcommand{\thetavec}{\bm{\theta}}

\newcommand{\zetavec}{\bm{\zeta}}
\newcommand{\pivec}{\bm{\pi}}

\newcommand{\mB}{\mathcal{B}}

\newcommand{\mD}{\mathcal{D}}

\newcommand{\mX}{\mathcal{X}}

\newcommand{\Gn}{\mathbb{G}_n}

\renewcommand{\Pr}{{\mathrm{P}}}

\renewcommand{\hat}{\widehat}

\renewcommand{\Pr}{{\mathrm{P}}}

\renewcommand{\hat}{\widehat}
\renewcommand{\leq}{\leqslant}
\renewcommand{\geq}{\geqslant}

\newcommand{\supp}{\mathrm{supp}}

\usepackage[textsize=tiny]{todonotes}

\renewcommand{\arraystretch}{0.92} 
\usepackage{algorithmic}

\newcommand{\mx}[1]{\textcolor{black}{#1}}

\newcommand{\cl}[1]{\textcolor{black}{#1}}
\newcommand{\rev}[1]{\textcolor{black}{#1}}

\begin{document}


\RUNAUTHOR{Lu et al.}

\RUNTITLE{Generative Augmented Inference}
\RUNTITLE{Generative Augmented Inference}

\TITLE{Generative Augmented Inference of\\LLM-generated Data for Market Research:\\Theory and Empirical Evidence}

\ARTICLEAUTHORS{%
\AUTHOR{Cheng Lu}
\AFF{Olin Business School, Washington University in St.\ Louis, USA, \EMAIL{cheng.lu@wustl.edu}}
\AUTHOR{Mengxin Wang}
\AFF{Naveen Jindal School of Management, University of Texas at Dallas, USA, \EMAIL{mengxin.wang@utdallas.edu}}
\AUTHOR{Dennis J.\ Zhang}
\AFF{Olin Business School, Washington University in St.\ Louis, USA, \EMAIL{denniszhang@wustl.edu}}
\AUTHOR{Heng Zhang}
\AFF{W.\ P.\ Carey School of Business, Arizona State University, USA, \EMAIL{hengzhang24@asu.edu}}
} 

\ABSTRACT{%
Marketing research often relies on parameters estimated from costly human-generated data, such as conjoint survey responses, purchase decisions, and field experiment outcomes. Recent advances in large language models (LLMs) and other AI systems offer inexpensive auxiliary data, but introduce a new challenge: AI outputs are not direct observations of the target outcomes, but could involve high-dimensional representations with complex and unknown relationships to human labels. Conventional methods leverage AI predictions as direct proxies for true labels, which can be inefficient or unreliable when this relationship is weak or misspecified. We propose Generative Augmented Inference (GAI), a general framework that incorporates AI-generated outputs as informative features for estimating models of human-labeled outcomes. GAI uses an orthogonal moment construction that enables consistent estimation and valid inference with a flexible, nonparametric relationship between LLM-generated outputs and human labels. We establish asymptotic normality and a key dominance result: under random labeling, GAI is optimal within a unified class of debiased estimators—including human-data-only estimators and state-of-the-art debiasing methods—and delivers strict improvements under a mild informativeness condition. Even when the labeled sample is not representative of the target population, an extended variant of GAI still dominates the weighted human-data-only estimator. Empirically, GAI outperforms benchmarks across diverse marketing research settings. In a conjoint analysis, it halves estimation error and reduces human labeling requirements by over 75\%. In a pricing study, it consistently outperforms alternative estimators when all methods receive identical auxiliary inputs. In a health insurance study, it saves over 90\% of labels while preserving decision accuracy.
}

\KEYWORDS{Large Language Model, Data Augmentation, Causal Inference, Demand Estimation}

\maketitle
%

\section{Introduction}
\label{sec:intro}

Marketing research across virtually every subfield relies on human-generated data that is expensive, time-consuming, or difficult to collect at scale. Estimating consumer preferences requires conjoint surveys costing \$5--20 per respondent \citep{green1990conjoint, bradlow2004modeling}; calibrating demand and choice models depends on transaction records that capture realized demand \citep{guadagni1983logit, berry1993automobile}; learning how customers respond heterogeneously to marketing policies typically requires large-scale customer-level data \citep{ban2021personalized, chen2022privacy, lu2025optimizing}; and measuring treatment effects demands carefully designed field experiments \citep{sahni2016advertising, huang2020social}. In each case, the cost of obtaining reliable human-generated labels fundamentally constrains sample sizes and, consequently, the precision of parameter estimates that inform marketing and business decisions.

The emergence of large language models (LLMs) and AI-based generation systems offers a potential solution to this data scarcity problem \citep{brown2020language, openai2023gpt4}. AI systems can generate annotations, simulate survey responses, and produce predictions at costs orders of magnitude lower than human labeling---often \$0.01--0.10 per observation. The digital twin paradigm, which uses LLM-based simulations of individual decision-makers to generate synthetic behavioral data, exemplifies this opportunity \citep{toubia2025database}. If AI-generated data could reliably augment scarce human labels, researchers could achieve the statistical precision of large samples while collecting only a fraction of the human-labeled observations.

However, realizing these efficiency gains while maintaining valid statistical inference presents a fundamental challenge for two reasons. First, when AI outputs are intended as noisy substitutes for human labels, they can differ systematically from human judgments in complex, context-dependent ways. In our retail pricing application, LLM-based digital twins \citep{toubia2025database} predict purchases at a 30\% rate when the true rate is 44\%---a substantial and systematic bias. 
Naively pooling such AI-generated data with human labels would introduce bias that could overwhelm any efficiency gains, resulting in misleading parameter estimates and invalid confidence intervals. Second, and more importantly, AI outputs can often be structured auxiliary representations rather than just noisy labels---they may include reasoning traces, confidence scores, persona descriptions, or high-dimensional embeddings that are categorically different from the outcome of interest. This distinction fundamentally reframes the statistical problem: rather than correcting inaccurate predictions, the goal is to leverage auxiliary information efficiently for valid inference.


Therefore, the core methodological question is: \emph{How can AI-generated signals be incorporated into statistical estimation without requiring them to serve as accurate surrogates for the outcome?} Existing approaches to AI-augmented inference, most notably Prediction-Powered Inference (PPI; \citealp{angelopoulos2023prediction, angelopoulos2023ppi++}), utilize AI-predicted labels as noisy proxies of the true outcome. Unlike the ``proxy" view, modern generative AI systems can produce outputs that are categorically different from the outcomes researchers seek to estimate. This highlights the need for an inferential framework that does not require AI outputs to function as surrogate labels. 

We propose \emph{Generative Augmented Inference} (GAI), a framework that resolves this tension through a conceptual shift: rather than treating AI outputs as proxies for outcomes, we treat them as \emph{informative features} that help predict outcomes (\cref{fig:method}). This distinction is subtle but consequential. Under the proxy view, AI outputs must approximate human labels to be useful; under the feature view, AI outputs need only be \emph{informative} of human labels to contribute information.
\begin{figure*}[t]
  \centering
  \includegraphics[width=0.8\textwidth]{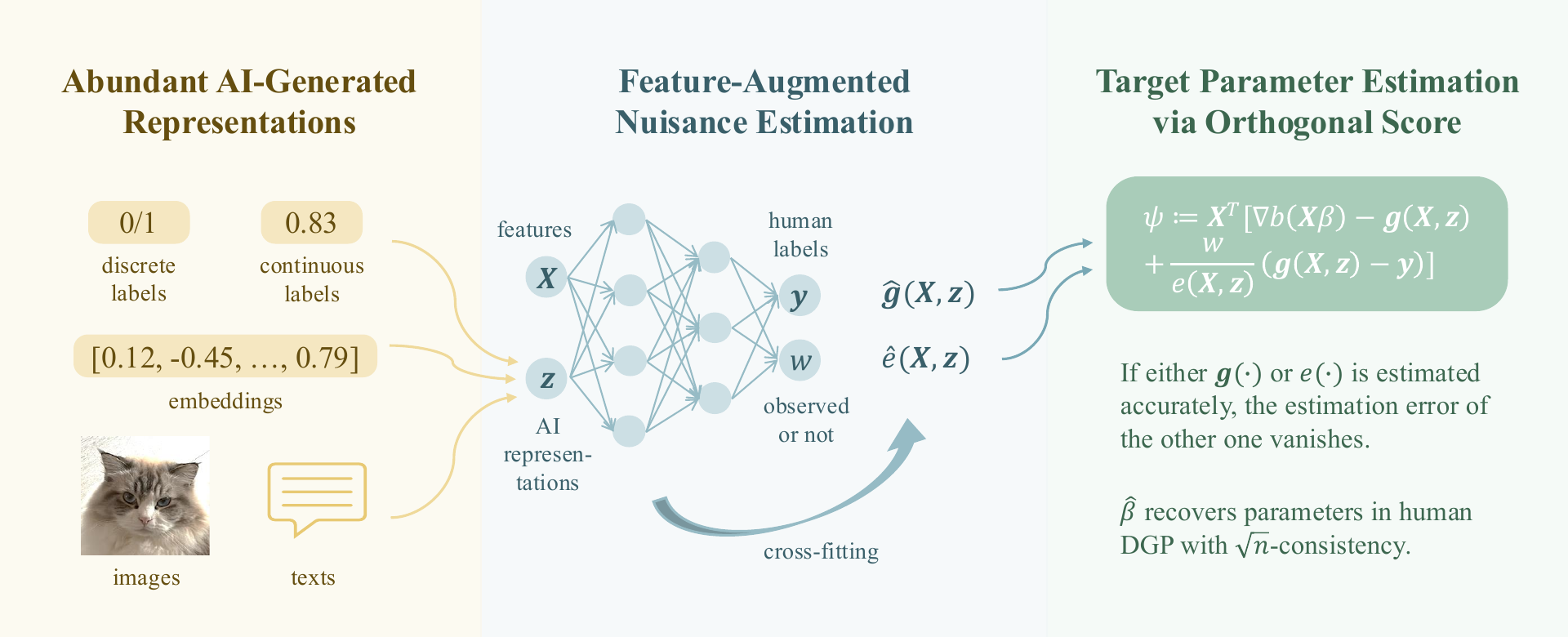}
  \caption{GAI incorporates AI outputs as auxiliary features rather than surrogate labels. In contrast to PPI, which applies corrections at the loss-function level and relies on AI predictions approximating the outcome, GAI embeds auxiliary signals within a Neyman-orthogonal score function. This construction allows the estimator to leverage auxiliary data for bias correction and efficiency gains, even when AI representations are biased, or weakly informative.}
  \label{fig:method}
\end{figure*}
To illustrate, consider an LLM predicting whether a consumer will purchase a product. Under the proxy view, its prediction substitutes for the true outcome, and the problem is framed as correcting the bias. Under the feature view, even a biased prediction, together with other AI outputs, can still be useful if it helps predict actual purchases: a predicted “no” may lower the expected purchase probability, whereas a predicted “yes” may raise it. GAI exploits the representations as features rather than requiring the predictions to be accurate proxies of human labels.

Concretely, consider estimating the parameters of a generalized linear model (GLM) from a dataset where human labels are observed for only a subset of observations, but AI-generated auxiliary information is available for all observations. The auxiliary information may be an LLM's binary purchase prediction, a high-dimensional embedding, unstructured texts, or any other AI-generated signal---its structure is left deliberately unrestricted. GAI proceeds in two steps:
\begin{enumerate}
    \item \textbf{Estimate nuisance functions.} Using the labeled subset, we estimate two functions: (i) the \emph{outcome prediction function}---the best prediction of human labels given both covariates and AI outputs, capturing whatever predictive information the AI signal contains about the outcome; and (ii) the \emph{propensity score}---the probability that an observation receives a human label, which may depend on covariates and AI outputs. These nuisance functions can be estimated using flexible machine learning methods such as random forests, neural networks, or gradient boosting.
    \item \textbf{Construct a bias-corrected estimator.} We combine the estimated nuisance functions into a Neyman-orthogonal score function that integrates information from both labeled and unlabeled observations. The Neyman orthogonality property \citep{chernozhukov2018double} ensures that estimation error in the nuisance functions affects the final parameter estimates only at second order. This robustness allows researchers to use flexible, data-adaptive methods for nuisance estimation without compromising the validity of inference on the target parameters.
\end{enumerate}

\subsection{Contribution}

This paper makes three main contributions to the marketing and business research literature:

\begin{enumerate}
    \item \textbf{Problem Formulation and Methodology.} 
    We identify that AI-assisted labeling creates a new statistical regime: AI outputs are not noisy surrogates for human labels, but structured auxiliary representations that can be leveraged as informative features. This distinction leads to a different problem formulation, in which AI outputs augment the covariate space rather than substitute for missing outcomes. Building on this insight, we develop GAI, a unified framework for AI-augmented estimation in GLMs, where we allow model misspecification. The method exploits a natural orthogonal moment structure, yielding a principled estimator that supports valid statistical inference while accommodating the full diversity of modern AI outputs, including discrete predictions, continuous scores, and high-dimensional or unstructured representations such as text embeddings. In this way, GAI provides a general statistical framework for integrating AI-generated signals into classical estimation problems.

    \item \textbf{Theoretical Guarantees.} 
    We establish that the GAI estimator is asymptotically normal, $\sqrt{n}$-consistent, and admits a closed-form variance expression (Theorem~\ref{thm:normality}). \mx{Beyond these standard asymptotic guarantees, we prove a key dominance property: under random labeling, GAI is guaranteed to weakly improve upon a unified family of debiased estimators, which includes human-data-only estimation and PPI-based methods (\cref{thm:dominance}).}
    Thus, \rev{under the random-labeling conditions of \cref{thm:dominance},} incorporating AI outputs can never worsen performance relative to ignoring auxiliary AI data entirely, regardless of how biased or inaccurate those outputs may be. \rev{We further extend the theory to \emph{covariate shift}, where the labeled sample need not be representative of the target population (Section~\ref{sec:covshift}): we develop a weighted GAI estimator for a known density ratio and a doubly robust GAI estimator for an estimated density ratio, each retaining $\sqrt{n}$-asymptotic normality for the target parameter.} \cl{We also prove that the weighted GAI estimator improves efficiency relative to weighted human-data-only estimation.}

    We further show that the efficiency gains from GAI admit an interpretable variance decomposition (Corollary~\ref{cor:variance_decomposition}). Specifically, variance reduction can arise from three distinct sources: (i) \emph{sample expansion}, where auxiliary observations effectively enlarge the sample available for estimation; (ii) \emph{representational power of AI}, where AI outputs provide richer transformations of the covariates and help approximate complex conditional expectations even when they contain no information beyond $\Xvec$; and (iii) \emph{extra predictive information from AI}, where AI outputs contain signal about human labels beyond the observed covariates. This decomposition provides a unified explanation for when and why AI augmentation improves statistical efficiency.

    \item \textbf{Empirical Evidence Across Three Auxiliary Data Regimes.}
    We evaluate GAI across three real-world business applications spanning core estimation tasks in marketing research. The three applications span progressively more favorable regimes for existing methods. In vaccine conjoint analysis, we estimate a logit model using both LLM choice predictions and high-dimensional embeddings of the model’s reasoning. The LLM achieves only 54\% prediction accuracy and observes no additional information beyond features in the outcome model. In retail pricing, all estimators receive the same systematically biased digital-twin purchase prediction \citep{toubia2025database}, which contains extra information from persona characteristics. This controlled comparison isolates the benefit of GAI’s feature-based formulation from the additional reasoning information. Finally, in health insurance coverage analysis, we use accurate, well-calibrated ML predictions constructed by the PPI authors, creating a setting particularly favorable to existing methods. 
    
    GAI delivers consistent gains across all three regimes, demonstrating robustness to differences in predictive quality, auxiliary-data format, and information content.
    Across all settings, GAI substantially improves point estimation accuracy, maintains near-nominal confidence interval coverage in most configurations without inflating interval width, and reduces required human labels by 67--90\%. The consistency of these results across diverse auxiliary data formats, information structures, and prediction quality levels demonstrates the broad applicability of the framework.

\end{enumerate}

\section{Literature Review}
\label{sec:lit}

\textbf{AI and Machine Learning in Marketing Research.} Marketing research has increasingly embraced AI and machine learning methods across a wide range of applications. AI-based algorithms are now used in pricing \citep{miklos2019collusion, zhang2021frontiers}, targeting \citep{rafieian2021targeting, iyer2024competitive, lu2025optimizing, rafieian2025multiobjective}, personalization and recommendation \citep{yoganarasimhan2020search, xu2022scalable, jiang2026sparsity}, product design \citep{toubia2017idea, huang2016consumer, sisodia2025generative}, and sales coaching \citep{luo2021artificial}. Beyond these applications, researchers are exploring how AI can enhance or substitute for traditional data collection approaches. LLMs have shown promise for simulating consumer behavior in market research tasks \citep{chen2023emergence, wang2024llm}. The digital twin paradigm, which uses LLM simulations of individual decision-makers, has emerged as a promising approach for generating synthetic behavioral data \citep{toubia2025database}. Recent work also explores training open-source LLMs to solve optimization problems \citep{huang2025orlm}. However, recent studies have identified important limitations: LLMs may fail to replicate the mechanisms underlying human decisions and produce data with biases that differ from human behavior \citep{gui2023challenge, goli2024frontiers, chen2025manager}. Our work contributes to this literature by providing a principled framework for combining scarce human-labeled data with abundant AI-generated representations while remaining robust to systematic AI errors.

\textbf{Transfer Learning.} Classical semi-supervised methods \citep{nigam2000text, zhu2009introduction, lee2013pseudo, chakrabortty2018efficient} use unlabeled features to improve prediction, while transfer learning \citep{pan2009survey, weiss2016survey} addresses distribution shift between source and target domains, such as leveraging existing knowledge to improve decisions in new settings \citep{chen2025data, duan2024target} or learning policies robust to environment shifts between training and deployment \citep{si2023distributionally}. 
Similar challenges arise in business research: demand models, pricing algorithms, targeting policies, and operations management require real-world behavioral data that may be scarce, expensive, or difficult to collect at scale \citep{bertsimas2006robust, besbes2013demand, levi2015data, kaynar2023estimating, xiong2023large, singhvi2024data, tang2025offline}. The AI-augmented estimator (AAE) framework \citep{wang2024llm} addresses the analogous challenge of combining AI and human data through parametric debiasing. Our setting is distinct: we have AI-generated \textit{representations} (not just labels) for auxiliary data, and utilize them as features in semi-parametric estimations. GAI can be viewed as addressing the sim-to-real problem in inference: how to leverage abundant simulated or AI-generated data flexibly while maintaining valid inference grounded in limited real observations.

\textbf{Prediction-Powered Inference.} \citet{angelopoulos2023prediction, angelopoulos2023ppi++} introduced PPI and PPI++ for constructing valid confidence intervals by leveraging ML predictions on unlabeled data. PPI uses labeled data to estimate and correct for the bias in predictions, yielding valid inference when predictions are reasonably accurate proxies for outcomes. A recent concurrent work by \citet{wang2025efficient} extends this line of research by fine-tuning LLMs for downstream statistical inference and applying PPI-based rectification. \mx{We show that GAI weakly dominates all estimators in a unified family \rev{of} debiased estimators, which includes PPI and PPI++.} We compare GAI and PPI-based methods more comprehensively in Online Appendix \ref{apdx: example_ppi}

\textbf{Neyman Orthogonality and Doubly Robust Estimation.}
The AI-assisted labeling problem induces a specific data structure: outcomes are partially missing, but for all observations we observe structured AI-generated representations that may be high-dimensional, unstructured, or categorically different from the outcome. This structure is richer than the classical missing-data setting \citep{rubin1976inference, little2019statistical}, where the researcher observes only covariates for units with missing outcomes. In our setting, the auxiliary signal encodes predictive information about the latent response that goes beyond what standard covariates provide.

This problem structure naturally admits an orthogonal moment representation. Semiparametric methods developed for missing outcomes, such as augmented inverse probability weighted (AIPW) estimators, provide useful conceptual tools \citep{robins1994estimation, robins1995semiparametric, hirano2003efficient, bang2005doubly}. 
While the resulting GAI score function shares the algebraic form of AIPW estimators, the role of each component differs substantively.
First, the augmentation term leverages AI-generated representations
that encode high-dimensional predictive information about the latent response, rather than correcting for missing outcomes using observed covariates alone.
Second, our analysis is developed within the double machine learning framework \citep{chernozhukov2018double} that has been adopted and extended in recent econometrics and ML research \citep{kallus2024localized}. We provide a clean set of conditions
and practical implementation tools such as cross-fitting, allowing our estimator to achieve valid inference across a broad class of GLMs. Our contribution is thus not a new semiparametric principle, but rather the recognition that AI-generated representations create a new estimation regime that is well-suited to orthogonal moment methods, together with the formal analysis of the estimator's properties in this context.

An earlier, preliminary version of this work was accepted at the International Conference on Machine Learning (ICML) \citep{anonymous2026gai}. The present manuscript substantially extends that version by \cl{generalizing the dominance theorem to a unified family of debiased estimators,} developing a variance-decomposition result that captures the sources of efficiency gains, \cl{extending the framework to covariate-shift settings, }strengthening the comparison with baseline approaches, providing additional implementation guidance for model selection in nuisance estimation, and adding a new application that extends the experiment to richer AI data regimes.

\section{Problem Formulation}
\label{sec:problem}
In this section, we formulate the problem within a 
\textit{misspecified GLM} framework and describe the corresponding data-generating and data-collection processes.

\subsection{A Misspecified Generalized Linear Model Framework and Applications}
\label{sec:framework}

We consider the following estimation framework: 
We let $\yvec \in \R^k$ denote a human label of interest, associated with a random feature matrix $\Xvec \in \mathcal{X} \subset \R^{k \times d}$. $\yvec\,|\,\Xvec$ follows an \emph{unknown} ground truth distribution. We aim to recover a target parameter $\betavec^* \in \mathcal{B} \subset \R^d$ , which solves:
\begin{equation}
    \betavec^*\,\in\,\arg\min_{\betavec \in \mathcal{B}} \E\left[\ell(\Xvec, \yvec; \betavec)\right],
    \label{eq:population}
\end{equation}
where $\ell(\Xvec, \yvec; \betavec) := b(\Xvec \betavec) - \yvec^{\top}\Xvec \betavec$ denotes a loss function and $b(\cdot): \Theta \to \R$ is a known twice-continuously differentiable convex function on an open convex set $\Theta \subset \R^k$. 
We assume that $\mB$ is an open convex set such that, for some compact set $\Breve{\mX}$ with $\mX \subset \Breve{\mX}$, it holds that $\Xvec \betavec \in \Theta$ whenever $\Xvec \in \Breve{\mX}$ and $\betavec \in \mB$.
Under these conditions, one can verify that 
$\betavec^*$ satisfies the first-order condition:
\begin{equation}
    \E\!\left[\nabla_{\betavec}\ell(\Xvec, \yvec; \betavec^*)\right]
\,=\,
\E\left[\Xvec^\top \left(\nabla b(\Xvec\betavec^*) - \yvec\right)\right]\,=\,0\,.
    \label{eq:foc}
\end{equation}


An important special case of our framework is the class of canonical GLMs, which impose a structured functional form on the conditional distribution of $\yvec \mid \Xvec$. In particular, canonical GLMs assume that $\yvec \mid \Xvec$ belongs to the exponential family, with density $f(\yvec\,|\,\Xvec; b, \betavec^*)  \propto \exp\left\{b(\Xvec\betavec^*) - \yvec^\top \Xvec\betavec^*\right\}$. Under this specification, \eqref{eq:population} corresponds to the expected log-likelihood minimization estimator and $\Xvec^\top \left(\nabla b(\Xvec\betavec^*) - \yvec\right)$ in \eqref{eq:foc} is the canonical score function. Consequently, our framework covers canonical GLMs that are widely used in applied modeling pipelines, including:

\begin{itemize}
    \item \textit{Linear regression:} $k=1$, $b(\theta)=\tfrac12\theta^2$. 
    This corresponds to squared-loss estimation for continuous outcomes and is the workhorse model in empirical studies, including treatment-effect estimation, forecasting, and policy evaluation.
    Note that linear regression, viewed as a special case of GLMs, requires normal densities conditional on $\Xvec$. 

    \item \textit{Logistic Regression and MNL:} for $k$ non-baseline classes,
    $b(\thetavec)=\log\!\big(1+\sum_{j=1}^{k}\exp(\theta_j)\big)$.
    This is the standard model for discrete choice among multiple alternatives, which is widely used for multiclass classification applications, such as text classification, image recognition, and recommendation systems, where predicted utilities determine selection probabilities among competing options. When $k=1$, this reduces to logistic regression.

    \item \textit{Poisson regression:} $k=1$, $b(\theta)=\exp(\theta)$.
    This models count outcomes and event rates, with applications such as demand incidence, click/conversion counts, and arrival processes in computer systems. Poisson GLMs are standard for modeling event rates in systems monitoring, reliability, and online experimentation.
\end{itemize}

While canonical GLMs posit a specific exponential-family form for the conditional distribution of $\yvec \mid \Xvec$, real-world data-generating processes rarely conform exactly to these assumptions. Our framework explicitly accommodates such deviations by relaxing the exponential-family requirement. In particular, we impose no restrictions on the true conditional distribution of $\yvec \mid \Xvec$ and allow for an arbitrary, unknown ground-truth relationship. As a result, our approach permits model misspecification and interprets the target parameter as the best approximation within a selected GLM class: With a $b(\cdot)$ function from GLM family, \eqref{eq:population} seeks a GLM density closest to the ground truth density in the conditional Kullback–Leibler (KL) divergence, i.e., 
$$
\betavec^*\,\in\,\arg\min_{\betavec \in \mathcal{B}}
\E_{\Xvec}\left[
\E_{\yvec \mid \Xvec}\left[
\log\left(
\frac{f^*(\yvec \mid \Xvec)}{f(\yvec \mid \Xvec; b, \betavec)}
\right)
\right]
\right],
$$
where $f^*(\yvec\mid \Xvec)$ corresponds to the true, unknown conditional density.
Therefore, we broadly view \eqref{eq:population} as a \textit{best-in-class} estimator, which can be regarded as the best generalized linear representation of the relationship between $\yvec$ and $\Xvec$.

\subsection{Data Generation and Collection}
\label{sec:data-generation}

To estimate the best-in-class parameter $\betavec^*$, we rely on observed data and an appropriate estimation procedure. In this section, we describe the available data, the underlying data-generating process, and the mechanism by which observations are collected.

Consider a dataset $\mathcal{D} = \{\Xi_i = (\Xvec_i, \yvec_i, w_i, \zvec_i)\}_{i=1}^n$ with \emph{partially observable} human labels. \cl{In the baseline setting (Sections~\ref{sec:problem}--\ref{sec:theory}), the observed data follows the same distribution as the target population underlying \eqref{eq:population}.}
Each data sample is i.i.d. In each sample, $\Xvec_i$ represent the feature matrix, $\yvec_i$ denote the human label of interest. Due to the substantial cost of human labeling, not all human labels are observed. We let $w_i \in \{0,1\}$ indicate whether the human label $\yvec_i$ is observed. While human annotations are expensive, AI-generated labels are comparatively low-cost. Therefore, for each sample we additionally observe $\zvec_i$, which denotes AI-generated auxiliary information, whose structure is left deliberately unrestricted. $\zvec$ may be structured (e.g., categorical labels, probability scores, or continuous values), high-dimensional vectors (e.g., embeddings), or unstructured objects (e.g., texts, images). Given $\mathcal{D}$, we define the \emph{primary} (human-labeled) and \emph{auxiliary} (AI-only) subsamples
$\mD^{\text{\upshape P}}=\{\Xi_i\in\mD:w_i=1\}, ~
\mD^{\text{\upshape A}}=\{\Xi_i\in\mD:w_i=0\}
$
and let $\rev{n_P}=|\mD^{\text{\upshape P}}|$, $\rev{n_A}=|\mD^{\text{\upshape A}}|$.

\subsubsection{Human Annotation Mechanism} 
We assume human labels are unobserved at random conditional on features and AI representations, i.e., $w \perp \yvec \mid \Xvec, \zvec$. This assumption is natural when primary data collection is carefully designed (e.g., random sampling of human annotations). We define $e^*(\Xvec, \zvec) = \E[w \,|\, \Xvec, \zvec] > \kappa$ for some $\kappa > 0$. In particular, $e^*(\Xvec, \zvec)$ represents the probability that the actual human label is observed given $\Xvec$ and $\zvec$. Such a data collection mechanism is common in practice, as the decision to collect real human labels can be based on observed feature information and AI-generated signals. For example, practitioners may prioritize human annotation for samples with specific covariate profiles or for cases where AI-generated predictions exhibit high uncertainty.

\subsubsection{AI Generation Mechanism} 
The core idea of our method is simple: we utilize $\zvec$ to extrapolate $\yvec$ on 
$\mD^{\text{\upshape A}}$ using a principled approach. 
Given AI-generated auxiliary information $\zvec$, we define $\gvec^*(\Xvec, \zvec) = \E[\yvec \,|\, \Xvec, \zvec]$, which serves as a key object that we estimate using machine learning methods in an intermediate step of our method (see Section~\ref{subsec:method}).
The benefit of introducing $\zvec$ is twofold. First, the AI signal $\zvec$ may contain predictive information about $\yvec$ beyond what is captured by $\Xvec$. This can occur when AI outputs are generated using auxiliary information not included in $\Xvec$, making $\zvec$ informative about $\yvec$ even after conditioning on $\Xvec$. Second, even when $\yvec \perp \zvec \mid \Xvec$, querying a powerful AI system can produce representations $\zvec$ that summarize complex patterns in $\Xvec$ relevant for predicting $\yvec$.
Learning $\gvec^*(\Xvec,\zvec)$ using machine learning methods is often easier compared with learning $\E[\yvec \mid \Xvec]$ in practice. For these two reasons, we can use the estimate of $\gvec^*(\Xvec,\zvec)$ to effectively extrapolate $\yvec$ on $\mD^{\text{\upshape A}}$ and 
the convergence requirement in Assumption~\ref{ass:ml_rate} is likely to be satisfied. 
Then, instead of using only 
$\mD^{\text{\upshape P}}$, we can utilize the entire sample $\mD$ to estimate the GLM with provably strong empirical performance. We illustrate with concrete examples below.
\vspace{2mm}

\begin{example}[$\yvec\,\not\perp\,\zvec \mid \Xvec$: Digital Twin Generation]
\label{example:correlated}
We tailor this example to the retail pricing experiment in \cref{sec:pricing}. Each observation corresponds to an individual--product--price query. For simplicity, we let $\Xvec \in \R^{n\times 2}$ denote the primary covariates used in the demand model (an intercept and price).
In practice, there can be other features included in $\Xvec$.
Let $\bU \in \R^{p}$ denote a vector of individual-level persona and demographic features. In the digital-twin setting, the agent is trained on each individual’s persona information, so the (unobserved) purchase outcome depends on both price and $\bU$. Concretely, suppose the ground-truth data-generating process is
\begin{align}
\Pe(y\,=\,1 \mid \Xvec,\bU)
\,=\,
\sigma\!\left(\eta_0^{*} + \eta_1^{*}\,\mathrm{price} + h(\bU)\right),
\label{eqn:pricing_truth_U}
\end{align}
where $\sigma(t)=1/(1+e^{-t})$ and $h(\bU)$ is an unknown (possibly high-dimensional) function capturing heterogeneous baseline propensity to purchase across persona types. The parameter of interest for the pricing study is the \emph{population}  effect of $\Xvec$ in a model that conditions only on $\Xvec$ (rather than on $\bU$). For instance, the analyst may wish to report price sensitivity under counterfactual population compositions of persona/demographic groups.

Now suppose we also observe an AI-generated digital-twin label $z \in \{0,1\}$. Because the digital twin is constructed using $\bU$, its prediction is also a function of $(\Xvec,\bU)$. For instance, we can model the twin output as
\begin{align}
\Pe(z\,=\,1 \mid \Xvec,\bU)
\,=\,
\sigma\!\left(\gamma_0 + \gamma_1\,\mathrm{price} + \tilde h(\bU)\right),
\label{eqn:pricing_twin_U}
\end{align}
where $\tilde h(\bU)$ reflects the twin’s learned representation of persona-driven purchase propensity. Under \eqref{eqn:pricing_truth_U}--\eqref{eqn:pricing_twin_U}, both $y$ and $z$ depend on the latent persona vector $\bU$. Clearly, after conditioning on $\Xvec$, the variables remain correlated since $\bU$ is a common driver of both the true purchase decision and the twin’s prediction. In this sense, $\yvec \not\perp \zvec \mid \Xvec$ and $z$ provides \emph{additional information beyond $\Xvec$} about $y$ by acting as a proxy for latent heterogeneity encoded in $\bU$. 
\hfill $\blacksquare$
\end{example}

\vspace{2mm}

Crucially, we note that the notation $\gvec^*(\Xvec, \zvec)$ does \emph{not} presume that $\yvec$ must be correlated with $\zvec$ conditional on $\Xvec$. In fact, our approach allows $\yvec$ and $\zvec$ to be conditionally independent given $\Xvec$, which is also a common scenario in AI data-generation schemes. In such cases, $\gvec^*(\Xvec, \zvec) = \E[\yvec \,|\, \Xvec, \zvec] = \E[\yvec \,|\, \Xvec]$. When $\gvec^*(\Xvec, \zvec) = \E[\yvec \,|\, \Xvec]$, including $\zvec$ as an argument in $\gvec^*$ simply reflects that AI-generated information may serve as a useful representation for the unknown ground-truth conditional expectation function $\E[\yvec \mid \Xvec]$, as illustrated in Example \ref{example:independent} in Online Appendix \ref{apdx: aux_example}. 

Together, Examples \ref{example:correlated} and \ref{example:independent} illustrate two typical AI data-generation schemes in which $\yvec \,\not\perp\, \zvec \mid \Xvec$ and $\yvec \,\perp\, \zvec \mid \Xvec$. Regardless of whether such conditional correlation holds, our proposed GAI approach achieves variance reduction
\rev{whenever the working GLM mean $\nabla b(\Xvec\betavec^*)$ deviates from the conditional mean $\E[\yvec \mid \Xvec, \zvec]$ in a direction picked up by the covariates---a mild informativeness/misspecification condition (\cref{thm:dominance}(iii))}. 
We analyze the mechanisms formally in \cref{sec:theory} and theoretically justify the resulting efficiency gains. We conduct empirical experiments corresponding to these two examples. \cref{sec:pricing} represents the digital twin generation setting described in \cref{example:correlated}, whereas \cref{sec: conjoint} represents the off-the-shelf LLM generation setting described in \cref{example:independent}. Both experiments yield statistically significant improvements consistent with our theoretical results.

\section{Methodology}
\label{sec:method}
In this section, we present our proposed methodology, GAI. We first discuss the key challenges in using both primary and auxiliary data to estimate $\betavec^*$ (\cref{subsec:challenges}). We formally introduce GAI in \cref{subsec:method} and discuss practical guidance for empirical implementation in \cref{sec:model_selection}. 
Unless noted otherwise, we let $\lVert \cdot \rVert$ denote the $\ell_2$-norm for a matrix or a vector.
Also, $\lVert \cdot \rVert_{F}$ and $\lVert \cdot \rVert_{\infty}$ are the Frobenius norm or the  $\ell_{\infty}$-norm of an appropriate object.  
Furthermore, given a multi-dimensional function $\mathbf{f}$, we write $\lVert \mathbf{f} \rVert_{Q,2} = \left(\E[\lVert\mathbf{f}\rVert^2]\right)^{1/2}.$
We use $\lambda_{\min}(\cdot)$ to denote the minimum eigenvalue. 

\subsection{The Challenges}
\label{subsec:challenges}
Here we formalize the challenges and discuss alternative methods.
The canonical score $\Xvec^\top(\nabla b(\Xvec\betavec)-\yvec)$ in~\eqref{eq:foc} requires observing $\yvec$, so it cannot be evaluated on auxiliary observations with $w=0$. Using only the primary sample $\mathcal{D}^{P}=\{\Xi_i=(\Xvec_i,\yvec_i,\rev{w_i,\zvec_i}): w_i=1\}$, the standard estimator $\widehat{\betavec}^{\,P}$ is the solution to the empirical score equation based on the canonical GLM score:
\begin{align}
\frac{1}{n_P}\sum_{i:w_i=1}\Xvec_i^{\top}\!\left\{\nabla b\!\left(\Xvec_i\betavec\right)-\yvec_i\right\}
\,=\,0\,.
\label{eq:primary_score}
\end{align}
We denote by $\widehat{\betavec}^{\,P}$ the primary-only estimator. By construction, this estimator relies exclusively on the primary sample and is therefore fundamentally constrained by the primary sample size, which is often small due to the high cost of human data collection.
The primary difficulty with this approach is that it ignores the auxiliary sample completely. 
We also remark that when $e(\Xvec, \zvec)$ is not a constant and the model is indeed misspecified, it is not difficult to construct examples where this estimator does not have the desired statistical properties, such as consistency.  

An alternative approach that utilizes the entire dataset is to na\"ively pool the primary (human-labeled) and auxiliary (AI-generated) observations and directly plug them into the empirical score equation as if all labels were equally reliable. Specifically, the naive estimator $\widehat{\betavec}^{\sf Naive}$ is the solution to the following empirical score equation:
\begin{align}
\frac{1}{n_P+n_A}\left\{\sum_{i:w_i=1}\Xvec_i^{\top}\!\left\{\nabla b\!\left(\Xvec_i\betavec\right)-\yvec_i\right\} + \sum_{i:w_i=0}\Xvec_i^{\top}\!\left\{\nabla b\!\left(\Xvec_i\betavec\right)-\zvec_i\right\}\right\}
\,=\,0\,.
\label{eq:naive_score}
\end{align}
While this procedure leverages a larger sample size, it generally yields \textit{biased} estimates when the AI-generated labels do not coincide with the true human outcomes. Additionally, when the auxiliary signal $\zvec$ takes forms beyond a label surrogate, the naive pooling formulation is no longer applicable.


\subsection{Generative Augmented Inference}
\label{subsec:method}

To address these challenges, we propose the following score function:
\begin{equation}
\label{eq:score}
    \Psivec(\Xi; e,\gvec; \betavec)\,:=\,
    \Xvec^\top \left[\nabla b(\Xvec\betavec) - \gvec(\Xvec, \zvec) + \frac{w}{e(\Xvec, \zvec)}(\gvec(\Xvec, \zvec) - \yvec)\right]\,.
\end{equation}
This proposed score can be viewed as an \emph{orthogonalized} version of the complete-data score:
it replaces the missing label $\yvec$ by the regression function $\gvec(\Xvec,\zvec)$ and corrects the residual using IPW through $e(\Xvec,\zvec)$. This construction allows auxiliary data (with $w=0$) to contribute information through $\gvec(\Xvec,\zvec)$, even though $\yvec$ is unobserved. For brevity, we sometimes suppress the dependency on $\Xvec$ and $\zvec$ in $e(\cdot)$ and $\gvec(\cdot)$ when clear from the context.






\begin{algorithm}[t]
\caption{Generative Augmented Inference}
\small
\setlength{\baselineskip}{1\baselineskip}
\begin{algorithmic}[1]
\label{alg: gai}
\REQUIRE Data $\mathcal{D} = \mathcal{D}^P \cup \mathcal{D}^A$, number of folds $K$
\vspace{1mm}

\STATE Randomly partition all data $\mathcal{D}$ into $K$ folds $I_1,\ldots,I_K$

\FOR{$k = 1,\ldots,K$}
    \STATE \textbf{Nuisance estimation:}
    \STATE \quad Estimate $\hat{e}^{(k)}(\Xvec, \zvec)$ using all observations in $\mathcal{D} \setminus I_k$
    \STATE \quad Estimate $\hat{\gvec}^{(k)}(\Xvec, \zvec)$ using primary observations in $\mathcal{D} \setminus I_k$
    \STATE \quad Compute out-of-sample predictions $\hat{e}^{(k)}_i, \hat{\gvec}^{(k)}_i$ for each $i \in I_k$
\ENDFOR

\STATE \textbf{Target estimation:}
\STATE \quad Obtain $\hat{\betavec}$ by minimizing the norm of the average score across folds:
$$        \left\lVert \frac{1}{n}\sum_{k=1}^{K}\sum_{i \in I_k}\Psivec_i(\hat{e}^{(k)}, \hat{\gvec}^{(k)}; \hat{\betavec}) \right \rVert\,\leq\,\inf_{\betavec\in \mB}\left\lVert \frac{1}{n}\sum_{k=1}^{K}\sum_{i \in I_k}\Psivec_i(\hat{e}^{(k)}, \hat{\gvec}^{(k)}; {\betavec}) \right \rVert
        + o_P(n^{-1/2})$$

\STATE \textbf{Variance estimation:}
\STATE \quad Compute $\hat{\Psivec}_i = \Psivec_i(\hat{e}^{(k)}, \hat{\gvec}^{(k)}; \hat{\betavec})$ for each $i \in I_k$, and estimate:
$$\hat{\Jvec}\,=\,\frac{1}{n}\sum_{k=1}^{K}\sum_{i \in I_k} \Xvec_i^\top \nabla^2 b(\Xvec_i\hat{\betavec}) \Xvec_i, \qquad \widehat{\Sigmavec}\,=\,\hat{\Jvec}^{-1} \left(\frac{1}{n}\sum_{k=1}^{K}\sum_{i \in I_k} \hat{\Psivec}_i \hat{\Psivec}_i^\top\right) \hat{\Jvec}^{-1}$$
\STATE \quad where $\hat{\Psivec}_i = \Psivec_i(\hat{e}^{(k)}, \hat{\gvec}^{(k)}; \hat{\betavec})$

\RETURN $\hat{\betavec}$, $\widehat{\Sigmavec}$
\end{algorithmic}
\end{algorithm}

We now propose the GAI algorithm used to estimate the target parameters $\betavec^*$. Algorithm \ref{alg: gai} summarizes the full procedure. GAI implements a cross-fitted estimation procedure that integrates human-labeled and AI-augmented data. We begin by partitioning all data $\mathcal{D}$ into $K$ folds. For each fold $k$, the algorithm proceeds in two steps. In step 1, we estimate the nuisance functions via flexible machine learning methods: the labeling propensity $\hat{e}^{(k)}(\Xvec,\zvec)$ is estimated using all observations not in fold $I_k$, while the conditional expectation $\hat{\gvec}^{(k)}(\Xvec,\zvec)$ is estimated using only primary observations not in fold $I_k$ (since estimating $\gvec$ requires observing $\yvec$). We then compute out-of-sample predictions for observations in fold $I_k$. In step 2, we aggregate the score functions across all folds and obtain $\hat{\betavec}$ by minimizing the norm of this aggregated score. Finally, we estimate the asymptotic variance using the sandwich formula from Theorem~\ref{thm:normality}.

GAI enables valid inference under flexible AI-generated data. In Section \ref{sec:theory}, we provide theoretical support for GAI and formally analyze how it addresses the aforementioned challenges. Section \ref{sec:experiments} presents empirical validations of its practical performance. We now turn to implementation considerations, offering practical guidance for applying GAI in real-world settings.

\subsection{Hyperparameter Selection for Nuisance Estimation}
\label{sec:model_selection}

Because GAI's score function satisfies Neyman orthogonality, first-order errors in estimating the nuisance functions $\gvec(\Xvec,\zvec)$ and $e(\Xvec,\zvec)$ do not affect the asymptotic distribution of the estimator $\hat{\betavec}$ \citep{chernozhukov2018double}. Model selection for the nuisance estimators is therefore not central to the method's validity. A reasonable fixed specification is sufficient in most applications.

When abundant data are available, however, practitioners may still wish to tune hyperparameters to improve finite-sample performance. We propose a nested cross-validation procedure that preserves the independence structure required for valid inference: within each outer fold $k$ of the cross-fitting procedure, an inner cross-validation loop selects the best model configuration using only data from the remaining $K-1$ folds, evaluating prediction loss on primary observations in the inner validation set. This ensures that model selection does not contaminate the out-of-fold predictions used for final inference. We demonstrate this procedure in the retail pricing experiment (Section~\ref{sec:pricing}), searching over 8 configurations spanning three model classes.

\section{Theoretical Guarantees}
\label{sec:theory}
Here we present the theoretical guarantees of GAI: we establish its asymptotic normality and strong dominance over primary-only estimators, and then analyze the sources of its efficiency improvement.

\subsection{Asymptotic Normality}
\label{sec:normality_and_dominance}
The theoretical properties we discuss here are built upon the following assumptions:

\begin{assumption}[Regularity]
\label{ass:regularity}
(i) $b(\cdot)$ is twice continuously differentiable with $\nabla^2 b(\theta) \succ 0$ on $\Theta$; (ii) $\E[\Xvec^\top \Xvec]$ is positive definite; (iii) $\|\Cov(\yvec \,|\, \Xvec,\zvec)\| \leq \tilde{\sigma}^2$; (iv) $\E\big[\lVert \yvec \rVert^2\big] < \infty$.
\end{assumption}

\begin{assumption}[ML Convergence Rate]
\label{ass:ml_rate}
There exists $\alpha(n) \downarrow 0$ and $r_1 + r_2 \geq 1/2$ such that:
\begin{align}
    \|\hat{e}(\Xvec,\zvec) - e^*(\Xvec,\zvec)\|_{Q,2}\,&\leq\,\alpha(n)/n^{r_1} \\
    \|\hat{\gvec}(\Xvec,\zvec) - \gvec^*(\Xvec,\zvec)\|_{Q,2}\,&\leq\,\alpha(n)/n^{r_2}
\end{align}
and $\sup_{\Xvec,\zvec}|\hat{e}(\Xvec,\zvec) - e^*(\Xvec,\zvec)| \to_P 0$.
\end{assumption}

The product rate condition $r_1 + r_2 \geq 1/2$ is a standard requirement in semiparametric estimation with machine-learned nuisance functions. Analogous conditions appear in policy learning \citep{athey2021policy} and causal inference with random forests \citep{wager2018estimation, athey2019generalized}. The condition is easily satisfied when both nuisance estimators converge at rate $n^{-1/4}$ or faster. This mild requirement is met by a wide range of modern ML methods under standard smoothness or sparsity conditions: Lasso and its variants achieve $n^{-1/4}$ rates in sparse high-dimensional models \citep{bickel2009simultaneous, buhlmann2011statistics}; random forests and kernel methods attain these rates under mild regularity conditions \citep{wager2018estimation}; neural networks achieve them when the target function admits a suitable compositional structure \citep{farrell2021deep}; and boosting methods reach comparable rates under appropriate complexity constraints \citep{luo2016high}. We refer to \citet{chernozhukov2018double} for a thorough discussion. Under these assumptions, our estimator has the following properties. Define the information matrix $\Jvec := \E[\Xvec^\top \nabla^2 b(\Xvec\betavec^*) \Xvec]$.

\begin{theorem}[Asymptotic Normality]
\label{thm:normality}
Under Assumptions~\ref{ass:regularity}--\ref{ass:ml_rate},
\begin{equation}
    \sqrt{n}(\hat{\betavec} - \betavec^*)\,\rightsquigarrow\,N(0, {\Sigmavec}^{\text{\upshape GAI}})
\end{equation}
where ${\Sigmavec}^{\text{\upshape GAI}} = \Jvec^{-1} \E[\Psivec\big(\Xi; e^*,\gvec^*; {\betavec}^*\big)\Psivec\big({\Xi}; e^*,\gvec^*; {\betavec}^*\big)^\top] \Jvec^{-1}$.
\end{theorem}

The proof (Online Appendix~\ref{apdx: proof of normality}) proceeds by: (1) establishing score validity and Neyman orthogonality (Lemma~\ref{lem: orthogonality}) and rates for the empirical scores (Lemma~\ref{lem: empirical-scores}); (2) proving consistency via convexity arguments (Lemma~\ref{lem: consistency}); (3) deriving the asymptotic expansion using Donsker theory.


\subsection{Efficiency with Uniform Random Sampling}

In typical applications, human labels are collected via uniform random sampling---for instance, randomly selecting respondents for annotation, randomly assigning observations to human review, or randomly subsampling from a larger dataset. 
Therefore, we focus on this natural scenario in this section and adopt the next assumption throughout. 

\begin{assumption}[Uniform Random Sampling]
\label{ass:random_labeling}
The labeling indicator $w$ is independent of $(\Xvec, \yvec, \zvec)$, so that $e^*(\Xvec, \zvec) = \Pr(w = 1) = \rho \in (0, 1]$ is a known constant.
\end{assumption}

Because $\rho$ is known, the propensity need not be estimated for GAI.
Therefore, we can set $\hat{e} \equiv \rho$ in Algorithm~\ref{alg: gai}.
Under this assumption,  we further show that GAI achieves efficiency in a unified family of debiased estimators, including primary-only (see for example \eqref{eq:primary_score}), PPI \citep{angelopoulos2023prediction}, and PPI++ \citep{angelopoulos2023ppi++}. To this end, for $\lambda \in [0,1]$ and $\fvec(\Xvec, \zvec) \in \R^{{k}}$, we define a unified score function
\begin{equation}
\label{eq:ppi_family_score}
\Psivec_{\lambda, \fvec}(\Xi; \betavec)
\,:=\,
\frac{w}{\rho}\,\Xvec^\top\!\Big[\big(\nabla b(\Xvec\betavec) - \yvec\big) - \lambda\big(\nabla b(\Xvec\betavec) - \fvec(\Xvec, \zvec)\big)\Big]
\;+\;
\frac{1 - w}{1 - \rho}\,\lambda\,\Xvec^\top\big(\nabla b(\Xvec\betavec) - \fvec(\Xvec, \zvec)\big),
\end{equation}
and define $\hat\betavec_{\lambda, \fvec} := \arg\min_{\betavec \in \mB} \big\lVert \Pe_n \Psivec_{\lambda, \fvec}(\cdot\,; \betavec) \big\rVert$.\footnote{It is shown in the proof of Theorem \ref{thm:dominance} that $\Pe_n \Psivec_{\lambda, \fvec}(\cdot\,; \betavec) = \mathbf{0}$ has a solution with probability one, so the minimizer here is well-defined. Alternatively, we can define $\hat\betavec_{\lambda, \fvec}$ as any solution that achieves an objective function value less than $\inf_{\betavec \in \mB}\big\lVert \Pe_n \Psivec_{\lambda, \fvec}(\cdot\,; \betavec) \big\rVert + o_P(1/\sqrt{n})$. The proof is slightly more complex but follows along the same line as Theorem \ref{thm:normality}.}

Let us assume that $\zvec$ can be partitioned into two parts, i.e., $\zvec = (\zvec^{\text{P}}, \zvec^{\text{O}})$, where $\zvec^{\text{P}}\in \R^{{k}}$ is AI-generated prediction of $\yvec$, and $\zvec^{\text{O}}$ is other information generated by AI, such as reasoning among others. 
The unified score function \eqref{eq:ppi_family_score} covers a wide range of debiased estimators, distinguished uniquely by $(\lambda, \fvec)$. Four estimators of interest are members of this family, as summarized in \cref{tab:unified_family}.
In particular, primary-only ($\lambda = 0$) discards the prediction term and reduces
\eqref{eq:ppi_family_score} to the score
$\tfrac{w}{\rho}\Xvec^\top(\nabla b(\Xvec\betavec) - \yvec)$. 
PPI and PPI++ advocate soliciting  the AI response $\zvec^{\text{P}}$ as the prediction and then  take $\fvec = \zvec^{\text{P}}$ \citep[e.g., see][]{kreps2020factors}. 
PPI and PPI++ differ only in the scalar $\lambda$.
PPI takes $\lambda = 1$.
PPI++ tunes {an} instance-specific $\lambda$ to minimize the asymptotic variance.
Let us denote it by $\lambda^{\rev{\star}}$.
GAI takes the
value $\lambda = 1 - \rho$ together with $\fvec = \gvec^*$, recovering the
GAI score \eqref{eq:score}.
That is, $\Psivec_{1 -\rho, \gvec^*}(\Xi; \betavec) = \Psivec(\Xi; e^*, \gvec^*; \betavec)$.
Our main result (Theorem~\ref{thm:dominance}) is that GAI {attains the unique optimal asymptotic variance $\Sigmavec^{\text{\upshape GAI}}$ within this family} in the positive-semidefinite (Loewner) order: it
weakly dominates primary-only, PPI, and PPI++ simultaneously in every direction, and {strictly dominates each competitor in every direction along which the competitor's score correction differs from GAI's on a set of positive probability. The optimal variance is unique; the pair $(\lambda, \fvec)$ attaining it need not be, and GAI is its canonical member.}
For convenience, we define $\mathbf{S} := \Xvec^\top(\nabla b(\Xvec\betavec^*) - \yvec) = \zetavec + \pivec$, with
$\zetavec := \Xvec^\top(\nabla b(\Xvec\betavec^*) - \gvec^*)$,
$\pivec := \Xvec^\top(\gvec^* - \yvec)$, and $\gvec^* = \E[\yvec \mid \Xvec, \zvec]$, so
that $\E[\pivec \mid \Xvec, \zvec] = \mathbf{0}$.

\begin{theorem}[Dominance over the Unified Family]
\label{thm:dominance}
Suppose Assumptions~\ref{ass:regularity} and \ref{ass:random_labeling} hold, with known $e \equiv \rho \in (0,1)$ (so that both labeled and unlabeled units occur), and $\Jvec \succ 0$. 
Also, {assume that} $\hat{\gvec}$ is $L_{Q,2}$-consistent, i.e., $\lVert \hat{\gvec} - \gvec^* \rVert_{Q,2} = o(1)$.
Consider the family
$\{\hat\betavec_{\lambda, \fvec} : \lambda \in \R,\ \fvec \in L_2(\Xvec, \zvec)\}$ defined by
\eqref{eq:ppi_family_score}. Then:
\begin{enumerate}
\item[(i)] Every member is consistent and asymptotically normal for the \emph{same}
$\betavec^*$:
$\sqrt{n}\,(\hat\betavec_{\lambda, \fvec} - \betavec^*) \rightsquigarrow N(\mathbf{0},\, \Sigmavec(\lambda, \fvec))$
with
$\Sigmavec(\lambda, \fvec) \,=\, \Jvec^{-1}\boldsymbol{\Omega}(\lambda, \fvec)\Jvec^{-1}$,
where
$\boldsymbol{\Omega}(\lambda, \fvec) = \E[\Psivec_{\lambda, \fvec}(\Xi; \betavec^*)\Psivec_{\lambda, \fvec}(\Xi; \betavec^*)^\top]$
admits, with $\mathbf{h} := \lambda\Xvec^\top(\nabla b(\Xvec\betavec^*) - \fvec)$ and
$\boldsymbol{\delta} := \mathbf{h} - (1 - \rho)\mathbf{S}$, the closed form
\begin{equation}
\label{eq:ppi_omega_closed}
\boldsymbol{\Omega}(\lambda, \fvec)
\,=\,
\E[\mathbf{S}\mathbf{S}^\top] \,+\, \frac{1}{\rho(1 - \rho)}\,\E[\boldsymbol{\delta}\boldsymbol{\delta}^\top]
\,=\,
\frac{1}{\rho}\,\E\big[(\mathbf{S} - \mathbf{h})(\mathbf{S} - \mathbf{h})^\top\big]
\,+\, \frac{1}{1 - \rho}\,\E[\mathbf{h}\mathbf{h}^\top].
\end{equation}

\item[(ii)] Primary-only, PPI, and PPI++ are the members $\lambda = 0$;
$(\lambda, \fvec) = (1, \zvec^{\text{\upshape{P}}})$; and $(\lambda, \fvec) = (\lambda^\star, \zvec^{\text{\upshape{P}}})$, respectively,
while GAI is the member $(\lambda, \fvec) = (1 - \rho, \gvec^*)$.

\item[(iii)] GAI is the Loewner-minimizer of the family. The optimal score correction is
$\mathbf{h}^\star = (1 - \rho)\zetavec$, attained inside the family by
$\lambda = 1 - \rho$, $\fvec = \gvec^*$, and yielding
\begin{equation}
\label{eq:ppi_omega_gai}
\boldsymbol{\Omega}^{\text{\upshape GAI}}
\,=\,
\E[\zetavec\zetavec^\top] \,+\, \frac{1}{\rho}\,\E[\pivec\pivec^\top].
\end{equation}
For every $(\lambda, \fvec)$, writing the residual
$\mathbf{r} := \lambda(\nabla b(\Xvec\betavec^*) - \fvec) - (1 - \rho)(\nabla b(\Xvec\betavec^*) - \gvec^*) \in \R^{\rev{k}}$,
\begin{equation}
\label{eq:ppi_gap}
\Sigmavec(\lambda, \fvec) - \Sigmavec^{\text{\upshape GAI}}
\,=\,
\frac{1}{\rho(1 - \rho)}\,\Jvec^{-1}\,\E\big[(\Xvec^\top\mathbf{r})(\Xvec^\top\mathbf{r})^\top\big]\,\Jvec^{-1}
\,\succeq\,\mathbf{0},
\end{equation}
with equality if and only if $\Xvec^\top\mathbf{r} = \mathbf{0}$ almost surely (for GAI,
$\mathbf{r} = \mathbf{0}$).

\item[(iv)] {The feasible GAI estimator $\hat{\betavec}$ of Algorithm~\ref{alg: gai} with $\hat{e} \equiv \rho$ is asymptotically equivalent to the oracle GAI member $(1 - \rho, \gvec^*)$: $\sqrt{n}\,(\hat{\betavec} - \betavec^*) \rightsquigarrow N(\mathbf{0},\, \Sigmavec^{\text{\upshape GAI}})$, so the Loewner dominance in part~(iii) extends to it.}
\end{enumerate}
\end{theorem}

The proof can be found in {Online Appendix~\ref{apdx: proof of unified dominance}}. \cref{thm:dominance} shows that GAI weakly dominates primary-only, PPI, and PPI++ in the Loewner order, and
strictly dominates each whenever the corresponding $\Xvec^\top\mathbf{r} \neq \mathbf{0}$ on
a set of positive probability.
In general, this condition holds for primary-only, PPI and PPI++, even in cases where $\zvec = \zvec^{\text{P}}$.
The optimal score function, and hence
$\boldsymbol{\Omega}^{\text{\upshape GAI}}$ and $\Sigmavec^{\text{\upshape GAI}}$, is unique.
This result provides {theoretical} support for the superior empirical performance of GAI over other estimators reported in the next section.

It is well-known that estimators based on Neyman-orthogonal score functions are semi-parametric efficient in many settings, for example, when the loss function model is correctly specified as the log-likelihood \citep{van2000asymptotic, chernozhukov2018double}.
However, this is not necessarily true in our setting due to model misspecification, so the dominance result given above is remarkable. \cref{thm:dominance} establishes a ``safe default'' property: under random labeling, adopting GAI can never degrade performance relative to ignoring the AI data entirely, and will strictly improve performance
whenever the AI outputs shift the conditional mean of the outcome in a direction reflected in the score (the informativeness condition of \cref{thm:dominance}(iii)). This guarantee is operationally valuable because it eliminates the risk of ``doing worse by trying to use AI data,'' which is a real concern with methods whose performance depends on prediction quality.

\begin{table}[]
\centering
    \begin{tabular}{lll}
    \toprule
    Estimator & $\lambda$ & $\fvec$ \\
    \midrule
    Primary-only (human labels, IPW) & $0$ & --- \\
    PPI & $1$ & $\zvec^\text{P}$ \\
    PPI++ & $\lambda^\star$ (asymptotic-variance/trace minimizer) & $\zvec^{\text{P}}$ \\
    GAI & $1 - \rho$ & $\gvec^* = \E[\yvec \mid \Xvec, \zvec]$ \\
    \bottomrule
    \end{tabular}
\caption{Members of the Unified Debiased Estimator Family}
\label{tab:unified_family}
\end{table}

\subsection{Source of Variance Reduction}

In this part we present deeper insight into why GAI achieves variance reduction. First, we formalize a decomposition of variance in the following corollary when comparing GAI and primary-only.

\begin{corollary}[Variance Decomposition]
\label{cor:variance_decomposition}
    Under the conditions of \cref{thm:dominance}, the variance gain of GAI over primary-only estimation decomposes as:
    \begin{align}
    \Sigmavec^{\text{\upshape P}} - \Sigmavec^{\text{\upshape GAI}} 
\,&\propto\,\underbrace{(\frac{1}{\rho} - 1)}_{\text{\upshape (I) Gain from sample expansion}} \cdot \Bigg\{
    \underbrace{\E\bigg[\Xvec^{\top}\Big(\nabla b(\Xvec\betavec^*) - \E[\yvec \mid \Xvec]\Big)\Big(\nabla b(\Xvec\betavec^*) - \E[\yvec \mid \Xvec]\Big)^{\top}\Xvec\bigg]}_{\text{\upshape (II) {Potential gain from representational power of $\zvec$}}} \notag \\
    &\qquad\qquad + \underbrace{\E\bigg[\Xvec^{\top}\Big(\E[\yvec \mid \Xvec] - \E[\yvec \mid \Xvec, \zvec]\Big)\Big(\E[\yvec \mid \Xvec] - \E[\yvec \mid \Xvec, \zvec]\Big)^{\top}\Xvec\bigg]}_{\text{\upshape (III) Gain from extra information from } \zvec \text{\upshape\ given } \Xvec}
    \Bigg\},
    \label{eq:variance_decomposition}
    \end{align}
where {$\rho \in (0,1)$ is the labeling probability of Assumption~\ref{ass:random_labeling}}.
\end{corollary}

We remind readers that $\gvec^* = \E[\yvec \mid \Xvec, \zvec]$. 
Both terms (II) and (III) in Equation \ref{eq:variance_decomposition} are positive semidefinite. This decomposition identifies three distinct sources of variance reduction:

\emph{Source (I): Sample Expansion.} The factor $1/\rho-1$ captures the variance reduction from expanding the effective sample size. Primary-only estimation uses only the $n_P$ labeled observations, whereas GAI leverages both labeled and auxiliary observations. If either term (II) or term (III) is positive definite, then, to leading order, this factor increases as the auxiliary sample grows relative to the primary sample (i.e., as $\rho \to 0$), enlarging the variance reduction. This leading-order gain comes with a finite-sample caveat: driving $n_A \to \infty$ is counterproductive, since as $\rho = n_P/n \to 0$ the inverse-propensity weight $1/\rho$ inflates the Neyman-orthogonality remainder of Lemma~\ref{lem: orthogonality}---a weak-overlap effect whose control is bottlenecked by $n_P$, the only data informing $\hat{\gvec}$---so the Gaussian approximation degrades and $n_A$ is best chosen at a finite ``sweet spot'' rather than as large as possible.
Our empirical studies presented in subsequent sections show that in practice this choice is usually quite robust when $\rho \in [0.2, 0.8]$.

\emph{Source (II): Representational Power of $\zvec$.} This term captures the representational power of $\zvec$ when the ground truth model is mis-specified, meaning $\nabla b(\Xvec\betavec^*) \neq \E[\yvec \mid \Xvec]$.
GAI could achieve variance reduction even if $\zvec$ contains no information beyond $\Xvec$, i.e., $\yvec \perp \zvec \mid \Xvec$. 
This occurs because the true conditional expectation function $\E[\yvec \mid \Xvec]$ is unknown and may be highly complex, falling outside the canonical GLM family. Model misspecification {operates here at the level of conditional means: term (II) is nonzero precisely when $\Xvec^{\top}\big(\nabla b(\Xvec\betavec^*) - \E[\yvec \mid \Xvec]\big) \neq \mathbf{0}$ with positive probability---that is, when the GLM mean deviates from the true conditional mean in a direction picked up by the covariates, the generic situation under misspecification---and positive definite when this vector is not almost surely confined to a proper subspace. A nonzero term (II) allows} GAI to realize variance reduction through sample expansion.
Beyond this asymptotic improvement, GAI also yields a gain in finite-sample performance when $\zvec$, leveraging its pre-trained knowledge, provides a better representation of $\Xvec$. Specifically, if $\E[\yvec \mid \Xvec]$ is highly complex, estimating it using only the limited primary sample is challenging. The auxiliary representation $\gvec(\Xvec, \zvec)$ leverages $\zvec$ to capture part of this complexity, thereby facilitating learning in small samples. For example, in \cref{example:independent}, directly estimating $\Pe(\yvec|\Xvec)$ from primary data alone is challenging, whereas incorporating $\zvec$ as an additional feature allows even a simple logistic regression model to recover $\Pe(\yvec|\Xvec)$ effectively. This illustrates one source of variance reduction achieved by GAI.

\emph{Source (III): Extra Information in $\zvec$ Given $\Xvec$.} This term captures the additional predictive power of $\zvec$ for $\yvec$ beyond what $\Xvec$ alone provides. It is {nonzero precisely when $\Xvec^{\top}\big(\E[\yvec \mid \Xvec] - \E[\yvec \mid \Xvec, \zvec]\big) \neq \mathbf{0}$ with positive probability---that is, when $\zvec$ shifts the conditional \emph{mean} of $\yvec$ beyond what $\Xvec$ provides; conditional dependence between $\yvec$ and $\zvec$ given $\Xvec$ that appears only in higher moments is not sufficient---and positive definite when this vector is not almost surely confined to a proper subspace}.
This reveals another source of gain: if the auxiliary information provides additional signal about the outcome, GAI can exploit this through the conditional expectation $\E[\yvec \mid \Xvec, \zvec]$.

A decomposition similar to this corollary can be derived for PPI/PPI++, but it is less intuitive, so we skip it in the paper. 
Here we highlight the key failure modes for PPI/PPI++: 
Both methods choose an inferior direction $\fvec = \zvec^{\text{P}}$. 
PPI fixes $\lambda = 1$ regardless of the quality of $\zvec^{\text{P}}$;
when $\zvec^{\text{P}}$ is a poor or merely misaligned surrogate for $\yvec$ the residual
$\Xvec^\top\mathbf{r}$ is large, the penalty \eqref{eq:ppi_gap} is strictly positive, and
PPI can \emph{inflate} variance even relative to primary-only. PPI++ repairs this, but only
by shrinking the single scalar $\lambda$ toward the variance-optimal point of the fixed
direction $\fvec = \zvec^{\text{P}}$; it never alters that direction. 
Instead, GAI uses the optimal $\fvec = \gvec^*$. This is only achievable through a double machine learning type framework to remove the additional bias through Neyman orthogonality of the score function and cross-fitting. When additional information $\zvec^{\text{O}}$ is available, this advantage is amplified: by the argument of Theorem \ref{thm:dominance}, a version of GAI using $\fvec = \E[\yvec \mid \Xvec, \zvec^{\text{P}}]$
improves upon PPI/PPI++, and the full version of GAI further improves on it by focusing on the optimal direction $\fvec = \gvec^*$.
In Appendix \ref{apdx: example_ppi}, we construct a small and intuitive example in the simple case of mean estimation to illustrate the advantage of GAI over PPI and PPI++. 
 
Finally, the random selection condition ($e(\Xvec, \zvec) = \rho$ constant) is essential for the dominance guarantee. In Online Appendix~\ref{apdx: failure of dominance}, we provide a constructive example showing that dominance can fail when the labeling probability $e(\Xvec, \zvec)$ depends non-trivially on covariates or AI outputs. Intuitively, if human labeling is strategically directed toward observations where the outcome is more predictable (e.g., labeling ``easy'' cases), the unlabeled observations may contribute noise rather than signal. {When labeling is non-random, the unlabeled set may be systematically harder to predict, and the augmentation term can amplify noise rather than reduce it. The fundamental issue is that strategic labeling creates a selection effect that can make the labeled sample more informative per observation than the unlabeled sample, potentially reversing the efficiency ordering.} In practice, the dominance condition is satisfied by simple uniform random selection of human labels from a dataset whose covariate distribution matches the target population. Section~\ref{sec:covshift} relaxes this distributional alignment by allowing the labeled source population to differ from the target population while maintaining dominance over the weighted primary estimator.

\section{Extension to Covariate Shift}
\label{sec:covshift}

The analysis so far assumes that the dataset $\mD$ is drawn from the same population that defines the target parameter $\betavec^*$ in \eqref{eq:population}. In many marketing applications this assumption fails by construction. The parameter of interest is defined for a broad \emph{target} market---all current and prospective customers of a product, the population a pricing or messaging decision will ultimately affect---yet labeled and AI-augmented observations are almost never a representative draw from that market. Human labels are expensive, so they are collected wherever they are accessible: an online survey panel, a single customer segment, a cohort of early adopters, or a sub-population experiment fielded on whoever can be recruited. The covariate composition of this \emph{source} sample---its mix of demographics, prior purchases, price sensitivities---typically differs from that of the target market, so an estimator that ignores the discrepancy recovers the parameter of the accessible sub-population rather than of the market the decision is about. This is the classical problem of \emph{covariate shift}, studied extensively in transfer learning and the broader literature on dataset shift and domain adaptation \citep{shimodaira2000improving, sugiyama2012density, quinonero2008dataset}; here it interacts with the partial labeling and AI-augmentation structure of GAI, and this section develops the corresponding theory. \cl{The empirical counterparts in the vaccine conjoint application are reported in Online Appendix~\ref{apdx:conjoint_gaiw_covariate_shift} for the known-density-ratio design and Online Appendix~\ref{apdx:conjoint_gaidr_covariate_shift} for the estimated-density-ratio design.}

\paragraph{Setup.} Let $Q_{\text{\upshape S}}$ and $Q_{\text{\upshape T}}$ denote the source and target populations, respectively, both joint laws of $(\Xvec, \yvec, \zvec)$; we write $\E_{\text{\upshape S}}$ and $\E_{\text{\upshape T}}$ for the corresponding expectations and $Q_{\text{\upshape S},\Xvec}$, $Q_{\text{\upshape T},\Xvec}$ for the marginal laws of $\Xvec$. The dataset is drawn from the accessible source $Q_{\text{\upshape S}}$, but the parameter of interest is now defined under the target population,
\begin{equation}
\betavec^*_{\text{\upshape T}}\,\in\,\arg\min_{\betavec \in \mB}\, \E_{\text{\upshape T}}\!\left[\ell(\Xvec, \yvec; \betavec)\right],
\qquad\text{equivalently}\qquad
\E_{\text{\upshape T}}\!\left[\Xvec^{\top}\!\left(\nabla b\big(\Xvec\betavec^*_{\text{\upshape T}}\big) - \yvec\right)\right]\,=\,0,
\label{eq:foc_shift_body}
\end{equation}
the target analog of the first-order condition \eqref{eq:foc}. We make two assumptions throughout, stated here in prose. First, \emph{covariate shift only}: the conditional law of $(\yvec, \zvec)$ given $\Xvec$ is common to the source and target populations, which therefore differ \emph{only} in the marginal distribution of $\Xvec$. This is the defining restriction of the covariate-shift setting and guarantees, in particular, that the AI mechanism $\gvec^*(\Xvec, \zvec) = \E[\yvec \mid \Xvec, \zvec]$ is the same function under both populations. Second, \emph{strict overlap}: $Q_{\text{\upshape T},\Xvec}$ is absolutely continuous with respect to $Q_{\text{\upshape S},\Xvec}$ with a bounded density ratio
$$
r(\Xvec)\,:=\,\frac{dQ_{\text{\upshape T},\Xvec}}{dQ_{\text{\upshape S},\Xvec}}(\Xvec)\,\leq\,M < \infty
\qquad Q_{\text{\upshape S}}\text{-almost surely},
$$
so that every region of the target market is represented in the source sample. Together these two conditions yield the change-of-measure identity $\E_{\text{\upshape T}}[h(\Xvec, \yvec, \zvec)] = \E_{\text{\upshape S}}[r(\Xvec)\, h(\Xvec, \yvec, \zvec)]$, which is the workhorse of the analysis below. As in the no-shift case, we maintain the regularity conditions of Assumption~\ref{ass:regularity} under $Q_{\text{\upshape S}}$ together with $\E_{\text{\upshape T}}[\Xvec^{\top}\Xvec] \succ 0$, which identify $\betavec^*_{\text{\upshape T}}$ as the unique target minimizer. Two quantities recur in both subsections below,
so for compactness of notation, we define the target prediction-residual and human-label-noise terms
\begin{equation}
\zetavec_{\text{\upshape T}}(\Xvec, \zvec)\,:=\,\Xvec^{\top}\!\left(\nabla b\big(\Xvec\betavec^*_{\text{\upshape T}}\big) - \gvec^*(\Xvec, \zvec)\right)
\quad\text{and}\quad
\pivec(\Xvec, \yvec, \zvec)\,:=\,\Xvec^{\top}\!\left(\gvec^*(\Xvec, \zvec) - \yvec\right),
\label{eq:zeta_pi_body}
\end{equation}
the covariate-shift analogs of $\zetavec$ and $\pivec$ in Section~\ref{sec:theory}, with $\betavec^*$ replaced by $\betavec^*_{\text{\upshape T}}$ in $\zetavec_{\text{\upshape T}}$ ($\pivec$ is unchanged). By construction $\E[\pivec \mid \Xvec, \zvec] = 0$, while $\zetavec_{\text{\upshape T}}$ is the model-misspecification/AI-representation term that GAI eliminates on the unlabeled observations.

\subsection{Known Density Ratio}
\label{sec:covshift_known}

When the target covariate distribution is available from external sources---census tabulations, platform-wide logs, or a stratified sampling design as in Online Appendix~\ref{apdx:conjoint_gaiw_covariate_shift}---the density ratio $r(\Xvec)$ is known, and covariate shift is handled by a single change of measure. We retain the partial-labeling structure of the main paper under the random-labeling design (Assumption~\ref{ass:random_labeling}, with known $\rho$), and simply reweight the GAI score by $r(\Xvec)$:
$$
\Psivec_{\text{\upshape w}}(\Xi; \gvec; \betavec)\,:=\,r(\Xvec)\,\Psivec(\Xi; \rho, \gvec; \betavec),
$$
where $\Psivec$ is the GAI score \eqref{eq:score} with $e \equiv \rho$. The estimator $\hat{\betavec}_{\text{\upshape w}}$ \rev{(which we refer to as GAI-W)} is computed exactly as in Algorithm~\ref{alg: gai}, with $\hat{e}^{(k)} \equiv \rho$ and every observation's score (and Jacobian) multiplied by the known weight $r(\Xvec_i)$; the nuisance $\hat{\gvec}$ is cross-fitted on the source primary observations as before. Because $r(\Xvec)$ enters as a fixed multiplier and not through any nuisance, the entire argument is identical to the no-shift case, with $r$-reweighting throughout. Consequently, with $\Jvec_{\text{\upshape T}} := \E_{\text{\upshape T}}[\Xvec^{\top}\nabla^2 b(\Xvec\betavec^*_{\text{\upshape T}})\Xvec]$, the weighted GAI estimator is $\sqrt{n}$-asymptotically normal for the target parameter $\betavec^*_{\text{\upshape T}}$, requiring only $L_2$-consistency of $\hat{\gvec}$, and it strictly dominates the weighted primary-only estimator (the solution to $\sum_{i:\,w_i=1} r(\Xvec_i)\,\nabla_{\betavec}\ell(\Xvec_i, \yvec_i; \betavec) = 0$) whenever $\rho < 1$ and
\rev{the weighted mean-level residual $r(\Xvec)\,\Xvec^{\top}\big(\nabla b(\Xvec\betavec^*_{\text{\upshape T}}) - \E[\yvec \mid \Xvec, \zvec]\big)$ is not almost surely confined to a proper subspace (the dominance condition displayed in Theorem~\ref{thm:normality_shift})}---with the same variance structure as Theorem~\ref{thm:normality} and Corollary~\ref{cor:dominance} \rev{(stated in Online Appendix~\ref{apdx: cor dominance})}, the efficiency gain again coming entirely from eliminating the $\zetavec_{\text{\upshape T}}$ term (now weighted by $r(\Xvec)^2$) on the unlabeled observations while the irreducible noise $\pivec$ contributes identically to both. We defer the formal combined statement and its proof to Online Appendix~\ref{apdx: cs known} (Theorem~\ref{thm:normality_shift}). {Online Appendix~\ref{apdx:conjoint_gaiw_covariate_shift} implements this known-density-ratio construction in the vaccine conjoint application.}

\subsection{Doubly Robust Estimation with an Estimated Density Ratio}
\label{sec:covshift_dr}

Outside of designed sampling schemes the density ratio is rarely known and must be estimated, which is the genuinely new methodological challenge. We work with a natural two-sample design. The first is a fully labeled \emph{source sample} $\mD_{\text{\upshape S}} = \{(\Xvec_i, \yvec_i, \zvec_i)\}_{i=1}^{n_{\text{\upshape S}}}$ of i.i.d.\ draws from $Q_{\text{\upshape S}}$ (the $\rho = 1$ specialization of the previous subsection). The second is an independent \emph{target covariate sample} $\mD_{\text{\upshape T}} = \{(\Xvec_j, \zvec_j)\}_{j=1}^{n_{\text{\upshape T}}}$, with $\Xvec_j \sim Q_{\text{\upshape T},\Xvec}$ and $\zvec_j$ obtained by querying the AI mechanism on $\Xvec_j$; the human label $\yvec$ is never observed in $\mD_{\text{\upshape T}}$. Such a target covariate sample is realistic precisely when human labels are not: an experiment eliciting $\yvec$ can be run only on an accessible sub-population, but the platform already records the covariates $\Xvec$ of its entire user base---the target market---and can cheaply generate the AI output $\zvec$ for each. We study the regime $n_{\text{\upshape S}}, n_{\text{\upshape T}} \to \infty$ with $n_{\text{\upshape S}}/n_{\text{\upshape T}} \to \gamma \in [0, \infty)$, and write $r^*$ for the true ratio and $\hat{r}$ for its estimate.

The naive fix---plugging $\hat{r}$ into the weighted score $\hat{r}(\Xvec)\,\Psivec(\Xi; 1, \hat{\gvec}; \betavec)$ on the source sample---fails, because it is not Neyman-orthogonal with respect to $\hat{r}$: the directional derivative of the population score in the ratio is $\E_{\text{\upshape S}}[\Delta_r(\Xvec)\,\zetavec_{\text{\upshape T}}(\Xvec, \zvec)] \neq 0$ in general (with $\Delta_r := \hat{r} - r^*$), so under misspecification first-order error in $\hat{r}$ contaminates the estimate of $\betavec^*_{\text{\upshape T}}$ and invalidates $\sqrt{n_{\text{\upshape S}}}$-inference. The two-sample structure suggests the remedy. The only role of $r(\Xvec)$ in the weighted score is to convert a source average of the prediction component into a target average; with target covariates in hand, that component can instead be evaluated \emph{directly} on $\mD_{\text{\upshape T}}$, with no reweighting at all. This yields the doubly robust score
\begin{equation}
\bm{\Psi}_{\text{\upshape DR}}(\gvec, r; \betavec)\,:=\,\underbrace{\frac{1}{n_{\text{\upshape T}}}\sum_{j=1}^{n_{\text{\upshape T}}} \Xvec_j^{\top}\!\left(\nabla b(\Xvec_j\betavec) - \gvec(\Xvec_j, \zvec_j)\right)}_{\text{\upshape prediction component (target sample)}} \;+\; \underbrace{\frac{1}{n_{\text{\upshape S}}}\sum_{i=1}^{n_{\text{\upshape S}}} r(\Xvec_i)\,\Xvec_i^{\top}\!\left(\gvec(\Xvec_i, \zvec_i) - \yvec_i\right)}_{\text{\upshape bias correction (source sample)}},
\label{eq:score_dr_body}
\end{equation}
and the estimator $\hat{\betavec}_{\text{\upshape DR}}$ \rev{(which we refer to as GAI-DR)} as any approximate root of $\bm{\Psi}_{\text{\upshape DR}}(\hat{\gvec}, \hat{r}; \cdot)$. The prediction component replaces the unobserved target labels by the AI-based prediction $\gvec(\Xvec_j, \zvec_j)$ and uses no human labels; the bias correction estimates the systematic error of that prediction from the fully labeled source, transported to the target by the weight $r(\Xvec_i)$. The nuisance $\hat{\gvec}$ is cross-fitted on the source sample as in Algorithm~\ref{alg: gai}, while $\hat{r}$ is trained either on an auxiliary source fold (e.g., a probabilistic classifier distinguishing source from target covariates, converted by the standard odds transformation) or on an external sample independent of $\mD_{\text{\upshape S}}$.

The score \eqref{eq:score_dr_body} is Neyman-orthogonal with respect to both nuisances: perturbing $\gvec$, the change in the prediction component, $-\E_{\text{\upshape T}}[\Xvec^{\top}\Delta_{\gvec}]$, is cancelled exactly by the change in the bias correction, $\E_{\text{\upshape S}}[r^*(\Xvec)\,\Xvec^{\top}\Delta_{\gvec}] = \E_{\text{\upshape T}}[\Xvec^{\top}\Delta_{\gvec}]$, by the change of measure; perturbing $r$ changes the population mean by $\E_{\text{\upshape S}}[\Delta_r(\Xvec)\,\pivec(\Xvec, \yvec, \zvec)] = 0$ since $\E[\pivec \mid \Xvec, \zvec] = 0$. Crucially, in contrast to the naive plug-in, the ratio error now multiplies the mean-zero residual $\pivec$ rather than the misspecification term $\zetavec_{\text{\upshape T}}$---hence the estimator is doubly robust, its population mean at $\betavec^*_{\text{\upshape T}}$ remaining zero if either nuisance is correct. Because $\hat{r}$ enters multiplicatively rather than through an inverse weight, only an $L_2$ rate on $\hat{r}$ is needed (no sup-norm condition), and a product-rate condition $\lVert \hat{\gvec} - \gvec^*\rVert_{Q_{\text{\upshape S}},2}\,\lVert \hat{r} - r^*\rVert_{Q_{\text{\upshape S}},2}$ of the usual order suffices; both rates are governed by the source sample size $n_{\text{\upshape S}}$, which carries the human labels and is the binding constraint.

\begin{theorem}[Asymptotic Normality under Covariate Shift: Estimated Density Ratio]
\label{thm:normality_dr_body}
Under the covariate-shift and overlap conditions above, the rate conditions of Assumption~\ref{ass:ml_rate_shift} \rev{(product-rate conditions on $\hat{\gvec}$ and $\hat{r}$ analogous to Assumption~\ref{ass:ml_rate}; stated in Online Appendix~\ref{apdx: cs estimated})}, and $n_{\text{\upshape S}}/n_{\text{\upshape T}} \to \gamma \in [0, \infty)$,
$$
\sqrt{n_{\text{\upshape S}}}\big(\hat{\betavec}_{\text{\upshape DR}} - \betavec^*_{\text{\upshape T}}\big)\,\rightsquigarrow\,N\big(\bm{0}, \Sigmavec_{\text{\upshape DR}}\big),
\qquad
\Sigmavec_{\text{\upshape DR}}\,=\,\Jvec_{\text{\upshape T}}^{-1}\, \bm{\Omega}_{\gamma}\, \Jvec_{\text{\upshape T}}^{-1},
$$
where $\Jvec_{\text{\upshape T}} = \E_{\text{\upshape T}}[\Xvec^{\top}\nabla^2 b(\Xvec\betavec^*_{\text{\upshape T}})\Xvec]$ and
$$
\bm{\Omega}_{\gamma}\,=\,\gamma\, \E_{\text{\upshape T}}\!\left[\zetavec_{\text{\upshape T}}(\Xvec,\zvec)\,\zetavec_{\text{\upshape T}}(\Xvec,  \zvec)^{\top}\right] + \E_{\text{\upshape S}}\!\left[r^*(\Xvec)^2\, \pivec(\Xvec, \yvec, \zvec)\,\pivec(\Xvec, \yvec, \zvec)^{\top}\right].
$$
\end{theorem}

The two terms of $\bm{\Omega}_{\gamma}$ correspond to the two independent samples: the first is the sampling variability of the prediction component over the target sample (carrying the factor $\gamma$ because that component averages over $n_{\text{\upshape T}}$ observations while the normalization is $\sqrt{n_{\text{\upshape S}}}$), and the second is the reweighted irreducible human-label noise from the source---the $\rho = 1$ analog of the $\pivec$-term in the known-ratio case. The proof is in Online Appendix~\ref{apdx: cs estimated} (Theorem~\ref{thm:normality_dr_body}). 
We note that in this setting there is no natural benchmark for a dominance result that is parallel to Corollary~\ref{cor:dominance} or Theorem~\ref{thm:normality_shift}, because the primary-data-only (or source-data-only) estimator is biased and there is no natural rescue to it because the density ratio is unknown. \cl{Online Appendix~\ref{apdx:conjoint_gaidr_covariate_shift} implements this estimated-density-ratio approach in the vaccine conjoint application.}

\section{Experiments}
\label{sec:experiments}

We evaluate GAI across three empirical applications that span materially different auxiliary data regimes. We deliberately order these applications so that the conditions become progressively more favorable for PPI-based methods, making GAI's sustained advantage increasingly difficult to attribute to any single feature of the data. Moreover, these applications demonstrate GAI's relevance to the central estimation problems that marketing researchers face.

First, a \emph{vaccine conjoint analysis} (Section~\ref{sec: conjoint}) uses high-dimensional LLM embeddings derived from near-random predictions (54\% accuracy) with no extra information ($\rev{\yvec} \perp \zvec \mid \Xvec$). This is the most demanding setting for any augmentation method and illustrates GAI's unique ability to extract signal from unstructured AI outputs that PPI-based methods cannot even accept as input. Second, a \emph{retail pricing study} (Section~\ref{sec:pricing}) places all methods on a level playing field: every estimator receives the same binary AI prediction. Since the digital twin is trained on persona features beyond price, $\zvec$ provides extra predictive information ($\rev{\yvec} \not\perp \rev{\zvec} \mid \Xvec$), yet GAI's advantage stems from its methodological approach, not from a data advantage. We also illustrate the hyperparameter selection procedure (Section~\ref{sec:model_selection}). Third, a \emph{health insurance census analysis} (Section~\ref{sec: census}) evaluates GAI on PPI's home turf: well-calibrated ML predictions with 85\% accuracy and extra information ($\rev{\yvec} \not\perp \rev{\zvec} \mid \Xvec$), using the dataset and AI representations constructed by the PPI authors themselves.

\paragraph{Evaluation Metrics.} Across all three applications, we evaluate each method along two complementary dimensions: \emph{point estimation} and \emph{inference quality}.
\begin{enumerate}[leftmargin=*]
    \item \textbf{Point estimation}, measured by mean absolute percentage error (MAPE) relative to ground-truth parameters $\betavec^*$:
    \begin{equation}
        \text{MAPE}\,=\,\frac{100}{d} \sum_{j=1}^{d} \frac{|\hat{\beta}_j - \beta^*_j|}{|\beta^*_j| + c},
    \end{equation}
    where $d$ is the number of parameters. Following prior work \citep{wang2024llm}, we set $c=1$ in the conjoint and pricing studies to stabilize ratios when $|\beta^*_j|$ is close to zero. To be consistent with \cite{angelopoulos2023prediction}, we set $c=0$ in the census study. MAPE summarizes the average percentage deviation of the estimated coefficients from their true values, with lower values indicating more accurate estimation.
    \item \textbf{Inference quality}, measured by three complementary metrics: (a) \emph{confidence interval coverage}---the empirical frequency with which 95\% confidence intervals contain the true parameter value, where valid inference requires coverage near the nominal 95\% level; (b) \emph{confidence interval width}---the average width of 95\% confidence intervals, where narrower intervals indicate more precise inference but only if coverage remains valid; and (c) \emph{decision errors}---cases where inference leads to incorrect managerial decisions\footnote{We define such errors in two ways: (i) CIs that exclude the true parameter value and whose lower and upper bounds both indicate the opposite sign of the true effect; and (ii) CIs that exclude the true value but span zero, thereby leading managers to incorrectly conclude an effect is of the opposite sign or not significant.}.
\end{enumerate}

\paragraph{Benchmark Methods.} We compare GAI against four benchmark approaches that represent the main alternatives available to practitioners:
\begin{enumerate}[leftmargin=*]
    \item \textbf{Primary}: Maximum likelihood estimation using only human-labeled data, ignoring all AI-generated information. This approach provides valid inference but may lack precision when human labels are scarce.
    \item \textbf{Naive}: Maximum likelihood estimation that pools human and AI labels without correction, treating AI outputs as if they were human labels. This approach is simple but can introduce substantial bias when AI predictions differ systematically from human judgments.
    \item \textbf{PPI}: Prediction-powered inference approach of \citet{angelopoulos2023prediction}, treating AI outputs as noisy proxies for outcomes and constructs confidence intervals that adapt to prediction quality.
    \item \textbf{PPI++}: A computationally efficient variant of PPI \citep{angelopoulos2023ppi++} that tunes a shrinkage parameter to optimize interval width.
\end{enumerate}

Before presenting the detailed results, we highlight the main empirical findings that emerge consistently across all three applications. GAI delivers substantial improvements in point estimation accuracy, \rev{reducing MAPE relative to the strongest debiased benchmark (PPI++) by roughly 20--45\% in the vaccine and pricing studies and by 50--80\% in the census study}. These gains translate into dramatic label savings: in most applications, GAI with the smallest primary sample outperforms Primary with the largest. Regarding inference, GAI's confidence intervals attain coverage at or above the nominal 95\% level in nearly every configuration, whereas PPI++ frequently undercovers (as low as 83\% in the census study). Despite this higher coverage, GAI's intervals are always better or at least comparable in width to those of PPI++. Decision quality follows the same pattern---GAI achieves the lowest decision error rates \rev{among methods with valid coverage} across all studies, often by a wide margin. In other words, GAI achieves substantially better coverage and decision quality without meaningful sacrifice in interval precision. The consistency of these findings across three diverse applications---spanning different auxiliary data formats, information structures, and prediction accuracy levels---demonstrates the broad applicability of our approach.

\subsection{Vaccine Conjoint Analysis}
\label{sec: conjoint}

This application presents the most demanding test for AI-augmented estimation: the LLM's discrete predictions are barely better than a coin flip (54\% accuracy), and the AI observes no information beyond the covariates already in the model ($\rev{\yvec} \perp \zvec \mid \Xvec$). Any gains must therefore arise from GAI's ability to extract useful structure from the AI's outputs---either discrete labels or high-dimensional embeddings of the reasoning process. This setting showcases GAI's capability to leverage AI outputs in multiple forms, even when PPI-based methods struggle to utilize them effectively.

Conjoint analysis is one of the most widely used methodologies in marketing research, providing the foundation for product design, pricing, and positioning decisions \citep{green1990conjoint,louviere2000stated,bradlow2004modeling}. Respondents repeatedly choose among product alternatives described by attribute bundles, and researchers infer the underlying part-worth utilities governing these choices. However, conjoint studies can be expensive to administer, while the precision of part-worth estimates depends directly on sample size. This application explores whether flexible AI-generated outputs can augment scarce human responses to improve preference estimation.

\textbf{Data.} Following \citet{wang2024llm}, we utilize a vaccine conjoint experiment dataset with 1,971 respondents from \citet{kreps2020factors}. Each task presents two hypothetical vaccines with 11 attributes and lets respondents select their preferred option. Each respondent was asked to take five tasks in the survey. We estimate the benchmark ``ground truth'' parameters $\betavec^* \in \mathbb{R}^{11}$ by fitting the same logistic choice model to the full human dataset. We exclude respondents who never select any vaccine option, as many public LLMs are restricted from choosing an opt-out option when generating auxiliary labels; this exclusion is applied consistently both when constructing $\betavec^*$ and when forming experimental subsamples.


\textbf{AI Representations.} For each choice task, \citet{wang2024llm} prompts LLMs using CoT reasoning to generate detailed deliberation about the vaccine choice before making a prediction, recording the decision as labels $z \in \{-1,0,1\}$ (abstain, option 1, or option 2).\footnote{See \citet{wang2024llm} for full details on the AI representation construction and replication materials.} We evaluate multiple auxiliary generators that vary in the underlying model. Across these configurations, we obtain consistent, qualitatively similar results; for brevity, we report GPT-4o in the main text and defer the full set of robustness results to Online Appendix~\ref{apdx: conjoint_embedding_z_full}. We consider two forms of auxiliary representation: (1) \emph{Labels}: the discrete choice prediction $z \in \{-1,0,1\}$, which can be used by both GAI and PPI-based methods; and (2) \emph{Embeddings}: the full LLM response text converted into a vector representation using OpenAI's text-embedding-3-large model, yielding $\zvec \in \mathbb{R}^{3072}$ that capture the semantic content of the LLM's reasoning. The embedding encodes not just the predicted choice, but also the underlying reasoning patterns and attribute considerations that the LLM employs.

As explained, GPT-4o achieves low accuracy in this setting, barely better than random guessing in a binary choice task, and the LLM observes only the same vaccine attributes that comprise $\Xvec$, which implies $\rev{\yvec} \perp \zvec \mid \Xvec$. Any efficiency gains must therefore arise from GAI's sample expansion and representational power (Source (I) and (II) in Corollary~\ref{cor:variance_decomposition}) rather than from extra information. PPI-based methods struggle in this setting: while they can use the discrete labels, the low prediction accuracy limits their effectiveness, and they cannot directly utilize the high-dimensional embeddings.

\textbf{Experimental Design.} We construct repeated experiments by drawing a primary sample of size $n_P \in \{50,100,150,200\}$ and an auxiliary sample of size $n_A=1000$ from the available task pool.\footnote{We randomly sample $n_P/5$ respondents and $n_A/5$ respondents and use all their five choices to compose the primary set and auxiliary set.} For each $(n_P,n_A)$ configuration, we run 50 independent trials. Within each trial, the underlying random draw of observations is held fixed across methods to ensure fair comparison. All reported statistics are averaged over the 50 trials.

\textbf{GAI Implementation.} We implement GAI with both auxiliary representations using 5-fold cross-fitting. For GAI (Embeddings), given the high dimensionality of the embedding space ($d_z = 3{,}072$), we use strong $\ell_2$ regularization: logistic regression with ridge penalty (regularization strength $C=0.01$, maximum 2{,}000 iterations). The embeddings are standardized to zero mean and unit variance before model training. For GAI (Labels), we use $\ell_2$-regularized logistic regression with $C=0.05$.
In each fold, we fit $\gvec(\Xvec,\zvec)$ on four-fifths of the primary data and generate out-of-sample predictions for the held-out primary fold as well as for all auxiliary observations. This procedure is repeated across the five folds, and the resulting predictions for each auxiliary observation $\hat{\gvec}(\Xvec,\zvec)$ are averaged. Because auxiliary data are never used in training, all predictions for auxiliary observations are strictly out of sample. Predictions for primary observations are also out of sample by construction due to cross-fitting.
Since primary and auxiliary observations are randomly sampled from the same underlying population, we set $\hat e(\Xvec,\zvec)$ to the constant $n_P / (n_P + n_A)$.

Under the MNL model, we then estimate $\hat\betavec$ by solving a single optimization problem that minimizes the GAI-adjusted cross-entropy loss:
\[
\hat{\betavec}
\,=\,\arg\min_{\betavec}
\frac{1}{n}\sum_{i=1}^{n}
\left[
-\sum_{j=1}^{J} \hat{\tau}_{ij}\log p_{ij}(\betavec)
\right],\]
where
\[
\hat{\tau}_{ij}
\,=\,\hat{\gvec}(\Xvec_i, \zvec_i)\!\left(1 - \frac{w_i}{\hat e}\right)
+ \frac{w_i y_{ij}}{\hat e},\footnote{We clip and normalize $\hat{\tau}_i$ to derive well-defined probability for numerical stability.} \quad \text{ and } \quad p_{ij}(\betavec)
\,=\,
\frac{\exp(\Xvec_{ij}^\top \betavec)}{\sum_{\ell=1}^{J}\exp(\Xvec_{i\ell}^\top \betavec)}.
\]
This estimator coincides with the general procedure described in Algorithm~\ref{alg: gai}. Finally, we compute standard errors according to Theorem \ref{thm:normality}.

\textbf{Main Results.} Table~\ref{tab: conjoint_compare} reports estimation accuracy and inference quality across four primary sample sizes, with results averaged over 50 independent trials.

\begin{table}[t]
\caption{Benchmark Comparison for Conjoint Analysis: MAPEs, coverage probabilities, and CI widths}
\label{tab: conjoint_compare}
\vskip 0.1in
\centering
\small
\renewcommand{\arraystretch}{1}
\setlength{\tabcolsep}{3.5pt}
\begin{tabular}{@{}l|cccc|cccc|cccc@{}}
\toprule
& \multicolumn{4}{c|}{(a) MAPE (\%)} & \multicolumn{4}{c|}{(b) 95\% CI Coverage (\%)} & \multicolumn{4}{c}{(c) 95\% CI Width} \\
\cmidrule(lr){2-5} \cmidrule(lr){6-9} \cmidrule(l){10-13}
Method & 50 & 100 & 150 & 200 & 50 & 100 & 150 & 200 & 50 & 100 & 150 & 200 \\
\midrule
Primary & 32.02 & 25.27 & 19.67 & 19.01 & 97.82 & 92.73 & 92.55 & 88.00 & 2.43 & 1.51 & 1.19 & 1.01 \\
Naive & 48.52 & 45.75 & 43.33 & 40.29 & 26.55 & 26.36 & 26.36 & 27.09 & 0.56 & 0.53 & 0.51 & 0.49 \\
PPI & -- & 44.57 & 32.96 & 29.21 & 99.82 & 98.00 & 95.82 & 92.91 & 9.28 & 2.71 & 1.94 & 1.58 \\
PPI++ & -- & 29.69 & 22.96 & 21.00 & 94.55 & 89.09 & 90.91 & 85.09 & 2.44 & 1.52 & 1.19 & 1.00 \\
GAI (Labels) & 16.86 & 17.52 & 15.97 & 16.64 & 99.82 & 97.09 & 94.73 & 90.91 & 1.94 & 1.36 & 1.11 & 0.96 \\
GAI (Embeddings) & 16.50 & 17.23 & 15.73 & 16.24 & 99.45 & 98.55 & 96.91 & 94.55 & 2.11 & 1.50 & 1.21 & 1.05 \\
\bottomrule
\end{tabular}
\vskip 0.05in
\footnotesize
\textbf{Notes}: A ``–" symbol indicates cases where the value exceeds 1,000, which occurs because PPI-based methods can suffer from singularity problems in small primary samples.
\vskip -0.1in
\end{table}

\emph{Point Estimation.} The most striking result is the magnitude of GAI's advantage. Panel (a) shows that both GAI variants achieve MAPE of 16--17\% across all sample sizes---nearly flat, indicating that both labels and embeddings provide sufficient signal to stabilize estimates even with very few human labels. Between the two representations, GAI (Embeddings) achieves consistently lower MAPE than GAI (Labels) across all sample sizes, suggesting that the richer semantic content in embeddings provides additional signal for point estimation. By contrast, the Primary estimator ranges from 32\% at $n_P=50$ down to 19\% at $n_P=200$, reflecting its dependence on sample size. The practical implication is dramatic: GAI with just 50 labels (16.5\% MAPE) outperforms Primary with 200 labels (19.0\% MAPE), representing more than 75\% reduction in required human annotations. Naive pooling performs worst (40--49\% MAPE), confirming that treating inaccurate LLM predictions as ground-truth labels is counterproductive. PPI-based methods encounter numerical instability with small samples (denoted ``--'' in the table), and when stable, their MAPEs remain substantially above GAI. Paired $t$-tests confirm that all improvements are statistically significant (Online Appendix~\ref{apdx: tests_conjoint}).

\emph{Inference Quality.} Panels (b) and (c) together reveal that GAI's estimation gains do not come at the cost of inference quality---if anything, GAI produces the most reliable confidence intervals. Coverage (Panel b) shows that GAI (Embeddings)
\rev{maintains high and stable coverage (94.6--99.5\%)---the highest of all methods for $n_P \geq 100$---at or near the nominal level except for a mild dip to 94.6\% at $n_P=200$, while GAI (Labels) is near-nominal at small samples but undercovers mildly as $n_P$ grows (94.7\% at $n_P=150$ and 90.9\% at $n_P=200$)}. In contrast, PPI++ undercovers throughout (85--95\%), and Primary drops to 88\% at $n_P=200$. Naive pooling fails (26--27\%). Regarding interval width (Panel c), GAI (Labels) produces the narrowest intervals among valid methods (1.94 at $n_P=50$, narrowing to 0.96 at $n_P=200$). A potential reason why GAI (Labels) yields tighter intervals than GAI (Embeddings) is that the discrete label has better representational power than the high-dimensional embeddings. GAI (Embeddings) maintains comparable widths to PPI++, while PPI produces extremely wide intervals at small samples (9.28 at $n_P=50$), reflecting overly conservative inference. This comparison illustrates a trade-off: embeddings yield smaller MAPE and better coverage while labels achieve narrower intervals, with both GAI variants outperforming PPI-based methods. Paired $t$-tests confirm these patterns (Online Appendix~\ref{apdx: tests_conjoint}). In terms of decision quality, GAI (Embeddings) achieves the lowest decision error rate---averaging 2.6\% compared to 6.9\% for Primary, 3.2\% for PPI, 8.7\% for PPI++, and 20.4\% for Naive (Online Appendix~\ref{apdx: conjoint_decision}).

The source of GAI's advantage in this setting is its ability to leverage AI outputs as informative features. The LLM's CoT reasoning generates both discrete choice labels and rich text that can be converted to high-dimensional embeddings. The embeddings capture semantic information about attribute trade-offs and decision heuristics, yielding superior coverage, while the discrete labels---despite their simplicity---enable tighter confidence intervals, potentially because they provide a more direct representation of the outcome structure. 
Both representations allow GAI to extract predictive signal without requiring the LLM's predictions to be accurate surrogates for human choices. Additional robustness checks across different LLMs are reported in Online Appendix~\ref{apdx: conjoint_embedding_z_full}.
Finally, we mention that GPT-4o's reasoning power has been surpassed by more recent releases (e.g., GPT-5.3 at the time of writing), and new models are coming out frequently with much stronger reasoning capabilities. As a result, in practical scenarios we suspect that such embedding-based performance can be even stronger than what is presented in this empirical study.

\subsection{Retail Pricing Study}
\label{sec:pricing}

Having established that GAI can extract signal from unstructured, near-random AI outputs, we now ask a sharper question: when every method receives the \emph{same} auxiliary information, does GAI's feature-based approach still outperform the surrogate-label strategy of PPI? This application provides a controlled comparison on a level playing field.

We consider the problem of estimating a logistic demand model that relates purchase decisions to price. The AI purchase predictions exhibit systematic bias (30\% vs.\ 44\% actual purchase rate), and we restrict all methods to the same binary AI prediction to ensure a fair comparison.

Estimating price sensitivity is a central task in quantitative marketing with demand models, as it is the key element in pricing decisions, revenue management, and competitive strategy \citep{berry1993automobile, vulcano2012estimating, kok2007demand}. However, obtaining reliable purchase data at the individual level requires costly price experiments or extensive observational data \citep{guadagni1983logit, ban2021personalized, chen2022privacy}. This application demonstrates how GAI can improve demand estimation by leveraging AI-generated purchase predictions as auxiliary features, even when these predictions are systematically biased.

\textbf{Data.} We use the data and pricing experiment from \cite{toubia2025database}, which contains 2,058 participants who completed over 500 questions spanning diverse domains including demographics, personality traits, values, and behavioral preferences. Each participant was presented with 40 products at randomized prices and asked to indicate whether they would purchase each item. We evaluate our methods on all products and observe consistent results. For conciseness, we report the results of the household cleaning product in the main text and show the comprehensive performance across products in Online Appendix~\ref{apdx: pricing_all_products}. The outcome model uses only price as the covariate: $\Xvec = [\mathbf{1}, \text{price}] \in \mathbb{R}^{n \times 2}$, yielding a two-parameter logistic regression ($d=2$). The ground-truth parameters $\betavec^* \in \mathbb{R}^{2}$ are estimated by fitting a logistic regression to the full dataset. The empirical purchase rate is 44\%, and the AI predictions agree with actual outcomes 67\% of the time.

\textbf{AI Representations.} The AI purchase labels $z \in \{0,1\}$ are generated by digital twins in \cite{toubia2025database}---LLM-based simulations of individual participants constructed from their survey responses.\footnote{See \citet{toubia2025database} for full details on the AI representation construction and replication materials.} For each participant, a persona profile is created by concatenating their answers to the 500+ survey questions, capturing demographics, personality traits, values, and stated preferences. This profile is then provided as context to an LLM, which is prompted to simulate whether that specific individual would purchase each product at the given price. The AI prediction rate (30.07\%) substantially underestimates the true purchase rate (44.10\%), indicating systematic bias. Since the digital twin is trained on persona and demographic features that are not included in the outcome model's covariates $\Xvec$, the AI prediction carries extra predictive information: $\rev{\yvec} \not\perp \rev{\zvec} \mid \Xvec$.

In our main analysis, we restrict the auxiliary information to the binary AI purchase prediction: $z \in \{0, 1\}$. This ensures that both GAI and PPI-based methods have access to identical information, isolating the effect of the estimation strategy. In Online Appendix~\ref{apdx: pricing_hdim}, we extend to high-dimensional auxiliary information $\zvec \in \mathbb{R}^{115}$ including persona and demographic features.

\textbf{Experimental Design.} We draw primary samples of size $n_P \in \{100, 200, 300, 400, 500\}$ with a fixed auxiliary sample of $n_A=1{,}000$. As in the conjoint study, we run 50 independent trials per configuration with the same random draw held fixed across methods.

\textbf{GAI Implementation.} We use 5-fold cross-fitting with model selection to estimate the nuisance function $\gvec(\Xvec,z)$. 
Since primary and auxiliary observations are randomly sampled from the same underlying population, we set $\hat e(\Xvec,z)$ to the constant $\hat e = n_P / (n_P + n_A)$. 

\textbf{Model Selection.} This application illustrates the cross-validation procedure described in Section~\ref{sec:model_selection}. We estimate $\gvec(\Xvec,z)$ by searching over 8 configurations spanning three model classes: regularized logistic regression (GLM with L1/L2 penalties and $C \in \{0.01, 0.1\}$), Random Forest (depth $\in \{5, 10\}$), and LightGBM (learning rate $\in \{0.05, 0.1\}$). Within each of the 5 outer folds, we perform cross-validation to select the best configuration using only the data from the remaining 4 folds. This nested structure ensures that model selection does not compromise the independence required for valid inference. Across trials, regularized logistic regression configurations are most frequently selected (particularly L1-penalized GLM with $C=0.1$), followed by Random Forest, with selection patterns varying by sample size and fold.

\textbf{Main Results.} Results in Table~\ref{tab: pricing_compare} are across five primary sample sizes, averaged over 50 trials.

\begin{table}[t]
\caption{Benchmark Comparison for Pricing Study: MAPEs, coverage probabilities, and CI widths}
\label{tab: pricing_compare}
\vskip 0.1in
\centering
\small
\renewcommand{\arraystretch}{1}
\setlength{\tabcolsep}{2.8pt}
\begin{tabular}{@{}l|ccccc|ccccc|ccccc@{}}
\toprule
& \multicolumn{5}{c|}{(a) MAPE (\%)} & \multicolumn{5}{c|}{(b) 95\% CI Coverage (\%)} & \multicolumn{5}{c}{(c) 95\% CI Width} \\
\cmidrule(lr){2-6} \cmidrule(lr){7-11} \cmidrule(l){12-16}
Method & 100 & 200 & 300 & 400 & 500 & 100 & 200 & 300 & 400 & 500 & 100 & 200 & 300 & 400 & 500 \\
\midrule
Primary & 22.36 & 16.03 & 12.47 & 9.59 & 9.53 & 95.00 & 97.00 & 95.00 & 99.00 & 92.00 & 1.17 & 0.82 & 0.66 & 0.57 & 0.51 \\
Naive & 22.22 & 21.97 & 19.99 & 19.67 & 19.43 & 50.00 & 50.00 & 50.00 & 50.00 & 50.00 & 0.37 & 0.35 & 0.33 & 0.32 & 0.31 \\
PPI & 26.91 & 15.31 & 14.10 & 13.98 & 12.88 & 95.00 & 98.00 & 97.00 & 94.00 & 96.00 & 1.40 & 1.01 & 0.85 & 0.75 & 0.69 \\
PPI++ & 22.15 & 14.16 & 11.56 & 9.86 & 9.66 & 94.00 & 99.00 & 93.00 & 99.00 & 92.00 & 1.10 & 0.78 & 0.64 & 0.55 & 0.49 \\
GAI & 12.43 & 9.14 & 8.74 & 6.57 & 7.60 & 100.00 & 100.00 & 99.00 & 100.00 & 96.00 & 1.16 & 0.81 & 0.65 & 0.56 & 0.50 \\
\bottomrule
\end{tabular}
\vspace{-0.5em}
\end{table}

\emph{Point Estimation.} Panel (a) confirms that GAI outperforms all benchmarks even when given the same information. GAI achieves MAPE of 6.6--12.4\%, compared to 9.5--22.4\% for Primary, 9.7--22.2\% for PPI++, and 12.9--26.9\% for PPI. Naive pooling plateaus around 20\% regardless of sample size, reflecting persistent bias. The label savings remain substantial: GAI with 100 labels (12.4\% MAPE) matches Primary with 300 labels (12.5\% MAPE), a 67\% reduction in annotation cost. All improvements are statistically significant by paired $t$-tests (Online Appendix~\ref{apdx: tests_pricing}). Consistent with these accuracy gains, Figure~\ref{fig: pricing_param_dist} shows that GAI yields parameter estimates that are more tightly concentrated around the ground truth for both the intercept and price coefficient, whereas PPI and PPI++ exhibit substantially greater dispersion.

\begin{figure*}[t]
  \centering
  \includegraphics[width=0.5\textwidth]{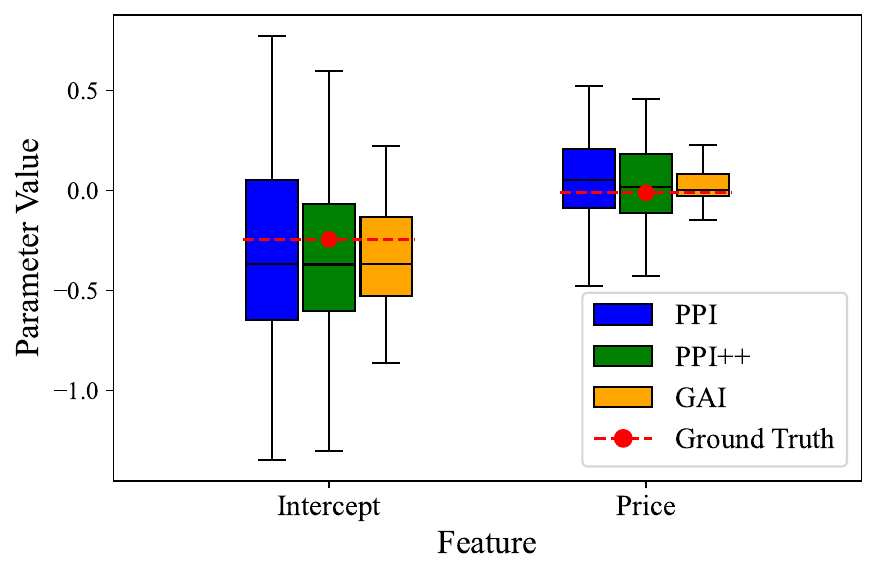}
  \caption{\centering Distribution of Estimates (PPI, PPI++, GAI) for Pricing Study ($n_P = 100$)}
  \label{fig: pricing_param_dist}
\end{figure*}

\emph{Inference Quality.} GAI's confidence intervals maintain coverage at or above 96\% across all sample sizes (Panel b), reaching 100\% at $n_P \in \{100, 200, 400\}$. PPI and PPI++ also maintain near-nominal coverage (92--99\%), while Naive severely undercovers at 50\%. Interval widths (Panel c) are comparable across GAI, Primary, and PPI++, while PPI intervals are noticeably wider. Paired $t$-tests indicate that PPI++ achieves modestly narrower intervals in this study, while GAI (Labels) produces significantly narrower intervals than PPI++ in the conjoint analysis (Table~\ref{tab: conjoint_compare}); overall, the two methods generate comparable CI widths with neither consistently dominating. GAI achieves a 0.2\% decision error rate, \rev{the lowest among methods with valid coverage; Naive's nominally lower 0.0\% rate is an artifact of its severe undercoverage (50\%), whose biased but narrow intervals rarely trigger a sign or significance error} (Online Appendix~\ref{apdx: pricing_slip}).

The key insight from this study is that GAI's advantage is methodological, not informational. Both GAI and PPI receive the same covariate vector $\Xvec = (1, \text{price})^\top$ and the same binary AI prediction $z$. Yet GAI substantially outperforms PPI-based methods because it models the conditional expectation $\gvec(\Xvec, z) = \E[\yvec \mid \Xvec, z]$, extracting information from the AI prediction without assuming it is a noisy proxy for the outcome. In Online Appendix~\ref{apdx: pricing_hdim}, we show that GAI's advantages extend further when $z \in \mathbb{R}^{115}$ includes persona and demographic features. Online Appendix~\ref{apdx: pricing_scalar_z} provides a robustness check confirming that GAI still outperforms benchmarks even when the AI representation does not provide extra information beyond the covariates.

\subsection{Health Insurance Census Analysis}
\label{sec: census}

Our final application tests GAI on PPI's home turf: well-calibrated, continuous AI predictions with high accuracy, using the dataset and AI representations constructed by the PPI authors \citep{angelopoulos2023prediction}. This is the regime most favorable to PPI-based methods; if GAI holds up here as well, practitioners need not switch estimators when moving to well-calibrated AI outputs.

Understanding how interventions and individual characteristics predict consumer decisions, such as insurance coverage and product adoption, is fundamental to marketing research and policy evaluation \citep{sahni2016advertising, huang2020social, lu2025optimizing, rafieian2025multiobjective}. Although this application uses observational census data rather than a field experiment, the statistical problem is representative of the analysis stage in market research where treatment effects are estimated from costly experimental or observational data \citep{ascarza2016perils, feit2019test, gao2023field, bojinov2023design}.

\textbf{Data.} We use the California census healthcare dataset with $318,\!215$ individuals to investigate the relationship between private health insurance coverage ($y \in \{0,1\}$) and income using a logistic model. Ground truth of logistic regression coefficients $\betavec^* \in \mathbb{R}^{2}$ is estimated using all data.

\textbf{AI Representations.} Auxiliary predictions are produced by \citet{angelopoulos2023prediction} using a gradient boosting classifier trained on a richer feature set including income, race, gender, and other covariates.\footnote{See \citet{angelopoulos2023prediction} for full details on the AI representation construction and replication materials.} The model outputs predicted probabilities $z = P(y=1 \,|\, \text{all features})$ with 85\% accuracy. Since the classifier uses features beyond the income covariate in $\Xvec$, the AI prediction provides extra information: $\rev{\yvec} \not\perp \rev{\zvec} \mid \Xvec$. Unlike the conjoint analysis, this representation is well-calibrated.

\textbf{GAI Implementation.} Because $z$ already represents an accurate estimate of the conditional mean based on richer features, learning an additional $\gvec(\Xvec,z)$ offers limited benefit. We therefore set $\gvec(\Xvec,z)=z$ directly, and $e(\Xvec,z)$ to the constant $n_P / (n_P + n_A)$. This yields a simplified version of our estimator requiring no first-stage nuisance estimation or cross-fitting.

\textbf{Experimental Design.} We vary $n_P \in \{100, 250, 500, 750, 1000\}$ with $n_A=2000$, again running 50 trials per configuration.

\begin{table}[t]
\caption{Benchmark Comparison for Census Analysis: MAPEs, coverage probabilities, and CI widths}
\label{tab: census_compare}
\vskip 0.1in
\centering
\small
\renewcommand{\arraystretch}{1}
\setlength{\tabcolsep}{2.8pt}
\begin{tabular}{@{}l|ccccc|ccccc|ccccc@{}}
\toprule
& \multicolumn{5}{c|}{(a) MAPE (\%)} & \multicolumn{5}{c|}{(b) 95\% CI Coverage (\%)} & \multicolumn{5}{c}{(c) 95\% CI Width ($10^{-5}$)} \\
\cmidrule(lr){2-6} \cmidrule(lr){7-11} \cmidrule(l){12-16}
Method & 100 & 250 & 500 & 750 & 1000 & 100 & 250 & 500 & 750 & 1000 & 100 & 250 & 500 & 750 & 1000 \\
\midrule
Primary & 979.83 & 563.85 & 480.99 & 320.68 & 291.80 & 80.00 & 78.00 & 78.00 & 82.00 & 78.00 & 2.16 & 1.31 & 0.92 & 0.74 & 0.65 \\
Naive & 407.95 & 384.70 & 324.56 & 263.13 & 240.38 & 30.00 & 31.00 & 34.00 & 38.00 & 38.00 & 0.61 & 0.58 & 0.53 & 0.49 & 0.46 \\
PPI & 821.83 & 518.82 & 383.92 & 295.40 & 270.66 & 88.00 & 94.00 & 95.00 & 96.00 & 94.00 & 2.48 & 1.85 & 1.29 & 1.10 & 0.96 \\
PPI++ & 866.27 & 508.75 & 404.75 & 299.52 & 282.43 & 83.00 & 88.00 & 91.00 & 93.00 & 89.00 & 2.24 & 1.70 & 1.24 & 1.02 & 0.89 \\
GAI & 161.37 & 147.93 & 147.38 & 139.65 & 141.49 & 100.00 & 100.00 & 100.00 & 99.00 & 99.00 & 2.64 & 1.79 & 1.28 & 1.08 & 0.93 \\
\bottomrule
\end{tabular}
\end{table}

\textbf{Main Results.} Table~\ref{tab: census_compare} summarizes performance across five primary sample sizes. The most striking finding is inferential: every other method fails to maintain nominal coverage, yet GAI attains 99--100\% throughout (Panel b). Primary and Naive exhibit severe undercoverage (30--82\%), and PPI++ ranges from 83\% to 93\%---consistently below the nominal 95\% level. Importantly, GAI achieves this coverage without inflating its intervals: widths are statistically indistinguishable from PPI++ at every sample size (Panel c; Online Appendix~\ref{apdx: tests_census}). Point estimation tells a similar story (Panel a): GAI achieves MAPE of 140--161\%, while Primary exceeds 290\% even at $n_P=1{,}000$ and PPI-based methods remain above 270\%. All pairwise MAPE differences are statistically significant (Online Appendix~\ref{apdx: tests_census}). The label savings are the most dramatic of our three studies---GAI with 100 labels (161\%) surpasses Primary with 1,000 labels (292\%), a more than 90\% reduction. GAI also achieves zero decision errors across all configurations, compared to 4.8\% for Primary, 3.0\% for PPI, 4.2\% for PPI++, and 15.8\% for Naive (Online Appendix~\ref{apdx: census_decision}).

Taken together, the three applications demonstrate that GAI's advantages are robust across auxiliary data regimes: high-dimensional embeddings with $\rev{\yvec} \perp \zvec \mid \Xvec$ (conjoint), systematically biased discrete predictions with $\rev{\yvec} \not\perp \rev{\zvec} \mid \Xvec$ (pricing), and well-calibrated continuous probabilities with $\rev{\yvec} \not\perp \rev{\zvec} \mid \Xvec$ (census). Whether the AI output is an embedding, a binary label, or a probability; whether it carries extra information or not; and whether the prediction is accurate or near-random---GAI consistently delivers substantial improvements in estimation accuracy and decision quality. It does so while maintaining valid coverage and comparable interval width to the best-performing benchmark. Additional implementation details and robustness checks are provided in Online Appendix~\ref{apdx: details}.

\section{Conclusion}
\label{sec:conclusion}

This paper introduces GAI, a framework that enables marketing researchers and practitioners to use AI-generated data for statistical estimation while preserving inferential validity. The central insight is conceptual: AI outputs should be treated as \emph{informative features} for predicting human-generated outcomes, rather than as \emph{surrogate labels} that replace them. Formalized through Neyman orthogonality, this perspective supports valid inference across a wide range of AI output quality, from systematically biased LLM predictions or reasoning texts to well-calibrated forecasts.

Our theory shows that GAI is asymptotically normal with a closed-form variance expression, allowing standard confidence intervals and hypothesis tests. \mx{More importantly, under random selection into labeling, GAI \emph{dominates} a unified class of debiased estimators—including primary-only estimation and state-of-the-art debiasing methods—and delivers strict improvements
\rev{under a mild informativeness condition (\cref{thm:dominance}(iii))}. }\cl{In settings where the labeled sample is unrepresentative of the population of interest, GAI can be extended to improve efficiency relative to weighted primary-only estimation. }These gains arise through three channels: sample expansion, representational power of AI, and extra predictive information from AI. Our empirical results across three applications provide strong evidence to support these predictions. 

GAI helps resolve a core tension in data-driven marketing research: precise estimation requires large labeled samples, but human-generated ground truth is costly. By combining a limited number of human labels with abundant AI-generated data, GAI reduces this burden while maintaining statistical rigor. Across all applications, GAI consistently matches the precision of substantially larger human-labeled samples, often reducing labeling needs by more than half. Such gains can translate into substantial cost savings in settings such as conjoint analysis, demand estimation, and field experimentation. The same tension extends beyond these core marketing applications: platform operations rely on expert human judgment to identify fraudulent reviews and content \citep{luca2016fake}, while clinical trials for healthcare interventions can cost tens of thousands of dollars per patient \citep{anderer2022adaptive}. More broadly, the framework applies wherever human labels are scarce or expensive and AI-generated signals are readily available. In each case, GAI provides a principled way to combine expensive human judgments with inexpensive AI-generated information.

One limitation is that integrating GAI into real-time decision systems may create computational challenges. Developing online versions of the framework is therefore a promising direction for future research.

More broadly, AI outputs create a statistical setting that is neither standard missing data nor noisy labeling, but one with auxiliary representations available for all observations. This shifts the inferential challenge from bias correction toward efficient information extraction. Human labels remain the gold standard, but AI-generated signals can substantially reduce the amount of human input needed to achieve a desired level of precision. As AI capabilities continue to advance, frameworks with rigorous guarantees for human-AI data integration will become increasingly important for credible empirical research in marketing and beyond.



\bibliographystyle{informs2014} 
\bibliography{references} 


\newpage

%
%
%
\section*{Online Appendices}
\begin{APPENDICES}
\section{Distinction between GAI and PPI}
\label[app]{apdx: example_ppi}

It is instructive to compare our approach with PPI \citep{angelopoulos2023prediction,angelopoulos2023ppi++}, which further highlights the generality of GAI for AI-data augmentation. The comparison reveals two key differences.
First, PPI treats AI outputs as surrogate labels for $\yvec$ and therefore cannot be directly applied when $\zvec$ is categorically different from $\yvec$---precisely the setting that motivates our framework, as discussed in the introduction.

Second, PPI’s confidence-interval construction and statistical guarantees require $\zvec$ to be a deterministic function based solely on $\Xvec$. However, AI-generated data differs in two key ways: (i) $\zvec$ may be inherently stochastic due to the internal randomization of AI models; and \textit{more fundamentally}, (ii) $\zvec$ may depend on predictive information beyond $\Xvec$, as in our retail pricing experiment with digital twins (see Example \ref{example:correlated} for details). In such cases, the PPI-based methods do not apply and might only inflate variance relative to using human labels alone, whereas GAI, by Corollary~\ref{cor:dominance}, is guaranteed to weakly dominate the primary-only estimator. 

We construct the following example illustrating the key distinction between GAI and PPI. Let $X$ be a binary random variable such that $\Pe(X=1) = \Pe(X=-1) = \tfrac{1}{2}$. Let $U$ be a hidden random variable with $\E[U]=0$ and $\mathrm{Var}(U)=\sigma^2$. $X$ and $U$ are independent. We assume that $y$ is governed by both $X$ and the hidden random variable as $y = XU$. We have $\E[y] = \E[X]\E[U] = 0$ and $\mathrm{Var}(y) = \mathrm{Var}(X)\mathrm{Var}(U) = \sigma^2$.
Suppose that the AI model fails to correctly specify the rule on $X$, but captures the hidden variable $U$, outputting $z = U$. We observe two independent datasets: $\mD^{\text{\upshape P}} = \{(X_i, y_i, z_i)\}_{i=1}^{n_P}$ i.i.d. and $\mD^{\text{\upshape A}} = \{(X_j, z_j)\}_{j=1}^{n_A}$ i.i.d..
Suppose our goal is to estimate the mean of $y$. Let $n_P/n = \rho$, and write $n := n_P$ and $N := n_A$ for the primary and auxiliary sample sizes. Following \cite{angelopoulos2023ppi++}, given a fixed constant $\lambda \in \mathbb{R}$, the PPI mean estimator can be written as
\[
\hat{\theta}^{\sf PPI}(\lambda)
\,=\,\frac{1}{n_P}\sum_{i=1}^{n_P} y_i 
+ \lambda \left( \frac{1}{N}\sum_{j=1}^{n_A} z_j - \frac{1}{n_A}\sum_{i=1}^{n_A} z_i \right).
\]
As $n \to \infty$, it is easy to verify that
\[
\sqrt{n}\, (\hat{\theta}^{\sf PPI}(\lambda) - 0)
\;\xrightarrow{d}\; 
\mathcal{N}\!\left(0,\; \frac{\sigma^2}{\rho} \left(1 + \frac{\lambda^2}{1 -\rho}\right)\right).
\]
Meanwhile, we know that the human data-only estimator is $\hat{\theta}^{\mathrm{P}} = \frac{1}{n_P}\sum_{i=1}^n y_i$, with
\[
\sqrt{n}({\hat{\theta}^{\mathrm{P}}} - 0)
\;\xrightarrow{d}\;
\mathcal{N}\!\left(0,\;\frac{\sigma^2}{\rho}\right)\,.
\]
The basic PPI corresponds to the case of $\lambda = 1$, which inflates the variance from the human data-only estimator. PPI++ optimizes $\lambda$ by minimizing the asymptotic variance, yielding $\lambda^* = 0$ (see Example 6.1 of \cite{angelopoulos2023ppi++}.)
As a result, PPI++ has no improvement over the human data-only estimator, as it is essentially discarding all auxiliary samples.

Under this setup, the AI-signal correlates weakly with the actual human label because it fails to capture the entire mechanism behind $y$ (missing $X$). However, it provides significant information on inferring actual human label by recovering the hidden variable $U$. PPI fails to utilize this information. In contrast, GAI operates in a fundamentally different way. Using GAI, we can easily recover the relationship between $X$ and $y$, obtaining $g^*(X,z) = Xz$ with the primary set. Then, the GAI estimator is $\hat{\theta}^{\mathrm{GAI}} = \frac{1}{n}\left(\sum_{i=1}^{n_P} y_i+\sum_{j=1}^{n_A} g^*(X_j,z_j)\right) =  \frac{1}{n}\left(\sum_{i=1}^{n_P} y_i+\sum_{j=1}^{n_A} y_j\right)$, where
\[
\sqrt{n}(\hat{\theta}^{\mathrm{GAI}} - \theta^*)
\;\xrightarrow{d}\,\mathcal{N}\!\left(0,\;\sigma^2\right)\,,\]
which always reduces the asymptotic variance from the human data-only estimator and efficiently uses all data samples.

\vspace{5pt}
\section{Auxiliary Example for Section \ref{sec:data-generation}}
\label{apdx: aux_example}

\begin{example}[$\yvec \perp \zvec \mid \Xvec$: Off-the-Shelf LLM Generation]
\label{example:independent}

This example follows the setting of \citet{wang2024llm}. 
They consider a conjoint analysis framework in which $\Xvec$ denotes the feature vectors of a choice set with $K$ alternatives, $\yvec$ represents the consumer's choice, and $\zvec$ is the AI-generated choice. 
The model corresponds to $b(\thetavec)=\log\!\big(1+\sum_{j=1}^{K}\exp(\theta_j)\big)$, yielding a standard MNL formulation.
While similar to Example~\ref{example:correlated}, $\zvec$ may in principle be correlated with $\yvec$ conditional on $\Xvec$, for example when the AI system is queried using additional customer persona or behavioral data. 
However, in many applications it is practically convenient to use LLMs in an off-the-shelf fashion. 
In such cases both $\yvec$ and $\zvec$ are generated solely from $\Xvec$, implying $\yvec \perp \zvec \mid \Xvec$. 
Consequently, $\zvec$ does not contain additional information about $\yvec$ beyond what is already present in $\Xvec$.

Despite this conditional independence, $\zvec$ can still be useful for estimation. 
If the AI output has strong representational power and provides a highly informative summary of $\Xvec$, then a simple parameterization such as
\begin{align*}
g_{(j)}(\Xvec, \zvec)
\,=\,
\frac{\exp\!\left({\betavec^{*\top}}\bx_{(j)} + \eta [\zvec = j]\right)}
{1 + \sum_{\ell \in \mathcal{K}}
\exp\!\left({\betavec^{*\top}}\bx_{(\ell)} + \eta [\zvec = \ell]\right)}
\,=\,g_j(\Xvec, \zvec; {\betavec^*}),
\quad \forall j\,\in\,\mathcal{K},
\end{align*}
as proposed in Online Appendix~A of \citet{wang2024llm}, can approximate $\Pe(y=j\mid \Xvec,\zvec)$ well.

To see the intuition, suppose $\zvec = \phi(\Xvec)$ where $\phi(\cdot)$ is a complex nonlinear mapping implicitly learned by the LLM. 
Recovering an equivalent representation using $\Xvec$ alone would require estimating
\begin{align*}
\frac{\exp\!\left({\betavec^{*\top}}\bx_{(j)} + \eta\, [\phi(\Xvec)=j]\right)}
{1 + \sum_{\ell \in \mathcal{K}}
\exp\!\left({\betavec^{*\top}}\bx_{(\ell)} + \eta\, [\phi(\Xvec)=\ell]\right)},
\quad \forall j\,\in\,\mathcal{K},
\end{align*}
which may be difficult in practice. 
One might attempt to approximate $\phi(\cdot)$ using flexible methods such as polynomial or spline expansions. 
However, with limited human-labeled data this approach may lead to substantial model misspecification or overfitting on the small labeled sample $\mD^{\text{\upshape P}}$, making it difficult to extrapolate $\yvec$ to the full dataset $\mD$. 
In such cases the resulting performance may not be substantially better than using $\mD^{\text{\upshape P}}$ alone and ignoring the auxiliary sample.

Although this example is deliberately stylized, the same intuition extends to more general settings—particularly when the AI signal $\zvec$ takes more flexible or high-dimensional forms. 
In such cases AI-generated features can act as rich summaries of latent structure in $\E[\yvec\mid\Xvec]$, improving estimation even when $\yvec \perp \zvec \mid \Xvec$.
\hfill $\blacksquare$

\end{example}


\section{Extension to Covariate Shift}
\label{apdx: covariate shift theory}

This appendix complements Section~\ref{sec:covshift}, which develops the extension of GAI to covariate shift in the main body and states the doubly robust theorem (Theorem~\ref{thm:normality_dr_body}) and the doubly robust score \eqref{eq:score_dr_body} there. Here we record the material deferred from the main body: the formal covariate-shift setting and assumptions (Section~\ref{apdx: cs setting}), and the combined known-density-ratio result that establishes both $\sqrt{n}$-asymptotic normality and variance dominance for the weighted GAI estimator (Section~\ref{apdx: cs known}). All proofs---of the combined known-ratio theorem below and of the main-body doubly robust theorem---are collected in \rev{Section}~\ref{apdx: supporting covshift}.

\subsection{Setting and Assumptions}
\label{apdx: cs setting}

Let $Q_{\text{\upshape S}}$ and $Q_{\text{\upshape T}}$ denote the source and target populations, respectively; both are joint laws of $(\Xvec, \yvec, \zvec)$. The dataset $\mathcal{D} = \{\Xi_i = (\Xvec_i, \yvec_i, w_i, \zvec_i)\}_{i=1}^n$ consists of i.i.d.\ draws from $Q_{\text{\upshape S}}$, with human labels partially observed according to the labeling indicator $w$, exactly as in Section~\ref{sec:data-generation}. We write $\E_{\text{\upshape S}}$ and $\E_{\text{\upshape T}}$ for expectations under $Q_{\text{\upshape S}}$ and $Q_{\text{\upshape T}}$, and $Q_{\text{\upshape S},\Xvec}$ and $Q_{\text{\upshape T},\Xvec}$ for the corresponding marginal laws of $\Xvec$. With slight abuse of notation, $\E_{\text{\upshape S}}$ also denotes expectation over the source data-generating process including the labeling indicator, i.e., over the joint law of $(\Xvec, \yvec, \zvec, w)$ with $(\Xvec, \yvec, \zvec) \sim Q_{\text{\upshape S}}$ and $w$ as in Section~\ref{sec:data-generation}. We let
$r(\Xvec) := \frac{dQ_{\text{\upshape T},\Xvec}}{dQ_{\text{\upshape S},\Xvec}}(\Xvec)$
denote the density ratio of the covariate distributions; we always display the argument and write $r(\Xvec)$, so that no confusion with the rates $r_1, r_2$ in Assumption~\ref{ass:ml_rate} arises. Analogously to $\lVert\cdot\rVert_{Q,2}$ in the main text, given a multi-dimensional function $\mathbf{f}$ we write $\lVert \mathbf{f} \rVert_{Q_{\text{\upshape S}},2} := \left(\E_{\text{\upshape S}}[\lVert \mathbf{f} \rVert^2]\right)^{1/2}$ and $\lVert \mathbf{f} \rVert_{Q_{\text{\upshape T}},2} := \left(\E_{\text{\upshape T}}[\lVert \mathbf{f} \rVert^2]\right)^{1/2}$.

The parameter of interest is the target parameter $\betavec^*_{\text{\upshape T}}$ defined in \eqref{eq:foc_shift_body}, which, under Assumption~\ref{ass:covariate_shift} below and by the same argument as for \eqref{eq:foc}, satisfies the first-order condition
\begin{equation}
\E_{\text{\upshape T}}\left[\Xvec^{\top}\left(\nabla b\big(\Xvec\betavec^*_{\text{\upshape T}}\big) - \yvec\right)\right]\,=\,0.
\label{eq:foc_shift}
\end{equation}
Following the proof of Corollary~\ref{cor:dominance} (\rev{Section}~\ref{apdx: proof of dominance}), we define
$$
\zetavec_{\text{\upshape T}}(\Xvec, \zvec)\,:=\,\Xvec^{\top}\left(\nabla b\big(\Xvec\betavec^*_{\text{\upshape T}}\big) - \gvec^*(\Xvec, \zvec)\right)
\quad\text{and}\quad
\pivec(\Xvec, \yvec, \zvec)\,:=\,\Xvec^{\top}\left(\gvec^*(\Xvec, \zvec) - \yvec\right);
$$
these are the covariate-shift analogs of $\zetavec$ and $\pivec$ defined there, with $\betavec^*$ replaced by $\betavec^*_{\text{\upshape T}}$ in $\zetavec_{\text{\upshape T}}$ ($\pivec$ is unchanged); they coincide with the quantities $\zetavec_{\text{\upshape T}}$ and $\pivec$ introduced in \eqref{eq:zeta_pi_body}. Our analysis relies on the following assumption, which formalizes the two conditions stated in prose in Section~\ref{sec:covshift}.

\begin{assumption}[Covariate Shift Setting]
\label{ass:covariate_shift}
(i) \emph{(Covariate shift only)} The conditional law of $(\yvec, \zvec)$ given $\Xvec$ is identical under $Q_{\text{\upshape S}}$ and $Q_{\text{\upshape T}}$; the two populations differ only in the marginal distribution of $\Xvec$.
(ii) \emph{(Strict overlap)} $Q_{\text{\upshape T},\Xvec}$ is absolutely continuous with respect to $Q_{\text{\upshape S},\Xvec}$, and $r(\Xvec) \leq M < \infty$, $Q_{\text{\upshape S}}$-almost surely.
(iii) \emph{(Homogeneous labeling)} $w$ is independent of $(\Xvec, \yvec, \zvec)$ with $\Pe(w = 1) = \rho \in (0, 1]$, and $\rho$ is known.
(iv) \emph{(Target identification)} Assumption~\ref{ass:regularity} holds under $Q_{\text{\upshape S}}$; in addition, $\E_{\text{\upshape T}}[\Xvec^{\top}\Xvec] \succ 0$ and $\betavec^*_{\text{\upshape T}} \in \mB$ solves the target population problem $\min_{\betavec \in \mB} \E_{\text{\upshape T}}\left[\ell(\Xvec, \yvec; \betavec)\right]$.
\end{assumption}

We record several direct consequences of Assumption~\ref{ass:covariate_shift} that we use repeatedly. First, under item (i), $\gvec^*(\Xvec, \zvec) = \E[\yvec \mid \Xvec, \zvec]$ is the same function under both populations, and items (i)--(ii) yield the change-of-measure identity
\begin{equation}
\E_{\text{\upshape T}}\left[h(\Xvec, \yvec, \zvec)\right]\,=\,\E_{\text{\upshape S}}\left[r(\Xvec)\, h(\Xvec, \yvec, \zvec)\right]
\label{eq:change_of_measure}
\end{equation}
for any $Q_{\text{\upshape T}}$-integrable function $h$. Second, item (ii) implies $\supp(Q_{\text{\upshape T},\Xvec}) \subseteq \supp(Q_{\text{\upshape S},\Xvec}) \subseteq \Breve{\mX}$, where $\Breve{\mX}$ is the compact set containing $\mX$ introduced in Section~\ref{sec:framework}; hence the boundedness of $\Xvec$ over $\Breve{\mX}$ extends to the target population without further conditions; likewise, the bound $\lVert \Cov(\yvec \mid \Xvec, \zvec) \rVert \leq \tilde{\sigma}^2$ in Assumption~\ref{ass:regularity}(iii) transfers to the target population automatically by items (i)--(ii), and we do not re-assume it. Similarly, the square-integrability of $\yvec$ in Assumption~\ref{ass:regularity}(iv), which holds under $Q_{\text{\upshape S}}$ by item (iv), extends to the target population, since $\E_{\text{\upshape T}}\big[\lVert \yvec \rVert^2\big] = \E_{\text{\upshape S}}\big[r(\Xvec)\lVert \yvec \rVert^2\big] \leq M\, \E_{\text{\upshape S}}\big[\lVert \yvec \rVert^2\big]$ by \eqref{eq:change_of_measure} and item (ii). Third, under item (iv), $\betavec^*_{\text{\upshape T}}$ is the unique minimizer of the target population problem: strict convexity of $\E_{\text{\upshape T}}\left[\ell(\Xvec, \yvec; \betavec)\right]$ follows from $\E_{\text{\upshape T}}[\Xvec^{\top}\Xvec] \succ 0$ by the same argument as in the proof of Lemma~\ref{lem: empirical-scores}.

We also record the rate conditions used by the doubly robust theorem of Section~\ref{sec:covshift} (Theorem~\ref{thm:normality_dr_body}), the counterpart of Assumption~\ref{ass:ml_rate}. To avoid any confusion with the density ratio $r^*(\Xvec)$ and with the rates $r_1, r_2$ of Assumption~\ref{ass:ml_rate}, we denote the rate exponents by $r_g$ and $r_r$, and write $r^*(\Xvec)$ for the true density ratio (the object denoted $r(\Xvec)$ above) and $\hat{r}$ for its estimator.

\begin{assumption}[ML Convergence Rates for Doubly Robust Estimation]
\label{ass:ml_rate_shift}
There exist $\alpha(n) \downarrow 0$ and $r_g, r_r \geq 0$ with $r_g + r_r \geq 1/2$ such that
$$
\left\lVert \hat{\gvec} - \gvec^* \right\rVert_{Q_{\text{\upshape S}},2}\,\leq\,\alpha(n_{\text{\upshape S}})/n_{\text{\upshape S}}^{r_g}
\qquad\text{and}\qquad
\left\lVert \hat{r} - r^* \right\rVert_{Q_{\text{\upshape S}},2}\,\leq\,\alpha(n_{\text{\upshape S}})/n_{\text{\upshape S}}^{r_r}.
$$
\end{assumption}

Two features of Assumption~\ref{ass:ml_rate_shift} are worth highlighting. First, unlike the condition on $\hat{e}$ in Assumption~\ref{ass:ml_rate}, no sup-norm consistency condition on $\hat{r}$ is required: $\hat{r}$ enters the doubly robust score \eqref{eq:score_dr_body} multiplicatively rather than through an inverse weight, so its error never appears in a denominator and an $L_2$ rate suffices. Second, both rates are indexed by the source sample size $n_{\text{\upshape S}}$, even though $\hat{r}$ may also use the (typically much larger) target sample: estimating $\gvec^*$ requires the source labels, and the precision of $\hat{r}$ is limited by how well the source covariate distribution $Q_{\text{\upshape S},\Xvec}$ can be learned from $n_{\text{\upshape S}}$ draws, no matter how precisely $Q_{\text{\upshape T},\Xvec}$ is known. Finally, no separate rate condition under the target population is needed: by the change-of-measure identity \eqref{eq:change_of_measure} and Assumption~\ref{ass:covariate_shift}(ii),
$$
\left\lVert \Delta_{\gvec} \right\rVert_{Q_{\text{\upshape T}},2}^2\,=\,\E_{\text{\upshape S}}\left[r^*(\Xvec)\left\lVert \Delta_{\gvec} \right\rVert^2\right]\,\leq\,M \left\lVert \Delta_{\gvec} \right\rVert_{Q_{\text{\upshape S}},2}^2,
$$
so the source-population rate for $\hat{\gvec}$ transfers automatically to the target population; this is a derived fact, not an additional assumption.

\subsection{Known Density Ratio}
\label{apdx: cs known}

We first give the formal combined statement for the known-density-ratio case discussed in Section~\ref{sec:covshift_known}. When the density ratio $r(\Xvec)$ is known---for example when the target covariate distribution is available from external sources such as census tabulations or platform-wide logs, as in the stratified design of \rev{Section}~\ref{apdx:conjoint_gaiw_covariate_shift}---we retain the partial-labeling structure under the random labeling design (Assumption~\ref{ass:covariate_shift}(iii), with known $\rho$) and define the weighted GAI score
\begin{equation}
\Psivec_{\text{\upshape w}}(\Xi; \gvec; \betavec)\,:=\,r(\Xvec)\, \Xvec^{\top}\left[\nabla b(\Xvec\betavec) - \gvec(\Xvec, \zvec) + \frac{w}{\rho}\left(\gvec(\Xvec, \zvec) - \yvec\right)\right]\,=\,r(\Xvec)\, \Psivec(\Xi; \rho, \gvec; \betavec),
\label{eq:score_shift}
\end{equation}
where $\Psivec$ is the GAI score in \eqref{eq:score} with $e \equiv \rho$. The estimation procedure is Algorithm~\ref{alg: gai} with two modifications: (a) because the labeling probability is known, we set $\hat{e}^{(k)} \equiv \rho$ rather than estimating it; and (b) the score of each observation is multiplied by the known weight $r(\Xvec_i)$, as in \eqref{eq:score_shift}; correspondingly, in the variance-estimation step the Jacobian estimate is also weighted, $\hat{\Jvec}_{\text{\upshape w}} = \frac{1}{n}\sum_{k=1}^{K}\sum_{i \in I_k} r(\Xvec_i)\, \Xvec_i^{\top} \nabla^2 b\big(\Xvec_i\hat{\betavec}_{\text{\upshape w}}\big) \Xvec_i$. The nuisance $\hat{\gvec}^{(k)}$ is cross-fitted on the source primary observations exactly as in Algorithm~\ref{alg: gai}. The weighted GAI estimator $\hat{\betavec}_{\text{\upshape w}}$ then satisfies
$$
\left\lVert \frac{1}{n}\sum_{k=1}^{K}\sum_{i \in I_k} \Psivec_{\text{\upshape w}}\big(\Xi_i; \hat{\gvec}^{(k)}; \hat{\betavec}_{\text{\upshape w}}\big) \right\rVert
\,\leq\,
\inf_{\betavec \in \mB} \left\lVert \frac{1}{n}\sum_{k=1}^{K}\sum_{i \in I_k} \Psivec_{\text{\upshape w}}\big(\Xi_i; \hat{\gvec}^{(k)}; \betavec\big) \right\rVert + o_P(n^{-1/2}).
$$
Define $\Jvec_{\text{\upshape T}} := \E_{\text{\upshape S}}\left[r(\Xvec)\, \Xvec^{\top} \nabla^2 b\big(\Xvec\betavec^*_{\text{\upshape T}}\big) \Xvec\right] = \E_{\text{\upshape T}}\left[\Xvec^{\top} \nabla^2 b\big(\Xvec\betavec^*_{\text{\upshape T}}\big) \Xvec\right]$, where the equality follows from \eqref{eq:change_of_measure}. The natural human-data benchmark reweighted to the target population is the \emph{weighted primary-only} estimator $\hat{\betavec}^{\text{\upshape P},\text{\upshape w}}$, a solution to
\begin{equation}
\sum_{i:\, w_i = 1} r(\Xvec_i)\, \nabla_{\betavec}\ell(\Xvec_i, \yvec_i; \betavec)\,=\,0,
\label{eq:primary_score_shift}
\end{equation}
or, equivalently, $\sum_{i=1}^{n} w_i\, r(\Xvec_i)\, \nabla_{\betavec}\ell(\Xvec_i, \yvec_i; \betavec) = 0$. This estimating equation is valid for $\betavec^*_{\text{\upshape T}}$: by Assumption~\ref{ass:covariate_shift}(iii) and \eqref{eq:change_of_measure},
$$
\E_{\text{\upshape S}}\left[w\, r(\Xvec)\, \nabla_{\betavec}\ell\big(\Xvec, \yvec; \betavec^*_{\text{\upshape T}}\big)\right]
\,=\,\rho\, \E_{\text{\upshape S}}\left[r(\Xvec)\, \nabla_{\betavec}\ell\big(\Xvec, \yvec; \betavec^*_{\text{\upshape T}}\big)\right]
\,=\,\rho\, \E_{\text{\upshape T}}\left[\nabla_{\betavec}\ell\big(\Xvec, \yvec; \betavec^*_{\text{\upshape T}}\big)\right]\,=\,0,
$$
where the last equality is the first-order condition \eqref{eq:foc_shift}. The following theorem combines the asymptotic normality of the weighted GAI estimator with its dominance over this benchmark.

\begin{theorem}[Asymptotic Normality and Dominance under Covariate Shift: Known Density Ratio]
\label{thm:normality_shift}
Suppose Assumption~\ref{ass:covariate_shift} holds and $\lVert \hat{\gvec} - \gvec^* \rVert_{Q_{\text{\upshape S}},2} = o(1)$. Then the weighted GAI estimator is $\sqrt{n}$-asymptotically normal for the target parameter $\betavec^*_{\text{\upshape T}}$,
$$
\sqrt{n}\big(\hat{\betavec}_{\text{\upshape w}} - \betavec^*_{\text{\upshape T}}\big)\,\rightsquigarrow\,N\big(\bm{0}, \Sigmavec_{\text{\upshape w}}\big), \qquad \Sigmavec_{\text{\upshape w}}\,=\,\Jvec_{\text{\upshape T}}^{-1}\, \bm{\Omega}_{\text{\upshape w}}\, \Jvec_{\text{\upshape T}}^{-1},
$$
where $\bm{\Omega}_{\text{\upshape w}} := \E_{\text{\upshape S}}\big[\Psivec_{\text{\upshape w}}(\Xi; \gvec^*; \betavec^*_{\text{\upshape T}})\, \Psivec_{\text{\upshape w}}(\Xi; \gvec^*; \betavec^*_{\text{\upshape T}})^{\top}\big]$ admits the decomposition
\begin{equation}
\bm{\Omega}_{\text{\upshape w}}\,=\,\E_{\text{\upshape S}}\left[r(\Xvec)^2\, \zetavec_{\text{\upshape T}}\zetavec_{\text{\upshape T}}^{\top}\right] + \frac{1}{\rho}\, \E_{\text{\upshape S}}\left[r(\Xvec)^2\, \pivec\pivec^{\top}\right].
\label{eqn:omega-w-decomposition}
\end{equation}
Moreover, the weighted primary-only estimator $\hat{\betavec}^{\text{\upshape P},\text{\upshape w}}$ of \eqref{eq:primary_score_shift} is $\sqrt{n}$-asymptotically normal, $\sqrt{n}\big(\hat{\betavec}^{\text{\upshape P},\text{\upshape w}} - \betavec^*_{\text{\upshape T}}\big) \rightsquigarrow N\big(\bm{0}, \Sigmavec^{\text{\upshape P},\text{\upshape w}}\big)$, with
$$
\Sigmavec^{\text{\upshape P},\text{\upshape w}}\,=\,\frac{1}{\rho}\, \Jvec_{\text{\upshape T}}^{-1}\left(\E_{\text{\upshape S}}\left[r(\Xvec)^2\, \zetavec_{\text{\upshape T}}\zetavec_{\text{\upshape T}}^{\top}\right] + \E_{\text{\upshape S}}\left[r(\Xvec)^2\, \pivec\pivec^{\top}\right]\right)\Jvec_{\text{\upshape T}}^{-1},
$$
and $\hat{\betavec}_{\text{\upshape w}}$ dominates it:
\begin{equation}
\Sigmavec^{\text{\upshape P},\text{\upshape w}} - \Sigmavec_{\text{\upshape w}}\,=\,\left(\frac{1}{\rho} - 1\right)\Jvec_{\text{\upshape T}}^{-1}\, \E_{\text{\upshape S}}\left[r(\Xvec)^2\, \zetavec_{\text{\upshape T}}\zetavec_{\text{\upshape T}}^{\top}\right] \Jvec_{\text{\upshape T}}^{-1}\,\succeq\,0,
\label{eqn:variance-gap-shift}
\end{equation}
with $\Sigmavec^{\text{\upshape P},\text{\upshape w}} \succ \Sigmavec_{\text{\upshape w}}$ if and only if $\rho < 1$ and
$$
\E_{\text{\upshape S}}\bigg[r(\Xvec)^2\, \Xvec^{\top}\left(\nabla b\big(\Xvec\betavec^*_{\text{\upshape T}}\big) - \E[\yvec \mid \Xvec, \zvec]\right)\left(\nabla b\big(\Xvec\betavec^*_{\text{\upshape T}}\big) - \E[\yvec \mid \Xvec, \zvec]\right)^{\top}\Xvec\bigg]\,\succ\,0.
$$
\end{theorem}

The proof is given in \rev{Section}~\ref{apdx: proof normality shift}. The decomposition \eqref{eqn:omega-w-decomposition} holds because the cross terms between $\zetavec_{\text{\upshape T}}$ and $\pivec$ vanish: $\E[\pivec \mid \Xvec, \zvec] = 0$, and $w$ is independent of $(\Xvec, \yvec, \zvec)$ with $\E[w/\rho] = 1$ and $\E[w/\rho^2] = 1/\rho$. As in the known-$\rho$ case of Corollary~\ref{cor:dominance} and the discussion following it, the normality conclusion requires only $L_2$-consistency of $\hat{\gvec}$ in place of Assumption~\ref{ass:ml_rate}: since both $\rho$ and $r(\Xvec)$ are known, the remainder terms driven by a propensity estimate vanish identically, so no rate or sup-norm condition is needed. The gap \eqref{eqn:variance-gap-shift} is the exact structural analog of Corollary~\ref{cor:variance_decomposition}: the entire efficiency gain comes from eliminating the misspecification/representation term $\zetavec_{\text{\upshape T}}$ on the unlabeled observations---now weighted by $r(\Xvec)^2$---while the irreducible human-label noise term $\pivec$ contributes identically to both asymptotic variances.

\begin{remark}[No Comparison with Unweighted Primary-Only Estimation]
\label{rem:unweighted_primary}
The \emph{unweighted} primary-only estimator, i.e., a solution to $\sum_{i:\, w_i = 1} \nabla_{\betavec}\ell(\Xvec_i, \yvec_i; \betavec) = 0$ as in Corollary~\ref{cor:dominance}, converges to the source pseudo-true parameter $\betavec^*_{\text{\upshape S}} \in \arg\min_{\betavec \in \mB} \E_{\text{\upshape S}}\left[\ell(\Xvec, \yvec; \betavec)\right]$. Under model misspecification, $\betavec^*_{\text{\upshape S}} \neq \betavec^*_{\text{\upshape T}}$ in general, so the unweighted estimator carries an $O(1)$ asymptotic bias and is inconsistent for the target parameter $\betavec^*_{\text{\upshape T}}$; no meaningful variance comparison with $\hat{\betavec}_{\text{\upshape w}}$ exists. If the GLM were correctly specified---that is, if $\E[\yvec \mid \Xvec] = \nabla b(\Xvec\betavec^{\circ})$ for some $\betavec^{\circ} \in \mB$---then $\betavec^*_{\text{\upshape S}} = \betavec^*_{\text{\upshape T}} = \betavec^{\circ}$ and the issue would disappear; our framework, however, explicitly allows for misspecification.
\end{remark}

\subsection{Estimated Density Ratio}

For the estimated-density-ratio case, the doubly robust estimator, its doubly robust score \eqref{eq:score_dr_body}, the two-sample design, and the asymptotic normality result (Theorem~\ref{thm:normality_dr_body}) are developed in the main body in Section~\ref{sec:covshift_dr}; its proof is given in \rev{Section}~\ref{apdx: cs estimated}.


\section{Detailed Results of Experiments}
\label{apdx: details}

\subsection{Vaccine Conjoint Analysis}

\subsubsection{Tests for Performance Improvement}
\label{apdx: tests_conjoint}

We conduct paired $t$-tests comparing GAI to each benchmark method across repeated trials for each metric (MAPE, coverage, CI width). The paired design controls for sampling variation and isolates estimator differences.

\begin{table}[t]
\caption{Paired t-test p-values (GAI (Embeddings) vs benchmarks) for Conjoint Analysis}
\label{tab:conjoint_ttest_emb}
\vskip 0.1in
\centering
\small
\renewcommand{\arraystretch}{1}
\setlength{\tabcolsep}{2.8pt}
\begin{tabular}{@{}l|cccc|cccc|cccc@{}}
\toprule
& \multicolumn{4}{c|}{(a) MAPE} & \multicolumn{4}{c|}{(b) 95\% CI Coverage} & \multicolumn{4}{c}{(c) 95\% CI Width} \\
\cmidrule(lr){2-5} \cmidrule(lr){6-9} \cmidrule(l){10-13}
Method & 50 & 100 & 150 & 200 & 50 & 100 & 150 & 200 & 50 & 100 & 150 & 200 \\
\midrule
Primary & 5e-33 & 3e-37 & 2e-38 & 2e-40 & 0.02 & 9e-5 & 5e-4 & 2e-4 & 4e-11 & 0.56 & 3e-3 & 5e-8 \\
Naive & 9e-33 & 9e-37 & 6e-38 & 8e-40 & 2e-117 & 1e-94 & 6e-85 & 2e-69 & 1e-41 & 6e-51 & 3e-57 & 3e-57 \\
PPI & 0.31 & 5e-37 & 5e-38 & 3e-40 & 0.31 & 0.48 & 0.31 & 0.26 & 0.02 & 1e-20 & 3e-28 & 2e-30 \\
PPI++ & 0.01 & 4e-37 & 3e-38 & 2e-40 & 7e-5 & 4e-8 & 8e-5 & 1e-7 & 2e-8 & 0.42 & 0.02 & 1e-8 \\
\bottomrule
\end{tabular}
\vskip -0.1in
\end{table}

\begin{table}[t]
\caption{Paired t-test p-values (GAI (Labels) vs benchmarks) for Conjoint Analysis}
\label{tab:conjoint_ttest_lab}
\vskip 0.1in
\centering
\small
\renewcommand{\arraystretch}{1}
\setlength{\tabcolsep}{2.8pt}
\begin{tabular}{@{}l|cccc|cccc|cccc@{}}
\toprule
& \multicolumn{4}{c|}{(a) MAPE} & \multicolumn{4}{c|}{(b) 95\% CI Coverage} & \multicolumn{4}{c}{(c) 95\% CI Width} \\
\cmidrule(lr){2-5} \cmidrule(lr){6-9} \cmidrule(l){10-13}
Method & 50 & 100 & 150 & 200 & 50 & 100 & 150 & 200 & 50 & 100 & 150 & 200 \\
\midrule
Primary & 3e-15 & 2e-12 & 5e-7 & 3e-6 & 3e-3 & 6e-4 & 0.03 & 0.02 & 7e-19 & 5e-17 & 2e-16 & 2e-15 \\
Naive & 6e-46 & 2e-44 & 5e-49 & 3e-40 & 5e-71 & 8e-55 & 3e-49 & 6e-41 & 5e-53 & 3e-64 & 7e-67 & 1e-68 \\
PPI & 0.32 & 2e-18 & 8e-15 & 1e-15 & 1.00 & 0.30 & 0.29 & 0.15 & 0.02 & 2e-22 & 9e-30 & 9e-32 \\
PPI++ & 0.01 & 3e-15 & 3e-11 & 1e-10 & 3e-5 & 3e-7 & 3e-3 & 3e-5 & 2e-14 & 9e-17 & 9e-15 & 2e-15 \\
\bottomrule
\end{tabular}
\vskip -0.1in
\end{table}

Tables~\ref{tab:conjoint_ttest_emb} and \ref{tab:conjoint_ttest_lab} report p-values for GAI (Embeddings) and GAI (Labels), respectively. Both GAI variants achieve significant MAPE improvements over all benchmarks. Coverage differences relative to Primary, Naive, and PPI++ are significant for both variants. For CI width, GAI (Labels) produces significantly narrower intervals than PPI++ across all sample sizes, while GAI (Embeddings) produces narrower intervals than PPI++ at $n_P=50$ but modestly wider intervals at $n_P \in \{150, 200\}$.

\subsubsection{Decision Errors}
\label{apdx: conjoint_decision}

\begin{table}[t]
\caption{Decision Errors for Conjoint Analysis (\% of all CIs)}
\label{tab:conjoint_decision}
\vskip 0.1in
\centering
\small
\renewcommand{\arraystretch}{1}
\setlength{\tabcolsep}{6pt}
\begin{tabular}{@{}lcccc@{}}
\toprule
Method & $n_P$=50 & $n_P$=100 & $n_P$=150 & $n_P$=200 \\
\midrule
Primary & 1.82 & 6.73 & 7.09 & 12.00 \\
Naive & 19.09 & 20.00 & 20.73 & 21.64 \\
PPI & 0.00 & 1.82 & 3.82 & 7.09 \\
PPI++ & 2.91 & 9.09 & 8.18 & 14.55 \\
GAI (Labels) & 0.18 & 2.91 & 5.27 & 9.09 \\
GAI (Embeddings) & 0.55 & 1.45 & 3.09 & 5.45 \\
\bottomrule
\end{tabular}
\vskip -0.1in
\end{table}

Table~\ref{tab:conjoint_decision} reports decision errors. GAI (Embeddings) averages 2.6\%, lower than Primary (6.9\%), PPI (3.2\%), PPI++ (8.7\%), and Naive (20.4\%). GAI (Labels) averages 4.4\%, also substantially lower than most benchmarks, except for PPI. \rev{Although PPI achieves the lowest error rate only at $n_P=50$ (and the lowest average among the non-GAI benchmarks), GAI (Embeddings) attains the lowest average decision-error rate overall; PPI's low rates come at the cost of wider intervals.}

\subsubsection{Robustness to Auxiliary Generator}
\label{apdx: conjoint_embedding_z_full}

Tables~\ref{tab: conjoint_robustness_mape}--\ref{tab: conjoint_robustness_width} report results using different LLMs (GPT-3.5-Turbo, GPT-4, GPT-4o) with CoT prompting, comparing both label-based ($z$=Labels) and embedding-based ($z$=Embeddings) auxiliary representations. GAI consistently achieves the lowest MAPE across all configurations while maintaining near-nominal coverage, demonstrating robustness to both the choice of auxiliary generator and representation type.

\begin{table}[t]
\caption{Benchmark Comparison for Conjoint Analysis: MAPE (\%)}
\label{tab: conjoint_robustness_mape}
\vskip 0.1in
\centering
\footnotesize
\begin{tabular}{@{}lcccccccccccccc@{}}
\toprule
\midrule
 & \multicolumn{6}{c}{$n_P$ = 50} && \multicolumn{6}{c}{$n_P$ = 100} \\
\cmidrule(lr){2-7} \cmidrule(lr){9-14}
 & & & & & \multicolumn{2}{c}{GAI} && & & & & \multicolumn{2}{c}{GAI} \\
\cmidrule(lr){6-7} \cmidrule(lr){13-14}
Model & Primary & Naive & PPI & PPI++ & Lab & Emb && Primary & Naive & PPI & PPI++ & Lab & Emb \\
\midrule
GPT-3.5-Turbo-0613 & 32.02 & 48.83 & -- & -- & 17.23 & 16.51 && 25.27 & 45.89 & -- & -- & 17.61 & 17.06 \\
GPT-3.5-Turbo-0125 & 32.02 & 42.98 & -- & -- & 17.48 & 16.45 && 25.27 & 39.97 & 50.55 & 30.41 & 17.92 & 17.35 \\
GPT-4 & 32.02 & 50.14 & -- & -- & 16.79 & 16.81 && 25.27 & 47.75 & 48.56 & 29.40 & 17.37 & 17.36 \\
GPT-4o & 32.02 & 48.52 & -- & -- & 16.86 & 16.50 && 25.27 & 45.75 & 44.57 & 29.69 & 17.52 & 17.23 \\
\addlinespace[0.4em]
\midrule
 & \multicolumn{6}{c}{$n_P$ = 150} && \multicolumn{6}{c}{$n_P$ = 200} \\
\cmidrule(lr){2-7} \cmidrule(lr){9-14}
 & & & & & \multicolumn{2}{c}{GAI} && & & & & \multicolumn{2}{c}{GAI} \\
\cmidrule(lr){6-7} \cmidrule(lr){13-14}
Model & Primary & Naive & PPI & PPI++ & Lab & Emb && Primary & Naive & PPI & PPI++ & Lab & Emb \\
\midrule
GPT-3.5-Turbo-0613 & 19.67 & 42.73 & 39.02 & 22.60 & 16.35 & 16.62 && 19.01 & 40.10 & 35.08 & 21.20 & 17.22 & 16.87 \\
GPT-3.5-Turbo-0125 & 19.67 & 37.40 & 36.11 & 23.07 & 16.64 & 16.13 && 19.01 & 35.17 & 29.95 & 20.96 & 17.22 & 16.63 \\
GPT-4 & 19.67 & 46.00 & 33.54 & 22.62 & 15.96 & 16.04 && 19.01 & 43.11 & 29.21 & 20.96 & 16.64 & 16.53 \\
GPT-4o & 19.67 & 43.33 & 32.96 & 22.96 & 15.97 & 15.73 && 19.01 & 40.29 & 29.21 & 21.00 & 16.64 & 16.24 \\
\bottomrule
\end{tabular}
\vskip 0.05in
\footnotesize
\textbf{Notes}: All models use CoT prompting. A ``--'' symbol indicates cases where the value exceeds 1,000, which occurs because PPI-based methods can suffer from singularity problems in small primary samples.
\vskip -0.1in
\end{table}

\begin{table}[t]
\caption{Benchmark Comparison for Conjoint Analysis: 95\% CI Coverage Probability (\%)}
\label{tab: conjoint_robustness_coverage}
\vskip 0.1in
\centering
\footnotesize
\begin{tabular}{@{}lcccccccccccccc@{}}
\toprule
\midrule
 & \multicolumn{6}{c}{$n_P$ = 50} && \multicolumn{6}{c}{$n_P$ = 100} \\
\cmidrule(lr){2-7} \cmidrule(lr){9-14}
 & & & & & \multicolumn{2}{c}{GAI} && & & & & \multicolumn{2}{c}{GAI} \\
\cmidrule(lr){6-7} \cmidrule(lr){13-14}
Model & Primary & Naive & PPI & PPI++ & Lab & Emb && Primary & Naive & PPI & PPI++ & Lab & Emb \\
\midrule
GPT-3.5-Turbo-0613 & 97.82 & 27.27 & 100.00 & 95.82 & 99.64 & 100.00 && 92.73 & 27.64 & 99.82 & 89.09 & 96.91 & 98.18 \\
GPT-3.5-Turbo-0125 & 97.82 & 20.00 & 100.00 & 93.82 & 100.00 & 100.00 && 92.73 & 22.00 & 98.18 & 88.36 & 97.09 & 97.64 \\
GPT-4 & 97.82 & 28.55 & 100.00 & 94.18 & 99.45 & 100.00 && 92.73 & 28.00 & 99.09 & 89.64 & 96.91 & 98.00 \\
GPT-4o & 97.82 & 26.55 & 99.82 & 94.55 & 99.82 & 99.45 && 92.73 & 26.36 & 98.00 & 89.09 & 97.09 & 98.55 \\
\addlinespace[0.4em]
\midrule
 & \multicolumn{6}{c}{$n_P$ = 150} && \multicolumn{6}{c}{$n_P$ = 200} \\
\cmidrule(lr){2-7} \cmidrule(lr){9-14}
 & & & & & \multicolumn{2}{c}{GAI} && & & & & \multicolumn{2}{c}{GAI} \\
\cmidrule(lr){6-7} \cmidrule(lr){13-14}
Model & Primary & Naive & PPI & PPI++ & Lab & Emb && Primary & Naive & PPI & PPI++ & Lab & Emb \\
\midrule
GPT-3.5-Turbo-0613 & 92.55 & 27.45 & 98.00 & 90.36 & 94.73 & 96.73 && 88.00 & 29.09 & 95.64 & 84.73 & 90.18 & 94.36 \\
GPT-3.5-Turbo-0125 & 92.55 & 24.36 & 96.73 & 90.91 & 94.18 & 96.91 && 88.00 & 27.45 & 94.55 & 84.73 & 90.00 & 93.45 \\
GPT-4 & 92.55 & 27.27 & 97.64 & 90.18 & 94.73 & 96.91 && 88.00 & 26.18 & 92.55 & 84.36 & 91.09 & 95.64 \\
GPT-4o & 92.55 & 26.36 & 95.82 & 90.91 & 94.73 & 96.91 && 88.00 & 27.09 & 92.91 & 85.09 & 90.91 & 94.55 \\
\bottomrule
\end{tabular}
\vskip -0.1in
\end{table}

\begin{table}[t]
\caption{Benchmark Comparison for Conjoint Analysis: 95\% CI Width}
\label{tab: conjoint_robustness_width}
\vskip 0.1in
\centering
\footnotesize
\begin{tabular}{@{}lcccccccccccccc@{}}
\toprule
\midrule
 & \multicolumn{6}{c}{$n_P$ = 50} && \multicolumn{6}{c}{$n_P$ = 100} \\
\cmidrule(lr){2-7} \cmidrule(lr){9-14}
 & & & & & \multicolumn{2}{c}{GAI} && & & & & \multicolumn{2}{c}{GAI} \\
\cmidrule(lr){6-7} \cmidrule(lr){13-14}
Model & Primary & Naive & PPI & PPI++ & Lab & Emb && Primary & Naive & PPI & PPI++ & Lab & Emb \\
\midrule
GPT-3.5-Turbo-0613 & 2.43 & 0.57 & 14.21 & 2.62 & 1.93 & 2.15 && 1.51 & 0.54 & 5.18 & 1.50 & 1.36 & 1.50 \\
GPT-3.5-Turbo-0125 & 2.43 & 0.51 & 37.35 & 2.48 & 1.92 & 2.15 && 1.51 & 0.49 & 3.30 & 1.51 & 1.35 & 1.50 \\
GPT-4 & 2.43 & 0.56 & 21.23 & 2.44 & 1.93 & 2.07 && 1.51 & 0.54 & 3.41 & 1.51 & 1.35 & 1.48 \\
GPT-4o & 2.43 & 0.55 & 9.28 & 2.44 & 1.94 & 2.11 && 1.51 & 0.53 & 2.71 & 1.52 & 1.36 & 1.50 \\
\addlinespace[0.4em]
\midrule
 & \multicolumn{6}{c}{$n_P$ = 150} && \multicolumn{6}{c}{$n_P$ = 200} \\
\cmidrule(lr){2-7} \cmidrule(lr){9-14}
 & & & & & \multicolumn{2}{c}{GAI} && & & & & \multicolumn{2}{c}{GAI} \\
\cmidrule(lr){6-7} \cmidrule(lr){13-14}
Model & Primary & Naive & PPI & PPI++ & Lab & Emb && Primary & Naive & PPI & PPI++ & Lab & Emb \\
\midrule
GPT-3.5-Turbo-0613 & 1.19 & 0.52 & 2.39 & 1.19 & 1.11 & 1.23 && 1.01 & 0.50 & 1.99 & 1.01 & 0.96 & 1.05 \\
GPT-3.5-Turbo-0125 & 1.19 & 0.47 & 2.14 & 1.19 & 1.11 & 1.23 && 1.01 & 0.46 & 1.67 & 1.01 & 0.96 & 1.04 \\
GPT-4 & 1.19 & 0.52 & 1.97 & 1.19 & 1.11 & 1.21 && 1.01 & 0.50 & 1.61 & 1.00 & 0.96 & 1.04 \\
GPT-4o & 1.19 & 0.51 & 1.94 & 1.19 & 1.11 & 1.21 && 1.01 & 0.49 & 1.58 & 1.00 & 0.96 & 1.05 \\
\bottomrule
\end{tabular}
\vskip -0.1in
\end{table}

\subsection{Retail Pricing Study}

\subsubsection{Tests for Performance Improvement}
\label{apdx: tests_pricing}

We conduct paired $t$-tests comparing GAI to each benchmark method across repeated trials for each metric (MAPE, coverage, CI width). The paired design controls for sampling variation and isolates estimator differences.

Table~\ref{tab:pricing_ttest} reports p-values for the pricing study. GAI's MAPE improvements over all benchmarks are significant at most sample sizes. For CI width, PPI++ produces narrower intervals than GAI.

\begin{table}[t]
\caption{Paired t-test p-values (GAI vs benchmarks) for Pricing Study}
\label{tab:pricing_ttest}
\vskip 0.1in
\centering
\small
\renewcommand{\arraystretch}{1}
\setlength{\tabcolsep}{2.8pt}
\begin{tabular}{@{}l|ccccc|ccccc|ccccc@{}}
\toprule
& \multicolumn{5}{c|}{(a) MAPE} & \multicolumn{5}{c|}{(b) 95\% CI Coverage} & \multicolumn{5}{c}{(c) 95\% CI Width} \\
\cmidrule(lr){2-6} \cmidrule(lr){7-11} \cmidrule(l){12-16}
Method & 100 & 200 & 300 & 400 & 500 & 100 & 200 & 300 & 400 & 500 & 100 & 200 & 300 & 400 & 500 \\
\midrule
Primary & 7e-5 & 6e-5 & 0.01 & 2e-3 & 0.11 & 0.06 & 0.18 & 0.15 & 0.32 & 0.30 & 0.44 & 0.08 & 5e-3 & 5e-5 & 4e-8 \\
Naive & 4e-11 & 2e-20 & 8e-17 & 1e-31 & 1e-21 & $\approx$0 & $\approx$0 & 3e-43 & $\approx$0 & 2e-28 & 5e-55 & 3e-67 & 8e-61 & 4e-68 & 3e-63 \\
PPI & 1e-5 & 5e-4 & 3e-3 & 2e-5 & 9e-5 & 0.06 & 0.16 & 0.41 & 0.06 & 1.00 & 8e-14 & 5e-30 & 1e-41 & 2e-46 & 7e-58 \\
PPI++ & 1e-4 & 2e-3 & 0.06 & 1e-3 & 0.09 & 0.03 & 0.32 & 0.11 & 0.32 & 0.26 & 9e-5 & 1e-7 & 4e-4 & 7e-4 & 3e-3 \\
\bottomrule
\end{tabular}
\vskip 0.05in
\footnotesize
\textbf{Notes}: $\approx$0 represents p-value $<10^{-300}$.
\vskip -0.1in
\end{table}

\subsubsection{Decision Errors}
\label{apdx: pricing_slip}

Table~\ref{tab:pricing_decision} reports decision errors. GAI achieves 0.2\% averaged across sample sizes, lower than PPI (1.2\%), PPI++ (1.8\%), and Primary (2.2\%). \rev{Naive attains 0.0\%, but this reflects its severe undercoverage (50\%): its biased yet narrow intervals seldom trigger a sign or significance error, so its zero rate is not a sign of reliable inference.}

\begin{table}[t]
\caption{Decision Errors for Pricing Study (\% of all CIs)}
\label{tab:pricing_decision}
\vskip 0.1in
\centering
\small
\renewcommand{\arraystretch}{1}
\setlength{\tabcolsep}{6pt}
\begin{tabular}{@{}lccccc@{}}
\toprule
Method & $n_P$=100 & $n_P$=200 & $n_P$=300 & $n_P$=400 & $n_P$=500 \\
\midrule
Primary & 4.0 & 2.0 & 3.0 & 0.0 & 2.0 \\
Naive & 0.0 & 0.0 & 0.0 & 0.0 & 0.0 \\
PPI & 2.0 & 0.0 & 1.0 & 2.0 & 1.0 \\
PPI++ & 4.0 & 0.0 & 4.0 & 0.0 & 1.0 \\
GAI & 0.0 & 0.0 & 0.0 & 0.0 & 1.0 \\
\bottomrule
\end{tabular}
\vskip 0.05in
\vskip -0.1in
\end{table}

\subsubsection{Results Across All Products}
\label{apdx: pricing_all_products}

Table~\ref{tab:pricing_all_summary} reports results averaged across all 40 products in the \cite{toubia2025database} dataset, where we use the same experimental design as Section~\ref{sec:pricing} and estimate $\gvec(\Xvec,z)$ using $\ell_1$-regularized logistic regression ($C=0.1$) with 5-fold cross-fitting.

\begin{table}[t]
\caption{Pricing Study Averaged Across All 40 Products: MAPEs, coverage probabilities, and CI widths}
\label{tab:pricing_all_summary}
\vskip 0.1in
\centering
\small
\renewcommand{\arraystretch}{1}
\setlength{\tabcolsep}{2.8pt}
\begin{tabular}{@{}l|ccccc|ccccc|ccccc@{}}
\toprule
& \multicolumn{5}{c|}{(a) MAPE (\%)}& \multicolumn{5}{c|}{(b) 95\% CI Coverage (\%)} & \multicolumn{5}{c}{(c) 95\% CI Width} \\
\cmidrule(lr){2-6} \cmidrule(lr){7-11} \cmidrule(l){12-16}
Method & 100 & 200 & 300 & 400 & 500 & 100 & 200 & 300 & 400 & 500 & 100 & 200 & 300 & 400 & 500 \\
\midrule
Primary & 14.34 & 9.66 & 7.80 & 6.23 & 5.62 & 96.10 & 96.45 & 97.12 & 97.90 & 97.98 & 0.88 & 0.62 & 0.50 & 0.43 & 0.39 \\
Naive & 22.48 & 21.06 & 19.64 & 18.32 & 17.00 & 46.32 & 46.95 & 46.90 & 47.28 & 48.10 & 0.28 & 0.27 & 0.25 & 0.24 & 0.23 \\
PPI & 17.73 & 12.17 & 10.11 & 8.81 & 8.03 & 96.68 & 96.42 & 97.10 & 96.72 & 96.80 & 1.07 & 0.77 & 0.64 & 0.57 & 0.52 \\
PPI++ & 14.16 & 9.40 & 7.60 & 6.14 & 5.51 & 94.88 & 95.92 & 96.78 & 97.52 & 97.52 & 0.84 & 0.59 & 0.48 & 0.42 & 0.38 \\
GAI & 8.51 & 6.94 & 6.14 & 5.07 & 4.68 & 99.25 & 98.60 & 98.65 & 99.02 & 98.92 & 0.88 & 0.61 & 0.49 & 0.42 & 0.38 \\
\bottomrule
\end{tabular}
\vskip 0.05in
\footnotesize
\vskip -0.1in
\end{table}

The results across all 40 products reinforce the main findings: (1) GAI achieves the lowest MAPE at every sample size; (2) GAI maintains coverage above 98\% while Naive severely undercovers (46--48\%); and (3) CI widths are comparable across valid methods. These patterns confirm that GAI's advantages generalize across diverse product categories.

\subsection{Health Insurance Census Analysis}

\subsubsection{Tests for Performance Improvement}
\label{apdx: tests_census}

We conduct paired $t$-tests comparing GAI to each benchmark method across repeated trials for each metric (MAPE, coverage, CI width). The paired design controls for sampling variation and isolates estimator differences.

Table~\ref{tab:census_ttest} shows GAI significantly outperforms all benchmarks in MAPE and coverage. For CI width, GAI and PPI++ show no significant difference.

\begin{table}[t]
\caption{Paired t-test p-values (GAI vs benchmarks) for Census Analysis}
\label{tab:census_ttest}
\vskip 0.1in
\centering
\small
\renewcommand{\arraystretch}{1}
\setlength{\tabcolsep}{2.8pt}
\begin{tabular}{@{}l|ccccc|ccccc|ccccc@{}}
\toprule
& \multicolumn{5}{c|}{(a) MAPE} & \multicolumn{5}{c|}{(b) 95\% CI Coverage} & \multicolumn{5}{c}{(c) 95\% CI Width} \\
\cmidrule(lr){2-6} \cmidrule(lr){7-11} \cmidrule(l){12-16}
Method & 100 & 250 & 500 & 750 & 1k & 100 & 250 & 500 & 750 & 1k & 100 & 250 & 500 & 750 & 1k \\
\midrule
Primary & 2e-10 & 2e-8 & 5e-9 & 5e-6 & 7e-6 & 3e-4 & 1e-4 & 1e-4 & 8e-4 & 1e-4 & 0.09 & 5e-5 & 4e-8 & 1e-13 & 1e-15 \\
Naive & 1e-7 & 6e-8 & 3e-5 & 6e-4 & 4e-3 & 3e-25 & 4e-25 & 5e-25 & 6e-27 & 6e-27 & 4e-10 & 1e-15 & 2e-18 & 4e-23 & 1e-24 \\
PPI & 3e-7 & 7e-10 & 1e-7 & 6e-6 & 1e-4 & 2e-3 & 0.03 & 0.02 & 0.17 & 0.05 & 0.78 & 0.74 & 0.98 & 0.65 & 0.36 \\
PPI++ & 1e-9 & 2e-9 & 1e-7 & 1e-5 & 2e-5 & 6e-4 & 2e-3 & 0.01 & 0.05 & 0.01 & 0.16 & 0.54 & 0.53 & 0.26 & 0.30 \\
\bottomrule
\end{tabular}
\vskip -0.1in
\end{table}

\subsubsection{Decision Errors}
\label{apdx: census_decision}

\begin{table}[t]
\caption{Decision Errors for Census Analysis (\% of all CIs)}
\label{tab:census_decision}
\vskip 0.1in
\centering
\small
\renewcommand{\arraystretch}{1}
\setlength{\tabcolsep}{6pt}
\begin{tabular}{@{}lccccc@{}}
\toprule
Method & $n_P$=100 & $n_P$=250 & $n_P$=500 & $n_P$=750 & $n_P$=1000 \\
\midrule
Primary & 11.0 & 4.0 & 4.0 & 2.0 & 3.0 \\
Naive & 20.0 & 19.0 & 16.0 & 12.0 & 12.0 \\
PPI & 8.0 & 0.0 & 3.0 & 2.0 & 2.0 \\
PPI++ & 10.0 & 3.0 & 3.0 & 2.0 & 3.0 \\
GAI & 0.0 & 0.0 & 0.0 & 0.0 & 0.0 \\
\bottomrule
\end{tabular}
\vskip 0.05in
\footnotesize
\vskip -0.1in
\end{table}

As shown in Table \ref{tab:census_decision}, GAI achieves zero decision errors across all sample sizes, while Naive exhibits 12--20\% error rates and Primary, PPI, PPI++ show moderate rates (2--11\%).

\section{Vaccine Conjoint Analysis under Covariate Shift: Known Density Ratio}
\label{apdx:conjoint_gaiw_covariate_shift}

This appendix considers a practical problem: the labeled and AI-augmented data may come from one source population, while the parameter of interest is defined for a different target population. This can occur when a conjoint survey is fielded to an accessible online panel, a specific customer segment, or an early-adopter sample, but the researcher wants to report preferences for a broader market whose covariate distribution differs from the source sample. In this setting, the estimating equations need to be weighted to the target covariate distribution. \cl{This appendix corresponds to the known-density-ratio case in Section~\ref{sec:covshift_known}, where the source-to-target density ratio is determined by the sampling design.}

We illustrate this source-to-target setting in the vaccine conjoint application. The covariate shift is defined by the severe-side-effect probability of Vaccine A with two strata: $1$ in $1{,}000{,}000$ and $1$ in $10{,}000$. The target covariate distribution is balanced, placing 50\% mass on each stratum. For each $n_P \in \{50,100,150\}$\footnote{Because this stratified sampling design requires a larger underlying pool, we reduce both the primary and auxiliary sample sizes in this application.}, we draw the primary sample without replacement so that 70\% of observations have severe side-effect probability $1$ in $1{,}000{,}000$ and 30\% have severe side-effect probability $1$ in $10{,}000$. We then draw $n_A=700$ auxiliary observations without replacement from the complement of the primary sample, using the same 70/30 source distribution. Thus, the primary and auxiliary samples have matching covariate distributions, which differs from the balanced 50/50 target distribution. All estimators are evaluated on the same sample path. We compute the ground-truth parameter $\beta^*$ by fitting the same choice model to the full human-labeled sample with the target covariate distribution.

We correct the source-to-target covariate shift by weighting each observation's GAI score. For observation $i$ in severe-side-effect stratum $s$, the weight is
\[
q_i\,=\,\frac{\pi_{\text{target}}(s)}{\pi_{\text{sample}}(s)}.
\]
Under the 70/30 source distribution and the 50/50 target distribution, these weights are approximately 0.71 for the $1$ in $1{,}000{,}000$ stratum and 1.67 for the $1$ in $10{,}000$ stratum. GAI-W estimates $g(\mathbf{X}, z)=\E[Y \mid \mathbf{X}, z]$ with cross-fitted logistic regression. Since the primary and auxiliary samples share the same source covariate distribution, GAI-W uses the constant labeling probability $e=n_P/(n_P+n_A)$.

We also weight the benchmark estimators to adjust for this distributional difference. Primary-W fits a weighted logistic regression using only the primary human labels. Naive-W fits a weighted logistic regression on the primary-plus-auxiliary union, using human labels for primary observations and LLM labels for auxiliary observations. PPI-W and PPI++-W pass weights for both labeled and auxiliary rows to the PPI package.

\begin{table}[t]
\caption{Benchmark Comparison for GAI-W under Side-Effect Covariate Shift}
\label{tab:conjoint_gaiw_weighted}
\vskip 0.1in
\centering
\small
\renewcommand{\arraystretch}{1}
\setlength{\tabcolsep}{3pt}
\begin{tabular}{@{}l|ccc|ccc|ccc@{}}
\toprule
& \multicolumn{3}{c|}{(a) MAPE (\%)} & \multicolumn{3}{c|}{(b) 95\% CI Coverage (\%)} & \multicolumn{3}{c}{(c) 95\% CI Width} \\
\cmidrule(lr){2-4} \cmidrule(lr){5-7} \cmidrule(l){8-10}
Method & 50 & 100 & 150 & 50 & 100 & 150 & 50 & 100 & 150 \\
\midrule
Primary-W & 55.60 & 30.30 & 25.02 & 90.18 & 90.73 & 90.18 & 2.70 & 1.64 & 1.31 \\
Naive-W & 113.30 & 86.25 & 70.29 & 13.64 & 14.18 & 15.64 & 1.15 & 0.93 & 0.81 \\
PPI-W & 84.31 & 60.53 & 36.70 & 100.00 & 99.63 & 98.00 & 15.03 & 4.49 & 2.33 \\
PPI++-W & 50.87 & 31.59 & 25.11 & 89.84 & 90.91 & 90.18 & 2.67 & 1.63 & 1.30 \\
GAI-W & 18.04 & 16.78 & 16.01 & 99.64 & 99.09 & 97.45 & 2.15 & 1.49 & 1.21 \\
\bottomrule
\end{tabular}
\vskip 0.05in
\footnotesize
\textbf{Notes}: The auxiliary sample is drawn to match the primary sample's 70/30 severe-side-effect distribution. All methods use post-stratification weights to target the balanced severe-side-effect distribution. PPI-W and PPI++-W summaries drop numerical failures; they are based on 34 successful trials at $n_P=50$, 49 at $n_P=100$, and 50 at $n_P=150$.
\vskip -0.1in
\end{table}

\paragraph{Results.} Table~\ref{tab:conjoint_gaiw_weighted} shows that GAI-W has the lowest MAPE across all primary sample sizes, while maintaining coverage above the nominal level and producing narrower intervals than the other valid methods. Naive-W remains biased despite weighting, indicating that correcting the covariate distribution does not remove the bias from treating auxiliary labels as human labels. PPI-W attains high coverage, but with substantially wider intervals, especially at small primary sample sizes.

\begin{table}[t]
\caption{Paired t-test p-values (GAI-W vs weighted benchmarks) under Side-Effect Covariate Shift}
\label{tab:conjoint_gaiw_ttest}
\vskip 0.1in
\centering
\small
\renewcommand{\arraystretch}{1}
\setlength{\tabcolsep}{2.8pt}
\begin{tabular}{@{}l|ccc|ccc|ccc@{}}
\toprule
& \multicolumn{3}{c|}{(a) MAPE} & \multicolumn{3}{c|}{(b) 95\% CI Coverage} & \multicolumn{3}{c}{(c) 95\% CI Width} \\
\cmidrule(lr){2-4} \cmidrule(lr){5-7} \cmidrule(l){8-10}
Method & 50 & 100 & 150 & 50 & 100 & 150 & 50 & 100 & 150 \\
\midrule
Primary-W & 1.01e-19 & 2.42e-17 & 1.18e-14 & 1.40e-9 & 5.48e-9 & 5.19e-8 & 4.67e-15 & 2.01e-12 & 3.34e-14 \\
Naive-W & 9.37e-44 & 8.06e-40 & 8.40e-47 & 2.44e-56 & 2.64e-54 & 6.61e-51 & 2.47e-46 & 5.66e-48 & 2.34e-54 \\
PPI-W & 1.40e-7 & 2.91e-12 & 1.09e-17 & -- & 0.26 & 0.50 & 0.08 & 6.03e-7 & 2.07e-25 \\
PPI++-W & 1.21e-14 & 2.52e-16 & 3.21e-16 & 2.47e-7 & 1.02e-6 & 5.19e-8 & 1.79e-9 & 4.21e-14 & 1.62e-13 \\
\bottomrule
\end{tabular}
\vskip 0.05in
\footnotesize
\textbf{Notes}: A ``--" indicates that paired coverage for GAI-W and PPI-W was identical across trials when PPI-W does not fail due to singularity problems, so the paired $t$-statistic is undefined.
\vskip -0.1in
\end{table}

Table~\ref{tab:conjoint_gaiw_ttest} confirms that GAI-W's MAPE improvements are statistically significant relative to all weighted benchmarks at every primary sample size. Coverage improvements are significant relative to Primary-W, Naive-W, and PPI++-W. For interval width, GAI-W is significantly narrower than Primary-W and PPI++-W at all sample sizes and narrower than PPI-W at $n_P=100$ and $150$, while Naive-W is narrower but invalid because it severely undercovers.

\begin{table}[t]
\caption{Decision Errors for GAI-W under Side-Effect Covariate Shift (\% of all CIs)}
\label{tab:conjoint_gaiw_decision}
\vskip 0.1in
\centering
\small
\renewcommand{\arraystretch}{1}
\setlength{\tabcolsep}{6pt}
\begin{tabular}{@{}lccc@{}}
\toprule
Method & $n_P$=50 & $n_P$=100 & $n_P$=150 \\
\midrule
Primary-W & 7.27 & 8.36 & 7.82 \\
Naive-W & 23.27 & 22.73 & 22.55 \\
PPI-W & 0.00 & 0.37 & 1.82 \\
PPI++-W & 6.68 & 7.61 & 7.64 \\
GAI-W & 0.36 & 0.91 & 2.55 \\
\bottomrule
\end{tabular}
\vskip -0.1in
\end{table}

Table~\ref{tab:conjoint_gaiw_decision} reports decision errors for the same matched-auxiliary design. GAI-W has substantially lower decision-error rates than Primary-W, Naive-W, and PPI++-W. PPI-W has slightly lower decision-error rates, but this again reflects its wider confidence intervals and lower test power rather than superior point estimation.

\section{Vaccine Conjoint Analysis under Covariate Shift: Estimated Density Ratio}
\label{apdx:conjoint_gaidr_covariate_shift}

This appendix reports an estimated-density-ratio version of the vaccine conjoint covariate-shift experiment. \cl{This appendix provides the empirical implementation of the GAI-DR estimator developed in Section~\ref{sec:covshift_dr}.} In contrast to the known-ratio design in \rev{Section}~\ref{apdx:conjoint_gaiw_covariate_shift}, the source-to-target density ratio is not supplied to the researcher and must be estimated from source and target covariates.

We construct source and target populations by tilting the finite vaccine conjoint pool along three attributes: 90\% efficacy, five-year coverage, and rare severe side effects ($1$ in $1{,}000{,}000$). The source population places lower mass on these favorable attributes, while the target population places higher mass on them.

For each trial and each $n_P \in \{50,100,150\}$, we draw $n_P$ fully labeled source observations and $n_T=700$ target observations with covariates and LLM labels. All estimators are evaluated on the same sample paths and target the full-target-population ground-truth parameter. GAI-DR estimates the source-to-target density ratio using a cross-fitted logistic classifier with L2 regularization ($C=1$). The nuisance prediction $g(\mathbf{X},z)$ is estimated using cross-fitted L2 logistic regression ($C=0.05$). Because the density ratio is unknown in this setting, the empirical benchmarks do not receive density-ratio weights.

\begin{table}[t]
\caption{Benchmark Comparison for GAI-DR under Estimated Density-Ratio Covariate Shift}
\label{tab:conjoint_gaidr_compare}
\vskip 0.1in
\centering
\small
\renewcommand{\arraystretch}{1}
\setlength{\tabcolsep}{3pt}
\begin{tabular}{@{}l|ccc|ccc|ccc@{}}
\toprule
& \multicolumn{3}{c|}{(a) MAPE (\%)} & \multicolumn{3}{c|}{(b) 95\% CI Coverage (\%)} & \multicolumn{3}{c}{(c) 95\% CI Width} \\
\cmidrule(lr){2-4} \cmidrule(lr){5-7} \cmidrule(l){8-10}
Method & 50 & 100 & 150 & 50 & 100 & 150 & 50 & 100 & 150 \\
\midrule
Primary & 60.59 & 32.42 & 24.30 & 87.45 & 89.45 & 89.82 & 2.65 & 1.52 & 1.19 \\
Naive & 121.00 & 87.65 & 70.70 & 10.00 & 13.82 & 19.45 & 1.07 & 0.82 & 0.71 \\
PPI & 131.43 & 71.92 & 75.70 & 100.00 & 100.00 & 98.73 & 97.30 & 6.58 & 47.15 \\
PPI++ & 54.35 & 33.04 & 25.51 & 91.43 & 89.09 & 90.55 & 3.87 & 1.54 & 1.21 \\
GAI-DR & 27.29 & 19.21 & 17.60 & 97.45 & 98.91 & 98.18 & 1.89 & 1.53 & 1.43 \\
\bottomrule
\end{tabular}
\vskip 0.05in
\footnotesize
\vskip -0.1in
\end{table}

\paragraph{Results.} Table~\ref{tab:conjoint_gaidr_compare} shows that GAI-DR has the lowest MAPE at all three source-label sample sizes after estimating the source-to-target density ratio. The improvements are large: at $n_P=50$, GAI-DR reduces MAPE from 60.59\% for Primary and 54.35\% for PPI++ to 27.29\%. GAI-DR also restores nominal coverage, attaining 97--99\% coverage across sample sizes, whereas Primary, Naive, and PPI++ remain below 95\%. Importantly, the coverage improvement is not driven by excessively wide intervals: GAI-DR intervals are comparable to Primary and PPI++ at $n_P=100$ and only moderately wider at $n_P=150$, while being substantially narrower than PPI.

\begin{table}[t]
\caption{Decision Errors for GAI-DR under Estimated Density-Ratio Covariate Shift (\% of all CIs)}
\label{tab:conjoint_gaidr_decision}
\vskip 0.1in
\centering
\small
\renewcommand{\arraystretch}{1}
\setlength{\tabcolsep}{6pt}
\begin{tabular}{@{}lccc@{}}
\toprule
Method & $n_P$=50 & $n_P$=100 & $n_P$=150 \\
\midrule
Primary & 9.09 & 8.91 & 8.91 \\
Naive & 38.18 & 36.91 & 33.64 \\
PPI & 0.00 & 0.00 & 1.27 \\
PPI++ & 5.45 & 9.45 & 8.36 \\
GAI-DR & 1.82 & 0.73 & 1.27 \\
\bottomrule
\end{tabular}
\vskip -0.1in
\end{table}

Table~\ref{tab:conjoint_gaidr_decision} shows that GAI-DR also yields low decision-error rates. Its average decision-error rate is 1.27\%, compared with 8.97\% for Primary, 36.24\% for Naive, and 7.75\% for PPI++. PPI has similarly low decision errors, but this comes from extremely wide intervals rather than from accurate point estimation.

\begin{table}[t]
\caption{Paired t-test p-values (GAI-DR vs unweighted benchmarks) under Estimated Density-Ratio Covariate Shift}
\label{tab:conjoint_gaidr_ttest}
\vskip 0.1in
\centering
\small
\renewcommand{\arraystretch}{1}
\setlength{\tabcolsep}{2.8pt}
\begin{tabular}{@{}l|ccc|ccc|ccc@{}}
\toprule
& \multicolumn{3}{c|}{(a) MAPE} & \multicolumn{3}{c|}{(b) 95\% CI Coverage} & \multicolumn{3}{c}{(c) 95\% CI Width} \\
\cmidrule(lr){2-4} \cmidrule(lr){5-7} \cmidrule(l){8-10}
Method & 50 & 100 & 150 & 50 & 100 & 150 & 50 & 100 & 150 \\
\midrule
Primary & 2.70e-13 & 1.76e-18 & 3.76e-13 & 1.43e-06 & 3.97e-09 & 7.13e-08 & 4.13e-14 & 0.69 & 1.06e-09 \\
Naive & 9.57e-43 & 5.70e-49 & 1.62e-44 & 8.28e-49 & 1.75e-49 & 2.47e-46 & 9.76e-25 & 3.09e-26 & 4.88e-28 \\
PPI & 3.99e-04 & 1.16e-07 & 3.40e-03 & 6.25e-03 & 0.03 & 0.37 & 0.26 & 2.95e-03 & 0.30 \\
PPI++ & 1.46e-04 & 1.13e-18 & 2.03e-12 & 1.00e-03 & 5.78e-09 & 2.73e-07 & 0.15 & 0.73 & 1.04e-08 \\
\bottomrule
\end{tabular}
\vskip 0.05in
\footnotesize
\vskip -0.1in
\end{table}

Table~\ref{tab:conjoint_gaidr_ttest} confirms that GAI-DR's MAPE improvements are statistically significant relative to every benchmark at every source-label sample size. Coverage improvements relative to Primary, Naive, and PPI++ are also statistically significant throughout. For interval width, GAI-DR is substantially narrower than Primary at $n_P=50$ and PPI at $n_P=100$, but moderately wider than Primary and PPI++ at $n_P=150$. Naive yields the narrowest intervals around the biased point estimates.

\section{Pricing Study with High-dimensional Auxiliary Information}
\label{apdx: pricing_hdim}

We extend the auxiliary information to include persona and demographic features: $\zvec = [z, \text{persona}, \text{demographics}] \in \mathbb{R}^{115}$, where $z$ is the scalar AI prediction, persona includes 54 psychometric features, and demographics includes 60 indicator variables. The outcome model remains $\Xvec = [\mathbf{1}, \text{price}]$.

We estimate $\gvec(\Xvec,\zvec)$ using $\ell_1$-regularized logistic regression ($C=0.01$) with 5-fold cross-fitting.

\paragraph{Results.} Table~\ref{tab: pricing_hdim} reports results across 50 Monte Carlo trials.

\begin{table}[t]
\caption{Pricing Study with High-dimensional Auxiliary Information: MAPEs, coverage probabilities, and CI widths}
\label{tab: pricing_hdim}
\vskip 0.1in
\centering
\small
\renewcommand{\arraystretch}{1}
\setlength{\tabcolsep}{2.8pt}
\begin{tabular}{@{}l|ccccc|ccccc|ccccc@{}}
\toprule
& \multicolumn{5}{c|}{(a) MAPE (\%)} & \multicolumn{5}{c|}{(b) 95\% CI Coverage (\%)} & \multicolumn{5}{c}{(c) 95\% CI Width} \\
\cmidrule(lr){2-6} \cmidrule(lr){7-11} \cmidrule(l){12-16}
Method & 100 & 200 & 300 & 400 & 500 & 100 & 200 & 300 & 400 & 500 & 100 & 200 & 300 & 400 & 500 \\
\midrule
Primary & 22.36 & 16.03 & 12.47 & 9.59 & 9.53 & 95.00 & 97.00 & 95.00 & 99.00 & 92.00 & 1.17 & 0.82 & 0.66 & 0.57 & 0.51 \\
Naive & 22.22 & 21.97 & 19.99 & 19.67 & 19.43 & 50.00 & 50.00 & 50.00 & 50.00 & 50.00 & 0.37 & 0.35 & 0.33 & 0.32 & 0.31 \\
PPI & 26.91 & 15.31 & 14.10 & 13.98 & 12.88 & 95.00 & 98.00 & 97.00 & 94.00 & 96.00 & 1.40 & 1.01 & 0.85 & 0.75 & 0.69 \\
PPI++ & 22.15 & 14.16 & 11.56 & 9.86 & 9.66 & 94.00 & 99.00 & 93.00 & 99.00 & 92.00 & 1.10 & 0.78 & 0.64 & 0.55 & 0.49 \\
GAI & 6.64 & 4.34 & 4.42 & 3.95 & 4.62 & 100.00 & 100.00 & 100.00 & 100.00 & 100.00 & 1.12 & 0.80 & 0.65 & 0.57 & 0.50 \\
\bottomrule
\end{tabular}
\vskip 0.05in
\vskip -0.1in
\end{table}

Panel (a) shows GAI achieves substantially lower MAPE. Panel (b) shows GAI maintains 100\% coverage across all sample sizes while Naive undercovers (50\%). Panel (c) shows GAI produces CI widths comparable to Primary and PPI++ while achieving better point estimation. Decision errors (Table~\ref{tab:pricing_hdim_decision}): GAI averages 0.00\% vs 1.20\% for PPI, 1.80\% for PPI++, and 2.20\% for Primary.

\begin{table}[t]
\caption{Decision Errors for Pricing Study with High-dimensional Auxiliary Information (\% of all CIs)}
\label{tab:pricing_hdim_decision}
\vskip 0.1in
\centering
\small
\renewcommand{\arraystretch}{1}
\setlength{\tabcolsep}{6pt}
\begin{tabular}{@{}lccccc@{}}
\toprule
Method & $n_P$=100 & $n_P$=200 & $n_P$=300 & $n_P$=400 & $n_P$=500 \\
\midrule
Primary & 4.00 & 2.00 & 3.00 & 0.00 & 2.00 \\
Naive & 0.00 & 0.00 & 0.00 & 0.00 & 0.00 \\
PPI & 2.00 & 0.00 & 1.00 & 2.00 & 1.00 \\
PPI++ & 4.00 & 0.00 & 4.00 & 0.00 & 1.00 \\
GAI & 0.00 & 0.00 & 0.00 & 0.00 & 0.00 \\
\bottomrule
\end{tabular}
\vskip 0.05in
\vskip -0.1in
\end{table}

Table~\ref{tab:pricing_hdim_ttest} reports p-values. GAI achieves significant MAPE reductions at all sample sizes. For CI width, GAI and PPI++ produce comparable intervals.

\begin{table}[t]
\caption{Paired t-test p-values (GAI vs benchmarks) for Pricing Study with High-dimensional Auxiliary Information}
\label{tab:pricing_hdim_ttest}
\vskip 0.1in
\centering
\small
\renewcommand{\arraystretch}{1}
\setlength{\tabcolsep}{2.8pt}
\begin{tabular}{@{}l|ccccc|ccccc|ccccc@{}}
\toprule
& \multicolumn{5}{c|}{(a) MAPE} & \multicolumn{5}{c|}{(b) 95\% CI Coverage} & \multicolumn{5}{c}{(c) 95\% CI Width} \\
\cmidrule(lr){2-6} \cmidrule(lr){7-11} \cmidrule(l){12-16}
Method & 100 & 200 & 300 & 400 & 500 & 100 & 200 & 300 & 400 & 500 & 100 & 200 & 300 & 400 & 500 \\
\midrule
Primary & 7e-10 & 4e-11 & 9e-9 & 4e-9 & 1e-5 & 0.06 & 0.18 & 0.06 & 0.32 & 0.02 & 7e-4 & 9e-6 & 5e-4 & 0.06 & 4e-5 \\
Naive & 6e-49 & 1e-58 & 4e-44 & 3e-45 & 4e-42 & $\approx$0 & $\approx$0 & $\approx$0 & $\approx$0 & $\approx$0 & 8e-58 & 2e-69 & 3e-71 & 6e-69 & 3e-71 \\
PPI & 5e-9 & 1e-9 & 4e-8 & 2e-8 & 4e-10 & 0.06 & 0.16 & 0.18 & 0.06 & 0.10 & 8e-16 & 6e-31 & 2e-38 & 2e-45 & 6e-53 \\
PPI++ & 3e-9 & 2e-9 & 1e-7 & 2e-9 & 5e-6 & 0.03 & 0.32 & 0.05 & 0.32 & 0.01 & 0.07 & 6e-4 & 2e-5 & 2e-7 & 9e-8 \\
\bottomrule
\end{tabular}
\vskip 0.05in
\footnotesize
\textbf{Notes}: $\approx$0 represents p-value $<10^{-300}$.
\vskip -0.1in
\end{table}

\section{Pricing Study with High-dimensional Covariates and Scalar Auxiliary Information}
\label{apdx: pricing_scalar_z}

This appendix examines a setting where the outcome model includes high-dimensional covariates ($\Xvec \in \mathbb{R}^{115}$, reduced to 77 PCA components) while auxiliary information remains scalar ($z \in \{0,1\}$). This confirms GAI outperforms benchmarks even when the auxiliary prediction provides no extra information beyond what PPI-based methods can exploit.

We estimate $\gvec(\Xvec,z)$ using $\ell_1$-regularized logistic regression ($C=0.01$) with 5-fold cross-fitting.

\paragraph{Results.} Table~\ref{tab: pricing_scalar_z} reports results across 50 Monte Carlo trials.

\begin{table}[t]
\caption{Pricing Study with High-dimensional Covariates: MAPEs, coverage probabilities, and CI widths}
\label{tab: pricing_scalar_z}
\vskip 0.1in
\centering
\small
\renewcommand{\arraystretch}{1}
\setlength{\tabcolsep}{2.8pt}
\begin{tabular}{@{}l|ccccc|ccccc|ccccc@{}}
\toprule
& \multicolumn{5}{c|}{(a) MAPE (\%)} & \multicolumn{5}{c|}{(b) 95\% CI Coverage (\%)} & \multicolumn{5}{c}{(c) 95\% CI Width} \\
\cmidrule(lr){2-6} \cmidrule(lr){7-11} \cmidrule(l){12-16}
Method & 100 & 200 & 300 & 400 & 500 & 100 & 200 & 300 & 400 & 500 & 100 & 200 & 300 & 400 & 500 \\
\midrule
Primary & -- & 73.83 & 25.72 & 15.65 & 11.68 & 48.97 & 84.44 & 89.97 & 94.46 & 94.21 & -- & 2.57 & 1.09 & 0.77 & 0.57 \\
Naive & 6.61 & 5.93 & 5.40 & 5.08 & 4.77 & 94.31 & 93.85 & 94.18 & 93.67 & 92.62 & 0.34 & 0.31 & 0.27 & 0.27 & 0.23 \\
PPI & N/A & 38.16 & 97.06 & 56.76 & 52.33 & N/A & 100.00 & 99.91 & 99.85 & 99.91 & N/A & 5.30 & 57.75 & 19.14 & 16.90 \\
PPI++ & N/A & 24.55 & 25.10 & 15.66 & 11.32 & N/A & 91.35 & 91.92 & 94.81 & 93.32 & N/A & 4.04 & 4.15 & 1.11 & 0.47 \\
GAI & 3.95 & 3.80 & 3.68 & 3.47 & 3.32 & 100.00 & 99.44 & 98.85 & 98.95 & 99.15 & 0.81 & 0.57 & 0.44 & 0.44 & 0.39 \\
\bottomrule
\end{tabular}
\vskip 0.05in
\footnotesize
\textbf{Notes}: ``--'' indicates Primary estimator values exceeding 1,000 due to quasi-complete separation in unregularized logistic regression at small sample sizes. ``N/A'' indicates PPI/PPI++ failed to produce valid estimates due to singular Hessian matrices arising from extreme predicted probabilities.
\vskip -0.1in
\end{table}

Panel (a) shows GAI achieves MAPE of 3.3--3.9\% across all sample sizes, while Naive exhibits 4.8--6.6\% and PPI-based methods 11--97\%. PPI/PPI++ fail at $n_P=100$ due to quasi-complete separation. Statistical tests confirm GAI's improvements (Table~\ref{tab:pricing_scalar_ttest}).

\begin{table}[t]
\caption{Decision Errors for Pricing Study with High-dimensional Covariates (\% of all CIs)}
\label{tab:pricing_scalar_decision}
\vskip 0.1in
\centering
\small
\renewcommand{\arraystretch}{1}
\setlength{\tabcolsep}{6pt}
\begin{tabular}{@{}lccccc@{}}
\toprule
Method & $n_P$=100 & $n_P$=200 & $n_P$=300 & $n_P$=400 & $n_P$=500 \\
\midrule
Primary & 22.44 & 6.69 & 3.87 & 1.82 & 1.92 \\
Naive & 3.03 & 3.13 & 3.03 & 3.23 & 3.87 \\
PPI & N/A & 0.00 & 0.04 & 0.09 & 0.00 \\
PPI++ & N/A & 4.17 & 3.55 & 1.70 & 2.24 \\
GAI & 0.00 & 0.56 & 1.15 & 1.05 & 0.77 \\
\bottomrule
\end{tabular}
\vskip 0.05in
\footnotesize
\textbf{Notes}: PPI/PPI++ failed to produce valid estimates at $n_P=100$ due to singular Hessian matrices.
\vskip -0.1in
\end{table}

Panel (b) shows GAI maintains coverage near 99\%, while Primary undercovers severely at small $n_P$ (49\% at $n_P=100$). Decision errors (Table~\ref{tab:pricing_scalar_decision}): GAI averages 0.71\% vs 2.33\% for PPI++, 3.26\% for Naive, and 7.35\% for Primary. Panel (c) shows GAI intervals are narrower than PPI/PPI++ (0.39--0.81 vs 0.47--4.15) while maintaining valid coverage.

\begin{table}[t]
\caption{Paired t-test p-values (GAI vs benchmarks) for Pricing Study with High-dimensional Covariates}
\label{tab:pricing_scalar_ttest}
\vskip 0.1in
\centering
\small
\renewcommand{\arraystretch}{1}
\setlength{\tabcolsep}{2.8pt}
\begin{tabular}{@{}l|ccccc|ccccc|ccccc@{}}
\toprule
& \multicolumn{5}{c|}{(a) MAPE} & \multicolumn{5}{c|}{(b) 95\% CI Coverage} & \multicolumn{5}{c}{(c) 95\% CI Width} \\
\cmidrule(lr){2-6} \cmidrule(lr){7-11} \cmidrule(l){12-16}
Method & 100 & 200 & 300 & 400 & 500 & 100 & 200 & 300 & 400 & 500 & 100 & 200 & 300 & 400 & 500 \\
\midrule
Primary & 2e-19 & 5e-15 & 1e-16 & 7e-22 & 9e-19 & 3e-39 & 2e-11 & 6e-7 & 8e-7 & 9e-9 & 1e-10 & 1e-21 & 4e-26 & 1e-7 & 2e-5 \\
Naive & 3e-50 & 7e-27 & 4e-31 & 6e-31 & 2e-46 & 2e-23 & 1e-21 & 1e-19 & 7e-24 & 1e-27 & 2e-14 & 1e-12 & 1e-43 & 2e-3 & 1e-4 \\
PPI & -- & 7e-3 & 3e-5 & 1e-5 & 3e-3 & -- & 1e-7 & 9e-18 & 3e-15 & 5e-7 & -- & 0.20 & 0.04 & 0.03 & 0.17 \\
PPI++ & -- & 0.12 & 2e-3 & 2e-6 & 2e-4 & -- & 0.17 & 4e-3 & 7e-8 & 5e-7 & -- & 0.37 & 0.08 & 0.01 & 0.10 \\
\bottomrule
\end{tabular}
\vskip 0.05in
\footnotesize
\textbf{Notes}: ``--'' indicates PPI/PPI++ failed at $n_P=100$ due to singular Hessian matrices.
\vskip -0.1in
\end{table}

Table~\ref{tab:pricing_scalar_ttest} reports p-values. GAI achieves significant MAPE reductions and coverage improvement at all sample sizes.

\section{Supporting Results for Section \ref{sec:theory}}
\label{app:proofs}

This appendix collects the proofs of the results in Section~\ref{sec:theory} together with an example demonstrating that the dominance guarantee can fail outside the random-labeling design.

\subsection{Proof of Theorem \ref{thm:normality}}
\label{apdx: proof of normality}

In the proof, for simplicity we assume without loss of generality that we can split the dataset into $I$ and $I^c$, each of size $n$. 
We obtain $\hat{e}(\cdot)$ and $\hat{\gvec}(\cdot)$ on $I^c$ and estimate $\hat{\betavec}$ on $I$.
Let us define 
\begin{align}
    \tauvec(\Xi; e, \gvec)\,:=\,\Xvec^{\top}
    \left[
    \gvec(\Xvec,\zvec) - \frac{w}{e(\Xvec,\zvec)}(\gvec(\Xvec,\zvec) - \yvec) 
    \right]
\end{align}
so $\Psivec(\Xi; e, \gvec; \betavec) = \Xvec^{\top}\nabla b(\Xvec\betavec) - \tauvec(\Xi; e, \gvec)$ for all $\betavec \in \mB$.
Throughout these proofs we abbreviate the nuisance estimation errors by $\Delta_{\gvec} := \hat{\gvec}(\Xvec,\zvec) - \gvec^*(\Xvec,\zvec)$ and $\Delta_{e} := \hat{e}(\Xvec,\zvec) - e^*(\Xvec,\zvec)$, suppressing the arguments $(\Xvec,\zvec)$ when no confusion arises.
For simplicity, following the empirical process literature (see \citealt{van2000asymptotic,vaart2023empirical}), we often use the shorthand notation $\Pe_n f$ to denote the empirical expectation of a function $f$ based on data in $I$ and $\Gn f$ to denote the empirical process based on data in $I$; the population expectation of $f$ is written $\E[f]$.
With the boundedness assumption on $\mX$, we assume that there is $C$ such that for any $\Xvec \in\mX$, 
$\lVert \Xvec \rVert, \lVert \Xvec \rVert_{F} \leq C$.
We let $B(\xvec, \epsilon)$ denote an open ball of radius $\epsilon$ around $\xvec \in \R^q$.

The next lemma establishes that our score function given in \ref{eq:score} is valid and states the key consequence of its Neyman orthogonal design.  

\begin{lemma}[\normalfont \textbf{Score Function and Neyman Orthogonality}]
\label{lem: orthogonality} It holds that 
$$\E[\Psivec(e^*, \gvec^*; \betavec^*)]\,=\,0 ~~\text{and}~~
\Pe_n \tauvec(\hat{e}, \hat{\gvec})-
\Pe_n \tauvec(e^*, \gvec^*)\,=\,o_P(n^{-1/2}).
$$
\end{lemma}

The next two key lemmas explore the concavity of our loss function. The first one clarifies the convergence rate of the empirical score functions.
Building on this, we present a consistency proof.

\begin{lemma}[\normalfont \textbf{Rates of Empirical Scores}]
\label{lem: empirical-scores} It holds that 
$$
\big\lVert
\Pe_n \Psivec\big({e}^*, {\gvec}^*; \hat{\betavec}\big)
\big \rVert,\,\inf_{\beta \in \mB}\big\lVert
\Pe_n \Psivec\big(\hat{e}, \hat{\gvec}; \betavec\big)
\big \rVert\,=\,o_P(1/\sqrt{n}).
$$
\end{lemma}

\begin{lemma}[\normalfont \textbf{Consistency}]
\label{lem: consistency} $\hat{\betavec}\to_P \betavec^*$.
\end{lemma}

With these lemmas we finish the proof of the main result. 
For ease of exposition, from this point on we suppress $e^*$ and $\gvec^*$ in $\Psivec(\cdot; e^*, \gvec^*; \betavec)$ and write it as $\Psivec(\cdot; \betavec)$.
We first establish that the score function is locally Lipschitz. 
Indeed, fix an arbitrary $\epsilon$ and consider $\betavec_1, \betavec_2 \in B(\betavec^*, \epsilon)$, we have 
\begin{align*}
    &\left \lVert\Psivec(\Xi; \betavec_1) - \Psivec(\Xi; \betavec_2)\right \rVert 
\,=\,\left \lVert\Xvec^{\top}\nabla b(\Xvec\betavec_1) - \Xvec^{\top}\nabla b(\Xvec\betavec_2)\right \rVert
\,=\,
\left \lVert\Xvec^{\top}\hspace{-2mm}
\displaystyle \int_0^1
{\nabla^2} b\left(\Xvec\left(\betavec_1+t(\betavec_2 - \betavec_1)\right)\right)dt(\betavec_2- \betavec_1)\right \rVert
\\\,\leq\,&
\left \lVert\Xvec^{\top}\hspace{-2mm}
\displaystyle \int_0^1
{\nabla^2} b\left(\Xvec\left(\betavec_1+t(\betavec_2 - \betavec_1)\right)\right)dt\right \rVert \lVert \betavec_2- \betavec_1\rVert
\,\leq\,\lVert\Xvec\rVert 
\displaystyle \int_0^1
\left \lVert{\nabla^2} b\left(\Xvec\left(\betavec_1+t(\betavec_2 - \betavec_1)\right)\right)\right\rVert dt\lVert \betavec_2- \betavec_1\rVert\\
\,\leq\,&\lVert \betavec_2- \betavec_1\rVert\left(\lVert\Xvec\rVert\sup_{\betavec \in \overline{B(\betavec^*, \epsilon)}}\lVert {\nabla^2}b(\Xvec\betavec)\rVert\right)
\,\leq\, C_1\lVert\Xvec\rVert\lVert \betavec_2- \betavec_1\rVert
\end{align*}
where we define $C_1 := \sup_{\betavec \in \overline{B(\betavec^*, \epsilon)}, \,\Xvec\in\Breve{\mX}}\lVert{\nabla^2} b(\Xvec\betavec)\rVert < \infty$.
Indeed, the supremum is finite because it is taken over the compact set $\Breve{\mX} \times \overline{B(\betavec^*, \epsilon)}$ and ${\nabla^2} b(\cdot)$ is continuous; consequently the local Lipschitz modulus $m(\Xi) := C_1\lVert\Xvec\rVert$ satisfies $\E[m^2] = C_1^2\, \E\big[\lVert\Xvec\rVert^2\big] < \infty$.
The class moreover admits a square-integrable envelope. Writing $\Psivec(\Xi; \betavec) = \Xvec^{\top}\nabla b(\Xvec\betavec) - \tauvec(\Xi; e^*, \gvec^*)$, the $\betavec$-dependent term is uniformly bounded over $\betavec \in \overline{B(\betavec^*, \epsilon)}$ and $\Xvec \in \Breve{\mX}$ by the continuity of $\nabla b$ on this compact set, so it suffices to verify that $\E\big[\lVert\Psivec(\Xi; e^*, \gvec^*; \betavec^*)\rVert^2\big] < \infty$. Decompose $\Psivec(\Xi; e^*, \gvec^*; \betavec^*) = \zetavec + \tfrac{w}{e^*}\pivec$ with $\zetavec = \Xvec^{\top}\big(\nabla b(\Xvec\betavec^*) - \gvec^*\big)$ and $\pivec = \Xvec^{\top}(\gvec^* - \yvec)$. The irreducible-noise block is finite by Assumption~\ref{ass:regularity}(iii): using $w/e^* \leq 1/\kappa$ and $\lVert\Xvec\rVert_F \leq C$,
$$\E\Big[\Big\lVert\tfrac{w}{e^*}\pivec\Big\rVert^2\Big] \leq \kappa^{-2}\,\E\big[\lVert\Xvec\rVert^2_F\,\lVert\Cov(\yvec\mid\Xvec,\zvec)\rVert\big] \leq \kappa^{-2} C^2 \tilde{\sigma}^2 < \infty.$$
The representational block is finite by Assumption~\ref{ass:regularity}(iv): since $\gvec^* = \E[\yvec\mid\Xvec,\zvec]$, conditional Jensen gives $\E\lVert\gvec^*\rVert^2 \leq \E\lVert\yvec\rVert^2 < \infty$, so $\E\big[\lVert\zetavec\rVert^2\big] \leq 2C^2\big(\sup_{\Xvec\in\Breve{\mX}}\lVert\nabla b(\Xvec\betavec^*)\rVert^2 + \E\lVert\gvec^*\rVert^2\big) < \infty$. Hence $\E\big[\lVert\Psivec(\Xi; e^*, \gvec^*; \betavec^*)\rVert^2\big] < \infty$, and the envelope $\lVert\Psivec(\cdot; e^*, \gvec^*; \betavec^*)\rVert + m(\cdot)\,\mathrm{diam}\big(\overline{B(\betavec^*, \epsilon)}\big)$ is square-integrable.
Therefore, by Example 19.7 and Theorem 19.5 in \cite{van2000asymptotic}, $\left\{
\Psivec(\cdot; \betavec): \betavec \in B(\betavec^*, \epsilon)
\right\}$
forms a Donsker class. By Lemma \ref{lem: consistency}, $\hat{\betavec} \to_{P} \betavec^*$.
Then, by Theorem 19.9 in \cite{van2000asymptotic}, it holds that
$
\Gn\Psivec(\hat{\betavec}) - \Gn\Psivec( {\betavec}^*) \to_P 0.
$
Thus, 
$$
\sqrt{n}(\E[\Psivec( \hat{\betavec})] - \E[\Psivec( {\betavec}^*)])\,=\,-\big(\Gn\Psivec( \hat{\betavec}) - \Gn\Psivec({\betavec}^*)\big) + \sqrt{n}\Pe_n \Psivec\big( \hat{\betavec}\big) - \sqrt{n}\Pe_n \Psivec\big( {\betavec}^*\big)\,\stackrel{\text(a)}{=}\,-\sqrt{n}\Pe_n \Psivec\big( {\betavec}^*\big) +o_P(1),
$$
where (a) follows from Lemma \ref{lem: empirical-scores}.
We apply the Taylor expansion on the left-side of the equation, by continuous mapping, we have 
\begin{align}
\label{eqn:expansion}
    \sqrt{n}\left[
\Big(\E[\Xvec^{\top} {\nabla^2} b(\Xvec\betavec^*) \Xvec]\Big)\big(\hat{\betavec} - \betavec^* \big) + o_P(1)\left \lVert \hat{\betavec} - \betavec^*\right\rVert\right]\,=\,-\sqrt{n}\Pe_n \Psivec\big( {\betavec}^*\big) +o_P(1)\,=\,O_P(1).
\end{align}
As shown in the proof of Lemma \ref{lem: empirical-scores}, $\E[\Xvec^{\top} {\nabla^2} b(\Xvec\betavec^*) \Xvec]$ is invertible.
Therefore, \eqref{eqn:expansion} implies that $\sqrt{n}\left\lVert \hat{\betavec} - \betavec^*\right\rVert = O_P(1)$.
Therefore, applying \eqref{eqn:expansion} again, 
$$\sqrt{n}\left(\hat{\betavec} - \betavec^*\right)\,=\,-\Big(\E[\Xvec^{\top} {\nabla^2} b(\Xvec\betavec^*) \Xvec]\Big)^{-1}\sqrt{n}\Pe_n \Psivec\big( {\betavec}^*\big) +o_P(1).$$
Finally, since $\E[\Psivec(e^*, \gvec^*; \betavec^*)] = 0$ by Lemma~\ref{lem: orthogonality} and $\E\big[\lVert\Psivec(e^*, \gvec^*; \betavec^*)\rVert^2\big] < \infty$ as established above, the multivariate central limit theorem gives $\sqrt{n}\Pe_n \Psivec(\betavec^*) \rightsquigarrow N\big(\bm{0}, \E[\Psivec\Psivec^{\top}]\big)$ with $\E[\Psivec\Psivec^{\top}]$ finite; in particular $\sqrt{n}\Pe_n \Psivec(\betavec^*) = O_P(1)$, which justifies the order bound asserted in \eqref{eqn:expansion}. Combined with the invertibility of $\mathbf{J} = \E[\Xvec^{\top}{\nabla^2} b(\Xvec\betavec^*)\Xvec]$, the displayed linearization and Slutsky's theorem yield $\sqrt{n}(\hat{\betavec} - \betavec^*) \rightsquigarrow N\big(\bm{0}, \Sigmavec^{\text{\upshape GAI}}\big)$ with $\Sigmavec^{\text{\upshape GAI}} = \mathbf{J}^{-1}\E[\Psivec\Psivec^{\top}]\mathbf{J}^{-1}$ finite, as claimed in Theorem~\ref{thm:normality}.

\subsection{\texorpdfstring{\rev{Dominance over Primary-Only}}{Dominance over Primary-Only}}
\label{apdx: cor dominance}

\rev{The following corollary strengthens Theorem~\ref{thm:normality} under known $\rho$ (dropping the rate and sup-norm conditions of Assumption~\ref{ass:ml_rate}) and records the comparison with primary-only estimation. Its dominance conclusion is the primary-only special case of Theorem~\ref{thm:dominance}, but we prove it directly from the analysis of Theorem~\ref{thm:normality}, both because of the weaker assumptions and because the direct argument is reused in the covariate-shift proofs.}

\begin{corollary}[Dominance over Primary-Only]
\label{cor:dominance}
    Suppose Assumptions~\ref{ass:regularity} and~\ref{ass:random_labeling} hold and that $\hat{\gvec}$ is $L_2$-consistent, $\lVert \hat{\gvec} - \gvec^* \rVert_{Q,2} = o(1)$, with $\hat{e} \equiv \rho$. Then the conclusion of Theorem~\ref{thm:normality} continues to hold \emph{without} the rate and sup-norm conditions of Assumption~\ref{ass:ml_rate}. Moreover, let $\hat{\betavec}^{\text{\upshape{P}}}$ be the primary-only estimator obtained from human data alone, i.e., a solution to $\sum_{i \in \mD^{\text{\upshape P}}}\nabla_{\betavec}\ell(\rev{\Xvec_i, \yvec_i}; \betavec)\,=\,0$. Then $\sqrt{n}\left(\hat{\betavec}^{\text{\upshape{P}}} - \betavec^*\right)\leadsto N\left(\bm{0},\, \Sigmavec^{\text{\upshape P}}\right)$, where $\Sigmavec^{\text{\upshape P}} = \frac{1}{\rho}\Jvec^{-1} \E\!\left[
    \nabla_{\betavec}\ell(\Xvec, \yvec; \betavec^*)\nabla_{\betavec}\ell(\Xvec, \yvec; \betavec^*)^{\top}
    \right] \Jvec^{-1}$ and $\Sigmavec^{\text{\upshape P}} \succeq \Sigmavec^{\text{\upshape GAI}}$.  
    Also, $\Sigmavec^{\text{\upshape P}} \succ \Sigmavec^{\text{\upshape GAI}}$ as long as $\rho < 1$ and 
$$    \E \bigg[
    \Xvec^{\top}
    \left(
    \nabla b\big(\Xvec\betavec^*\big) - \E[\yvec \,|\, \Xvec, \zvec]
    \right)\cdot\notag
    \left(
    \nabla b\big(\Xvec\betavec^*\big) - \E[\yvec \,|\, \Xvec, \zvec]
    \right)^{\top}
    \Xvec
    \bigg]\,\succ\,0.$$
\end{corollary}


\subsubsection{Proof of Corollary \ref{cor:dominance}}
\label{apdx: proof of dominance}

\emph{Asymptotic normality of the GAI estimator.} We first record why the conclusion of Theorem~\ref{thm:normality} continues to hold under Assumption~\ref{ass:random_labeling} with Assumption~\ref{ass:ml_rate} weakened to the $L_2$-consistency of $\hat{\gvec}$. With $\hat{e} \equiv e^* \equiv \rho$ known, the proof of Theorem~\ref{thm:normality} changes only in the orthogonality step (Lemma~\ref{lem: orthogonality}), at three points: (i) the remainder $\Pe_n\tauvec(\hat{e}, \hat{\gvec}) - \Pe_n\tauvec(e^*, \gvec^*)$ collapses to term (I) alone, which is \emph{linear} in $\hat{\gvec} - \gvec^*$ and is $o_P(n^{-1/2})$ under $\lVert \hat{\gvec} - \gvec^* \rVert_{Q,2} = o(1)$ by the conditional Markov argument there, requiring no convergence rate; (ii) the second-order remainder term (III) is absent, as it carries the factor $\hat{e} - e^* = 0$; and (iii) term (II), driven by the estimation error of $\hat{e}$, vanishes for the same reason, so no rate or sup-norm condition on a propensity estimate enters. All other steps---consistency (Lemma~\ref{lem: consistency}), the Donsker class and Taylor expansion, and the central limit theorem---are unchanged, and $\sqrt{n}(\hat{\betavec} - \betavec^*) \rightsquigarrow N(\bm{0}, \Sigmavec^{\text{\upshape GAI}})$ follows as in Theorem~\ref{thm:normality}.

\emph{Dominance.} To start, we first note that the \rev{empirical score equation}, i.e., $\sum_{i \in \mD^{\text{\upshape P}}}\nabla_{\betavec}\ell(\rev{\Xvec_i, \yvec_i}; \betavec)\,=\,0$, is equivalent to
$
    \sum_{i=1}^n w_i\nabla_{\betavec}\ell(\rev{\Xvec_i, \yvec_i}; \betavec)\,=\,0.
$
Note that by independence of $w$ with respect to other random variables, it holds that
$$
\E\left[
w\nabla_{\betavec}\ell(\Xvec, \yvec; \rev{\betavec^*})\right]\,=\,\rho\E\left[
\nabla_{\betavec}\ell(\Xvec, \yvec; \betavec^*)\right]\,=\,0.
$$
Thus, under the preset assumptions, it is easy to show that with probability approaching one, the solution to this equation, i.e., $\hat{\betavec}^{\text{P}}$, exists, and $\sqrt{n}\left(\hat{\betavec}^{\text{P}} - \betavec^*\right)\leadsto N\left(\bm{0},\, \Sigmavec^{\text{\upshape P}}\right)$.
We skip the proof.
We further note that 
$
\nabla_{\betavec}\ell(\Xvec, \yvec; \betavec^*) := \zetavec(\Xvec, \zvec) + \pivec(\Xvec, \yvec, \zvec)
$ where
$\zetavec(\Xvec, \zvec) := \Xvec^{\top}\left({\nabla} b(\Xvec\betavec^*) - \rev{\gvec^*}(\Xvec, \zvec)\right)$
and
$\pivec(\Xvec, \yvec, \zvec) = \Xvec^{\top}\left(\rev{\gvec^*}(\Xvec, \zvec) - \yvec\right)$ and, because $\E\left[\rev{\gvec^*}(\Xvec, \zvec) - \yvec|\Xvec, \zvec\right] = 0$, we have
\begin{align}
\label{eqn:variance-primary-expansion}
    \Sigmavec^{\text{\upshape P}}\,=\,\frac{1}{\rho}\bigg(\Jvec^{-1}\E\left[
\zetavec(\Xvec, \zvec)\zetavec(\Xvec, \zvec)^{\top}
\right]\Jvec^{-1} +
\Jvec^{-1}\E\left[
\pivec(\Xvec,\yvec, \zvec)\pivec(\Xvec,\yvec, \zvec)^{\top}
\right]\Jvec^{-1}\bigg).
\end{align}
Furthermore, observe that 
$
\Psivec(\Xi; e, \gvec; \betavec) = \zetavec(\Xvec, \zvec) + \frac{w}{\rho}\pivec(\Xvec, \yvec, \zvec)
$, so 
\begin{align}
\label{eqn:variance-gai-expansion}
\Sigmavec^{\text{\upshape GAI}}\,=\,\Jvec^{-1}\E\left[
\zetavec(\Xvec, \zvec)\zetavec(\Xvec, \zvec)^{\top}
\right]\Jvec^{-1} +\frac{1}{\rho}
\Jvec^{-1}\E\left[
\pivec(\Xvec,\yvec, \zvec)\pivec(\Xvec,\yvec, \zvec)^{\top}
\right]\Jvec^{-1}.
\end{align}
The conclusion follows straightforwardly given \eqref{eqn:variance-primary-expansion} and \eqref{eqn:variance-gai-expansion}.

\subsection{Proof of Theorem~\ref{thm:dominance}}
\label{apdx: proof of unified dominance}

Throughout we work under Assumption~\ref{ass:regularity} and Assumption~\ref{ass:random_labeling}, so $e \equiv \rho \in (0,1)$ is known, $w \perp (\Xvec, \yvec, \zvec)$, and $\Jvec = \E[\Xvec^\top \nabla^2 b(\Xvec\betavec^*)\Xvec] \succ 0$. We use the full-data score $\mathbf{S} := \Xvec^\top(\nabla b(\Xvec\betavec^*) - \yvec) = \zetavec + \pivec$ with $\zetavec = \Xvec^\top(\nabla b(\Xvec\betavec^*) - \gvec^*)$, $\pivec = \Xvec^\top(\gvec^* - \yvec)$, $\gvec^* = \E[\yvec \mid \Xvec, \zvec]$, so that $\E[\pivec \mid \Xvec, \zvec] = \mathbf{0}$. We abbreviate the $(\Xvec, \zvec)$-measurable score correction by $\mathbf{h} := \lambda\Xvec^\top(\nabla b(\Xvec\betavec^*) - \fvec)$, where $\fvec = \fvec(\Xvec, \zvec)$ does not depend on $\yvec$.

\emph{Step 1 (mean-zero at the common $\betavec^*$).} Evaluating the family score \eqref{eq:ppi_family_score} at $\betavec^*$ and taking expectations,
\begin{align*}
\E\big[\Psivec_{\lambda, \fvec}(\Xi; \betavec^*)\big]
\,&\stackrel{\text{(a)}}{=}\,
\E\!\left[\frac{w}{\rho}\right]\E\big[\Xvec^\top(\nabla b(\Xvec\betavec^*) - \yvec)\big]
- \E\!\left[\frac{w}{\rho}\right]\E[\mathbf{h}]
+ \E\!\left[\frac{1 - w}{1 - \rho}\right]\E[\mathbf{h}]\\
\,&\stackrel{\text{(b)}}{=}\,
\E\big[\Xvec^\top(\nabla b(\Xvec\betavec^*) - \yvec)\big] - \E[\mathbf{h}] + \E[\mathbf{h}]
\,=\,\E[\mathbf{S}]\,\stackrel{\text{(c)}}{=}\,\mathbf{0},
\end{align*}
where $(a)$ factors the propensity weights out of the $(\Xvec, \yvec, \zvec)$-measurable terms using $w \perp (\Xvec, \yvec, \zvec)$ and writes $\mathbf{h} = \lambda\Xvec^\top(\nabla b(\Xvec\betavec^*) - \fvec)$, $(b)$ uses $\E[w/\rho] = \E[(1 - w)/(1 - \rho)] = 1$ so that the two $\lambda$-terms cancel, and $(c)$ is the population first-order condition $\E[\Xvec^\top(\nabla b(\Xvec\betavec^*) - \yvec)] = \mathbf{0}$ defining $\betavec^*$. Hence every member of the family has a mean-zero estimating equation at the \emph{same} $\betavec^*$. 

Furthermore, for any $\betavec$, again by $w \perp (\Xvec, \yvec, \zvec)$ and $\E[w/\rho] = \E[(1 - w)/(1 - \rho)] = 1$, the two $\lambda$-terms of $\E[\Psivec_{\lambda, \fvec}(\Xi; \betavec)]$ cancel exactly as in Step~1, so the population estimating function of \emph{every} member coincides with that of the primary GLM: $\bar{\Psivec}_{\lambda, \fvec}(\betavec) := \E[\Psivec_{\lambda, \fvec}(\Xi; \betavec)] = \E[\Xvec^\top(\nabla b(\Xvec\betavec) - \yvec)] = \nabla L(\betavec)$, where $L(\betavec) := \E[\ell(\Xvec, \yvec; \betavec)]$. By Assumption~\ref{ass:regularity}(i)--(ii), $\nabla^2 L(\betavec) = \E[\Xvec^\top\nabla^2 b(\Xvec\betavec)\Xvec] \succ 0$ for every $\betavec \in \mB$ (the fact established in the proof of Lemma~\ref{lem: empirical-scores}), so $L$ is strictly convex on the convex set $\mB$ with unique minimizer $\betavec^*$, which is thus the unique zero of $\bar{\Psivec}_{\lambda, \fvec}$ on $\mB$, with common nonsingular derivative $\nabla\bar{\Psivec}_{\lambda, \fvec}(\betavec^*) = \Jvec$, which is straightforward to verify by algebra again using $w \perp (\Xvec, \yvec, \zvec)$ and $\E[w/\rho] = \E[(1 - w)/(1 - \rho)] = 1$. 

\emph{Step 2 (consistency and asymptotic normality).}
Because $\lambda \in [0, 1]$, one can easily verify that $\bar{\Psivec}_{\lambda, \fvec}(\betavec) = \tilde{\ell}_{\lambda, \fvec}(\Xi; \betavec)$, where $\tilde{\ell}_{\lambda, \fvec}(\Xi; \betavec)$ is convex in $\betavec$ and defined as 
$$
\tilde{\ell}_{\lambda, \fvec}(\Xi; \betavec) \,:=\,
\left(
\frac{w}{\rho}(1 - \lambda) + \frac{1-w}{1-\rho}\lambda
\right)b(\Xvec\betavec) - \frac{w}{\rho}\yvec^{\top}\Xvec\betavec + \left(
\frac{w}{\rho} + \frac{1-w}{1-\rho}
\right)\lambda \fvec^{\top}\Xvec \betavec.
$$
Therefore, mimicking the steps in the proof of Lemma \ref{lem: empirical-scores}, we can show that with probability approaching one, there is a solution, $\hat\betavec_{\lambda, \fvec}$, to the minimization problem and gives $\Pe_n \Psivec_{\lambda, \fvec}(\cdot\,; \betavec)= \bm{0}$.
This implies that with probability that goes to one, it also solves 
$\min_{\beta \in \mB}\Pe_n\tilde{\ell}_{\lambda, \fvec}(\cdot; \betavec)$, whose solution is consistent by Theorem 2.7 of \cite{newey1994large}. 

Since $\rho \in (0,1)$ is known, the weights $w/\rho$ and $(1 - w)/(1 - \rho)$ are bounded by $1/\rho$ and $1/(1 - \rho)$. Grouping \eqref{eq:ppi_family_score} at $\betavec^*$ around $\mathbf{S} = \Xvec^\top(\nabla b(\Xvec\betavec^*) - \yvec)$ and $\mathbf{h} = \lambda\Xvec^\top(\nabla b(\Xvec\betavec^*) - \fvec)$ gives $\Psivec_{\lambda, \fvec}(\Xi; \betavec^*) = \tfrac{w}{\rho}(\mathbf{S} - \mathbf{h}) + \tfrac{1 - w}{1 - \rho}\mathbf{h}$. On the compact set $\Breve{\mX}$ we have $\lVert\Xvec\rVert_F \leq C$ and, by continuity of $\nabla b$, $\Xvec^\top\nabla b(\Xvec\betavec^*)$ is bounded; moreover $\E\lVert\Xvec^\top\yvec\rVert^2 \leq C^2\,\E\lVert\yvec\rVert^2 < \infty$ by Assumption~\ref{ass:regularity}(iv), and $\E\lVert\Xvec^\top\fvec\rVert^2 \leq C^2\,\E\lVert\fvec\rVert^2 < \infty$ because $\fvec \in L_2(\Xvec, \zvec)$. Hence $\E\big[\lVert\Psivec_{\lambda, \fvec}(\Xi; \betavec^*)\rVert^2\big] < \infty$, so $\boldsymbol{\Omega}(\lambda, \fvec)$ is finite. The only $\betavec$-dependent part of $\Psivec_{\lambda, \fvec}$ is $\big[\tfrac{w}{\rho}(1 - \lambda) + \tfrac{1 - w}{1 - \rho}\lambda\big]\Xvec^\top\nabla b(\Xvec\betavec)$, a bounded random multiple of $\Xvec^\top\nabla b(\Xvec\betavec)$; the local Lipschitz and envelope bounds of the proof of Theorem~\ref{thm:normality} therefore apply verbatim, and $\{\Psivec_{\lambda, \fvec}(\cdot; \betavec) : \betavec \in \overline{B(\betavec^*, \epsilon)}\}$ is a Donsker class with square-integrable envelope by Example 19.7 and Theorem 19.5 in \cite{van2000asymptotic}.
Then, the rest of the proof is the same as that of Theorem \ref{thm:normality}.
The form of the asymptotic matrix is easy to compute directly. 

\emph{Step 3 (closed form for $\boldsymbol{\Omega}$).} Regrouping \eqref{eq:ppi_family_score} at $\betavec^*$ around $\mathbf{S}$ and $\mathbf{h}$,
\begin{equation}
\label{eq:ppi_psi_regroup}
\Psivec_{\lambda, \fvec}(\Xi; \betavec^*)
\,=\,
\frac{w}{\rho}\big(\mathbf{S} - \mathbf{h}\big) \,+\, \frac{1 - w}{1 - \rho}\,\mathbf{h}.
\end{equation}
Taking the outer product and the expectation,
\begin{align}
\boldsymbol{\Omega}(\lambda, \fvec)
\,&\stackrel{\text{(a)}}{=}\,
\E\!\left[\frac{w^2}{\rho^2}(\mathbf{S} - \mathbf{h})(\mathbf{S} - \mathbf{h})^\top\right]
+ \E\!\left[\frac{w(1 - w)}{\rho(1 - \rho)}\big((\mathbf{S} - \mathbf{h})\mathbf{h}^\top + \mathbf{h}(\mathbf{S} - \mathbf{h})^\top\big)\right]
+ \E\!\left[\frac{(1 - w)^2}{(1 - \rho)^2}\mathbf{h}\mathbf{h}^\top\right]\notag\\
\,&\stackrel{\text{(b)}}{=}\,
\frac{1}{\rho}\,\E\big[(\mathbf{S} - \mathbf{h})(\mathbf{S} - \mathbf{h})^\top\big]
\,+\, \frac{1}{1 - \rho}\,\E[\mathbf{h}\mathbf{h}^\top],
\label{eq:ppi_omega_step}
\end{align}
where $(a)$ expands the outer product of \eqref{eq:ppi_psi_regroup}, and $(b)$ uses $w \in \{0, 1\}$, so that $w^2 = w$, $(1 - w)^2 = 1 - w$, and $w(1 - w) = 0$, whence the cross term vanishes identically; then $w \perp (\Xvec, \yvec, \zvec)$ gives $\E[w^2/\rho^2] = \E[w/\rho^2] = 1/\rho$ and $\E[(1 - w)^2/(1 - \rho)^2] = 1/(1 - \rho)$, factoring out of the $(\Xvec, \yvec, \zvec)$-measurable blocks. This is the second equality of \eqref{eq:ppi_omega_closed}. To obtain the first, complete the square around $(1 - \rho)\mathbf{S}$ by writing $\boldsymbol{\delta} := \mathbf{h} - (1 - \rho)\mathbf{S}$, so $\mathbf{S} - \mathbf{h} = \rho\mathbf{S} - \boldsymbol{\delta}$ and $\mathbf{h} = (1 - \rho)\mathbf{S} + \boldsymbol{\delta}$. Substituting into \eqref{eq:ppi_omega_step},
\begin{align*}
\boldsymbol{\Omega}(\lambda, \fvec)
\,&\stackrel{\text{(c)}}{=}\,
\frac{1}{\rho}\,\E\big[(\rho\mathbf{S} - \boldsymbol{\delta})(\rho\mathbf{S} - \boldsymbol{\delta})^\top\big]
+ \frac{1}{1 - \rho}\,\E\big[((1 - \rho)\mathbf{S} + \boldsymbol{\delta})((1 - \rho)\mathbf{S} + \boldsymbol{\delta})^\top\big]\\
\,&\stackrel{\text{(d)}}{=}\,
\big(\rho + (1 - \rho)\big)\E[\mathbf{S}\mathbf{S}^\top]
+ \big(-1 + 1\big)\big(\E[\mathbf{S}\boldsymbol{\delta}^\top] + \E[\boldsymbol{\delta}\mathbf{S}^\top]\big)
+ \left(\frac{1}{\rho} + \frac{1}{1 - \rho}\right)\E[\boldsymbol{\delta}\boldsymbol{\delta}^\top]\\
\,&\stackrel{\text{(e)}}{=}\,
\E[\mathbf{S}\mathbf{S}^\top] \,+\, \frac{1}{\rho(1 - \rho)}\,\E[\boldsymbol{\delta}\boldsymbol{\delta}^\top],
\end{align*}
where $(c)$ substitutes the two reparametrized increments, $(d)$ collects the $\mathbf{S}\mathbf{S}^\top$, cross, and $\boldsymbol{\delta}\boldsymbol{\delta}^\top$ blocks (the cross-block coefficient is $\tfrac{1}{\rho}(-\rho) + \tfrac{1}{1 - \rho}(1 - \rho) = -1 + 1 = 0$), and $(e)$ uses $\rho + (1 - \rho) = 1$ together with $\tfrac{1}{\rho} + \tfrac{1}{1 - \rho} = \tfrac{1}{\rho(1 - \rho)}$. This is exactly \eqref{eq:ppi_omega_closed} and completes part~(i).

\emph{Step 4 (Pythagoras and Loewner minimization).} Resolve $\boldsymbol{\delta}$ into its $(\Xvec, \zvec)$-measurable and conditionally mean-zero parts. With the residual $\mathbf{r} := \lambda(\nabla b(\Xvec\betavec^*) - \fvec) - (1 - \rho)(\nabla b(\Xvec\betavec^*) - \gvec^*) \in \R^{\rev{k}}$,
\begin{equation}
\label{eq:ppi_delta_decomp}
\boldsymbol{\delta}
\,=\,\mathbf{h} - (1 - \rho)\mathbf{S}
\,\stackrel{\text{(a)}}{=}\,
\big(\Xvec^\top\mathbf{r}\big) - (1 - \rho)\pivec,
\end{equation}
where $(a)$ writes $\mathbf{h} = \lambda\Xvec^\top(\nabla b(\Xvec\betavec^*) - \fvec)$ and $(1 - \rho)\mathbf{S} = (1 - \rho)\Xvec^\top(\nabla b(\Xvec\betavec^*) - \yvec) = (1 - \rho)\Xvec^\top(\nabla b(\Xvec\betavec^*) - \gvec^*) + (1 - \rho)\Xvec^\top(\gvec^* - \yvec)$, then collects the $(\Xvec, \zvec)$-measurable terms into $\Xvec^\top\mathbf{r}$ and identifies the remainder $-(1 - \rho)\Xvec^\top(\gvec^* - \yvec) = -(1 - \rho)\pivec$. Since $\Xvec^\top\mathbf{r}$ is $(\Xvec, \zvec)$-measurable and $\E[\pivec \mid \Xvec, \zvec] = \mathbf{0}$,
\begin{equation}
\label{eq:ppi_pythagoras}
\E[\boldsymbol{\delta}\boldsymbol{\delta}^\top]
\,\stackrel{\text{(b)}}{=}\,
\E\big[(\Xvec^\top\mathbf{r})(\Xvec^\top\mathbf{r})^\top\big]
\,+\, (1 - \rho)^2\,\E[\pivec\pivec^\top],
\end{equation}
where $(b)$ is the Pythagorean identity: the cross terms $\E[(\Xvec^\top\mathbf{r})\pivec^\top]$ and $\E[\pivec(\Xvec^\top\mathbf{r})^\top]$ vanish by the law of iterated expectations, since $\Xvec^\top\mathbf{r}$ passes through the conditional expectation given $(\Xvec, \zvec)$ and $\E[\pivec \mid \Xvec, \zvec] = \mathbf{0}$. The first summand of \eqref{eq:ppi_pythagoras} is the only $(\lambda, \fvec)$-dependent term, is positive semidefinite, and is Loewner-minimized---in \emph{every} direction simultaneously---at its zero value $\Xvec^\top\mathbf{r} = \mathbf{0}$ a.s. Combining $\boldsymbol{\delta} = \Xvec^\top\mathbf{r} - (1 - \rho)\pivec$ from \eqref{eq:ppi_delta_decomp} with $\mathbf{h} = (1 - \rho)\mathbf{S} + \boldsymbol{\delta}$ and $\mathbf{S} - \pivec = \zetavec$ gives the identity $\mathbf{h} = (1 - \rho)\zetavec + \Xvec^\top\mathbf{r}$; hence $\Xvec^\top\mathbf{r} = \mathbf{0}$ a.s.\ is equivalently $\mathbf{h} = \mathbf{h}^\star := (1 - \rho)\zetavec$, which corresponds to $\boldsymbol{\delta} = -(1 - \rho)\pivec$. This optimum lies inside the family: taking $\lambda = 1 - \rho$ and $\fvec = \gvec^*$ gives $\mathbf{h} = (1 - \rho)\Xvec^\top(\nabla b(\Xvec\betavec^*) - \gvec^*) = (1 - \rho)\zetavec = \mathbf{h}^\star$ and $\mathbf{r} = \mathbf{0}$. Substituting $\Xvec^\top\mathbf{r} = \mathbf{0}$ into \eqref{eq:ppi_omega_closed} through \eqref{eq:ppi_pythagoras},
\begin{equation*}
\boldsymbol{\Omega}^{\text{\upshape GAI}}
\,\stackrel{\text{(c)}}{=}\,
\E[\mathbf{S}\mathbf{S}^\top] + \frac{1}{\rho(1 - \rho)}\,(1 - \rho)^2\,\E[\pivec\pivec^\top]
\,\stackrel{\text{(d)}}{=}\,
\E[\zetavec\zetavec^\top] + \E[\pivec\pivec^\top] + \frac{1 - \rho}{\rho}\,\E[\pivec\pivec^\top]
\,\stackrel{\text{(e)}}{=}\,
\E[\zetavec\zetavec^\top] + \frac{1}{\rho}\,\E[\pivec\pivec^\top],
\end{equation*}
where $(c)$ sets the $(\lambda, \fvec)$-dependent block to zero, $(d)$ uses the orthogonal decomposition $\E[\mathbf{S}\mathbf{S}^\top] = \E[\zetavec\zetavec^\top] + \E[\pivec\pivec^\top]$ (valid because $\mathbf{S} = \zetavec + \pivec$ with $\E[\pivec \mid \Xvec, \zvec] = \mathbf{0}$, so the $\zetavec$--$\pivec$ cross terms vanish) and simplifies $\tfrac{(1 - \rho)^2}{\rho(1 - \rho)} = \tfrac{1 - \rho}{\rho}$, and $(e)$ combines $1 + \tfrac{1 - \rho}{\rho} = \tfrac{1}{\rho}$ on the $\E[\pivec\pivec^\top]$ block. This is \eqref{eq:ppi_omega_gai}. For the gap, subtract the GAI values from \eqref{eq:ppi_omega_closed} and \eqref{eq:ppi_pythagoras} and conjugate by $\Jvec^{-1}$:
\begin{equation*}
\Sigmavec(\lambda, \fvec) - \Sigmavec^{\text{\upshape GAI}}
\,\stackrel{\text{(f)}}{=}\,
\Jvec^{-1}\big(\boldsymbol{\Omega}(\lambda, \fvec) - \boldsymbol{\Omega}^{\text{\upshape GAI}}\big)\Jvec^{-1}
\,\stackrel{\text{(g)}}{=}\,
\frac{1}{\rho(1 - \rho)}\,\Jvec^{-1}\,\E\big[(\Xvec^\top\mathbf{r})(\Xvec^\top\mathbf{r})^\top\big]\,\Jvec^{-1}
\,\stackrel{\text{(h)}}{\succeq}\,\mathbf{0},
\end{equation*}
where $(f)$ uses $\Sigmavec = \Jvec^{-1}\boldsymbol{\Omega}\Jvec^{-1}$ with the common $\Jvec$ from Step~2, $(g)$ cancels the $(\lambda, \fvec)$-free blocks $\E[\mathbf{S}\mathbf{S}^\top]$ and $\tfrac{(1 - \rho)^2}{\rho(1 - \rho)}\E[\pivec\pivec^\top]$ between $\boldsymbol{\Omega}(\lambda, \fvec)$ and $\boldsymbol{\Omega}^{\text{\upshape GAI}}$ via \eqref{eq:ppi_pythagoras}, leaving only $\tfrac{1}{\rho(1 - \rho)}\E[(\Xvec^\top\mathbf{r})(\Xvec^\top\mathbf{r})^\top]$, and $(h)$ holds because $\E[(\Xvec^\top\mathbf{r})(\Xvec^\top\mathbf{r})^\top]$ is positive semidefinite, $\tfrac{1}{\rho(1 - \rho)} > 0$, and congruence by $\Jvec^{-1}$ preserves the Loewner order. This is \eqref{eq:ppi_gap}. Equality holds if and only if the only $(\lambda, \fvec)$-dependent block vanishes, i.e. $\E[(\Xvec^\top\mathbf{r})(\Xvec^\top\mathbf{r})^\top] = \mathbf{0}$, equivalently $\Xvec^\top\mathbf{r} = \mathbf{0}$ almost surely; for GAI, $\mathbf{r} = \mathbf{0}$ identically. This proves part~(iii), and hence the weak Loewner dominance over every member, with strictness whenever the corresponding $\Xvec^\top\mathbf{r} \neq \mathbf{0}$ on a set of positive probability.

The optimal correction $\mathbf{h}^\star = (1 - \rho)\zetavec$, and thus $\boldsymbol{\Omega}^{\text{\upshape GAI}}$ and $\Sigmavec^{\text{\upshape GAI}}$, is unique because $\Xvec^\top\mathbf{r} = \mathbf{0}$ a.s.\ pins down $\mathbf{h}$. The pair $(\lambda, \fvec)$ attaining it is not unique: $\Xvec^\top\mathbf{r} = \mathbf{0}$ holds whenever $\lambda(\nabla b(\Xvec\betavec^*) - \fvec) = (1 - \rho)(\nabla b(\Xvec\betavec^*) - \gvec^*)$, i.e.\ for any $\lambda \neq 0$ with $\fvec = \nabla b(\Xvec\betavec^*) - \tfrac{1 - \rho}{\lambda}(\nabla b(\Xvec\betavec^*) - \gvec^*)$, of which GAI ($\lambda = 1 - \rho \in (0,1)$, $\fvec = \gvec^*$) is the canonical representative.

\emph{Step 5 (special-case collapses).} We verify the two members directly from \eqref{eq:ppi_family_score} before. 
Note that although PPI and PPI++ are formulated as convex loss minimization problems (M-estimators), their first-order conditions are exactly special cases of \eqref{eq:ppi_family_score} with the specified $(\lambda, \fvec)$.
The asymptotic distribution of the M-estimator formulation is asymptotically equivalent to our Z-estimator formulation.

\emph{(i) GAI, $\lambda = 1 - \rho$, $\fvec = \gvec^*$.} Substituting,
\begin{align*}
\Psivec_{1 - \rho, \gvec^*}(\Xi; \betavec)
\,&\stackrel{\text{(a)}}{=}\,
\frac{w}{\rho}\Xvec^\top\!\Big[(\nabla b(\Xvec\betavec) - \yvec) - (1 - \rho)(\nabla b(\Xvec\betavec) - \gvec^*)\Big]
+ \frac{1 - w}{1 - \rho}(1 - \rho)\Xvec^\top(\nabla b(\Xvec\betavec) - \gvec^*)\\
\,&\stackrel{\text{(b)}}{=}\,
\frac{w}{\rho}\Xvec^\top\big[\rho\,\nabla b(\Xvec\betavec) - \yvec + (1 - \rho)\gvec^*\big]
+ (1 - w)\Xvec^\top(\nabla b(\Xvec\betavec) - \gvec^*)\\
\,&\stackrel{\text{(c)}}{=}\,
\Xvec^\top\Big[\nabla b(\Xvec\betavec) - \gvec^* + \frac{w}{\rho}(\gvec^* - \yvec)\Big],
\end{align*}
where $(a)$ inserts $\lambda = 1 - \rho$, $\fvec = \gvec^*$, $(b)$ expands the bracket inside $w/\rho$ as $(1 - (1 - \rho))\nabla b(\Xvec\betavec) - \yvec + (1 - \rho)\gvec^* = \rho\nabla b(\Xvec\betavec) - \yvec + (1 - \rho)\gvec^*$ and cancels $(1 - \rho)$ in the second term, and $(c)$ groups the $\nabla b$ and $\gvec^*$ coefficients: the $\nabla b(\Xvec\betavec)$ terms give $\tfrac{w}{\rho}\rho + (1 - w) = w + 1 - w = 1$, and the $\gvec^*$ terms give $\tfrac{w}{\rho}(1 - \rho) - (1 - w) = \tfrac{w}{\rho} - w - 1 + w = \tfrac{w}{\rho} - 1$, so the constant-in-$\betavec$ part is $-\gvec^* + \tfrac{w}{\rho}(\gvec^* - \yvec)$ after combining with the $-\tfrac{w}{\rho}\yvec$ term. This is exactly the GAI score \eqref{eq:score} with $e \equiv \rho$.

\emph{(ii) Primary-only, $\lambda = 0$.} Setting $\lambda = 0$ in \eqref{eq:ppi_family_score} annihilates $\mathbf{h}$ and the unlabeled block, leaving $\Psivec_{0, \fvec}(\Xi; \betavec) = \tfrac{w}{\rho}\Xvec^\top(\nabla b(\Xvec\betavec) - \yvec)$, the IPW human-labels-only score, whose asymptotic variance is $\Sigmavec^{\text{\upshape P}}$ from Corollary~\ref{cor:dominance}; this member uses no prediction $\fvec$, consistent with the entry ``$\lambda = 0$'' of part~(ii).

\emph{Step 6 (feasible GAI).} Part~(iv) concerns the feasible GAI estimator $\hat{\betavec}$ computed by Algorithm~\ref{alg: gai} with the known propensity $\hat{e} \equiv \rho$ and a cross-fitted nuisance $\hat{\gvec}$, in place of the oracle member $(1 - \rho, \gvec^*)$ analyzed in Steps~1--5. By Corollary~\ref{cor:dominance} (proved in Section~\ref{apdx: proof of dominance}, using exactly the $L_{Q,2}$-consistency of $\hat{\gvec}$ and $\hat{e} \equiv \rho$ under Assumption~\ref{ass:random_labeling}, and no rate or sup-norm condition), this feasible estimator satisfies $\sqrt{n}(\hat{\betavec} - \betavec^*) \rightsquigarrow N(\mathbf{0}, \Sigmavec^{\text{\upshape GAI}})$, the same limit law as the oracle GAI member of parts~(ii)--(iii). Hence the Loewner dominance established in part~(iii) applies verbatim to $\hat{\betavec}$, which proves part~(iv).
This completes the proof of Theorem~\ref{thm:dominance}.

\subsection{Proof of Lemma \ref{lem: orthogonality}}

To begin, first we note that by the definition of $\tauvec(e^*, \gvec^*)$
\begin{align*}
    \E[\tauvec(e^*, \gvec^*)]&\,=\,\E\left[\Xvec^{\top}
    \left[
    \gvec^*(\Xvec,\zvec) - \frac{w}{e^*(\Xvec,\zvec)}(\gvec^*(\Xvec,\zvec) - \yvec) 
    \right]\right]\\
    &
\,\stackrel{\text{(a)}}{=}\,\E\left[\Xvec^{\top}
    \left[
    \gvec^*(\Xvec,\zvec) - \frac{\E[w|\Xvec,\zvec]}{e^*(\Xvec,\zvec)}\big(\gvec^*(\Xvec,\zvec) - \E[\yvec|\Xvec,\zvec]\big) 
    \right]\right]\,\stackrel{\text{(b)}}{=}\,\E\left[\Xvec^{\top}
    \E[\yvec|\Xvec,\zvec]\right]\,=\,\E\left[\Xvec^{\top}\yvec\right],
\end{align*}
where (a) follows from the law of iterated expectations and the conditional independence of $w$ and $\yvec$, and (b) follows from the definition that $\E[w | \Xvec, \zvec] = e^*(\Xvec, \zvec)$.
Therefore, 
$$
\E[\Psivec(e^*, \gvec^*; \betavec^*)]\,=\,\E\left[\Xvec^{\top}\nabla b(\Xvec\betavec^*)-\tauvec(\Xi; e^*, \gvec^*)\right]\,=\,\E\left[\Xvec^{\top}\nabla b(\Xvec\betavec^*)-\Xvec^{\top}\yvec\right]\,=\,0\,.
$$
{Here the last equality is the population first-order condition \eqref{eq:foc} defining $\betavec^*$.} This proves the first part.
For the second, we notice that 
\begin{align*}
    \sqrt{n}\left(\Pe_n \tauvec(\hat{e}, \hat{\gvec})-
\Pe_n \tauvec(e^*, \gvec^*)\right)\,=\,\underbrace{\sqrt{n}\Pe_n\left[\left(1 - \frac{w}{e^*}\right)\Xvec^{\top}\Delta_{\gvec}\right] }_{\text{(I)}}+
\underbrace{\sqrt{n}\Pe_n \left[\frac{w\,\Delta_e}{\hat{e}e^*}\Xvec^{\top}(\gvec^* - \yvec)\right]}_{\text{(II)}}+\underbrace{
\sqrt{n}\Pe_n \left[\frac{w\,\Delta_e}{\hat{e}e^*}\Xvec^{\top}\Delta_{\gvec}\right]}_{\text{(III)}}.
\end{align*}
We treat each term separately. 
For (I), we note that 
\begin{align}
\label{eqn:I-expectation}
\E\left[
\left(1 - \frac{w}{e^*}\right)\Xvec^{\top}\Delta_{\gvec}\,\bigg|\, I^c
\right]\,=\,
\E\left[
\left(1 - \frac{\E[w\mid\Xvec, \zvec, I^c]}{e^*}\right)\Xvec^{\top}\Delta_{\gvec}\,\bigg|\, I^c
\right]\,=\,0
\end{align}
so (I) has zero mean. 
Therefore, by Markov inequality 
\begin{align}
\label{eqn:I}
    \Pe\left(
\lVert \text{(I)} \rVert \geq \epsilon \big| I^c
    \right)\,\leq\,&~ \frac{\E[\lVert \text{(I)} \rVert^2 | I^c]}{\epsilon^2}=\frac{1}{\epsilon^2}\sum_{j=1}^d\text{Var}\left[
    \sqrt{n}\Pe_n\left[\left(1 - \frac{w}{e^*}\right)\xvec_{(j)}^{\top}\Delta_{\gvec}\right]
    \bigg| I^c\right]
=\frac{1}{\epsilon^2}\sum_{j=1}^d\text{Var}\left[
    \left(1 - \frac{w}{e^*}\right)\xvec_{(j)}^{\top}\Delta_{\gvec}
    \bigg| I^c\right]\notag\\
\,=\,&~\frac{1}{\epsilon^2}\,\E\left[
    \left\lVert\left(1 - \frac{w}{e^*}\right)\Xvec^{\top}\Delta_{\gvec}\right\rVert^2
    \,\bigg|\, I^c\right].
\end{align}
From here, we note that 
\begin{align*}
    \left \lVert \left(1 - \frac{w}{e^*}\right)\Xvec^{\top}\Delta_{\gvec}\right \rVert_{Q,2}^2\,\leq\, \frac{4}{\kappa^2} \left \lVert \Xvec^{\top}\Delta_{\gvec}\right \rVert_{Q,2}^2
\,\leq\,
    \frac{4C^2}{\kappa^2} \left \lVert \Delta_{\gvec}\right \rVert_{Q,2}^2\,\stackrel{\text{(c)}}{\rightarrow}\,0,
\end{align*}
which implies the RHS of \eqref{eqn:I} converges to zero in probability by the Markov's inequality. 
Here (c) uses only the $L_2$-consistency $\lVert \hat{\gvec} - \gvec^* \rVert_{Q,2} \to 0$ implied by Assumption \ref{ass:ml_rate}; no convergence rate is needed for term (I).
Since $\Pe\left(
\lVert \text{(I)} \rVert \geq \epsilon \,|\, I^c
    \right)$ is bounded therefore uniformly integrable, we can take expectation and conclude that (I) is $o_P(1)$.
    An argument similar to \eqref{eqn:I-expectation} shows that (II) has zero mean. An argument similar to \eqref{eqn:I} shows that 
    \begin{align}
       & \Pe\left(
\lVert \text{(II)} \rVert\,\geq\,\epsilon \,\big|\, I^c
    \right)\,\leq\, \frac{1}{\epsilon^2}\E\left[\left\lVert\frac{w\,\Delta_e}{\hat{e}e^*}\Xvec^{\top}(\gvec^* - \yvec)\right\rVert^2 \,\bigg|\,I^c\right]\notag
\,\stackrel{\text{(d)}}{=}\,\frac{2}{\kappa^2\epsilon^2}\E\left[\left\lVert w\,\Delta_e\,\Xvec^{\top}(\gvec^* - \yvec)\right\rVert^2 \,\bigg|\,I^c\right] +o_P(1)\notag\\
\,=\,&\frac{2}{\kappa^2\epsilon^2}\E\left[|\Delta_e|^2
    \sum_{j=1}^d
    \xvec_{(j)}^{\top}\text{Cov}(\yvec\mid\Xvec, \zvec, I^c)\xvec_{(j)}
    \,\bigg|\,I^c\right] +o_P(1)
\,\stackrel{\text{(e)}}{\leq}\,\frac{2\tilde{\sigma}^2}{\kappa^2\epsilon^2}\E\left[|\Delta_e|^2
    \sum_{j=1}^d
    \xvec_{(j)}^{\top}\xvec_{(j)}
    \,\bigg|\,I^c\right] +o_P(1)
    \notag\\
\,=\,&\frac{2\tilde{\sigma}^2}{\kappa^2\epsilon^2}\E\left[|\Delta_e|^2
    \lVert \Xvec \rVert_{F}^2
    \,\bigg|\,I^c\right] +o_P(1)
\,=\,\frac{2C^2\tilde{\sigma}^2}{\kappa^2\epsilon^2}\E\left[|\Delta_e|^2
    \,\bigg|\,I^c\right] +o_P(1),
    \end{align}
    where (d) follows because for all $\Xvec$ and $\zvec$, (1) by assumption $e^*(\Xvec, \zvec) \geq \kappa$ and (2) Assumption \ref{ass:ml_rate} states that $\sup_{\Xvec, \zvec}| \Delta_e | \to_P 0$, so with probability approaching one it holds that $\hat{e}(\Xvec, \zvec) \geq \kappa/2$. 
    (e) holds because the distribution of $\yvec$ is independent of $I^c$ so 
    $\text{Cov}(\yvec|\Xvec, \zvec, I^c) = \text{Cov}(\yvec|\Xvec, \zvec)$ and $\lVert \text{Cov}(\yvec|\Xvec, \zvec) \rVert \leq \tilde{\sigma}^2$.
    By an argument similar to the above discussion regarding (I), we have $\E\left[|\Delta_e|^2
    \,\big|\,I^c\right] = o_P(1)$.
    Therefore, (II) is also $o_P(1)$.
    For the third term, we have 
    \begin{align*}
        & \Pe\left(
\lVert \text{(III)} \rVert_{\infty}\,\geq\,\epsilon \,\big|\, I^c
    \right)\,\leq\,\sum_{j=1}^d \Pe\left(\sqrt{n}\left |\Pe_n \frac{w\,\Delta_e}{\hat{e}e^*}\xvec_{(j)}^{\top}\Delta_{\gvec}\right|\,\geq\,\epsilon\,\bigg|\,I^c\right)
\,\leq\,\frac{\sqrt{n}}{\epsilon}\sum_{j=1}^d \E \left[\left|\frac{w\,\Delta_e}{\hat{e}e^*}\xvec_{(j)}^{\top}\Delta_{\gvec}\right| \,\bigg|\,I^c\right]\\
\,\stackrel{\text{(f)}}{\leq}\,&\frac{2\sqrt{n}}{\kappa^2\epsilon}\sum_{j=1}^d \E \left[\left|\Delta_e\,\xvec_{(j)}^{\top}\Delta_{\gvec}\right| \,\big|\,I^c\right] +o_P(1)
\,\stackrel{\text{(g)}}{\leq}\,\frac{2\sqrt{n}}{\kappa^2\epsilon}\sum_{j=1}^d \E \left[\left|\Delta_e\right|
    \left\lVert\xvec_{(j)}\right \rVert
    \left\lVert\Delta_{\gvec}\right \rVert
    \,\big|\,I^c\right] +o_P(1)\\
\,\leq\,&\frac{2\sqrt{n}C}{\kappa^2\epsilon}\sum_{j=1}^d \E \left[\left|\Delta_e\right|
    \left\lVert\Delta_{\gvec}\right \rVert
    \,\big|\,I^c\right] +o_P(1),
    \end{align*}
    where (f) follows from an argument similar to above and (g) follows from the Cauchy–Schwarz inequality. 
    Finally, we have 
    \begin{align*}
        \E \left[\left|\Delta_e\right|
    \left\lVert\Delta_{\gvec}\right \rVert
    \right]\,\leq\,\left\lVert \Delta_e\right\rVert_{Q,2}
    \left\lVert \Delta_{\gvec}\right\rVert_{Q,2}\,{\color{black}\leq\,\alpha(n)^2/n^{r_1 + r_2}}\,=\,o(n^{-1/2}).
    \end{align*}
    Therefore, $\Pe\left(
\lVert \text{(III)} \rVert_{\infty} \geq \epsilon \,|\, I^c
    \right)$ is $o_P(1)$ and \text{(III)} is $o_P(1)$ too. We conclude the proof.

\subsection{Proof of Lemma \ref{lem: empirical-scores}}

For any functions $e(\cdot)$ and $\gvec(\cdot)$, data point $\Xi$, and $\betavec$, 
let us define $$\bar{\ell}(\Xi; e, \gvec; \betavec)\,:=\,b(\Xvec \betavec) - \left[\gvec(\Xvec,\zvec) - \frac{w}{e(\Xvec,\zvec)}(\gvec(\Xvec,\zvec) - \yvec) \right]^{\top} \Xvec \betavec,$$
so $\nabla_{\betavec} \bar{\ell}(\Xi; e, \gvec; \betavec) = \Psivec(\Xi; e, \gvec; \betavec)$ and $\nabla^2_{\betavec} \bar{\ell}(\Xi; e, \gvec; \betavec) = \Xvec^{\top} {\color{black}\nabla^2} b(\Xvec\betavec) \Xvec \succeq 0$.
Therefore $\bar{\ell}(\cdot)$ is convex.
In particular, $\E[\bar{\ell}(\Xi, e^*, \gvec^*, \betavec)]$ is convex in $\betavec$.
We first notice that 
$$ \nabla_{\betavec} \E[\bar{\ell}(e^*, \gvec^*, \betavec^*)]\,=\,
\nabla_{\betavec} \E[{\ell}(\rev{\Xvec, \yvec; \betavec^*})]
+ \E\left[\Xvec^{\top}(\yvec -\rev{\gvec^*}(\Xvec,\zvec))
\left(
1- \frac{w}{\rev{e^*}(\Xvec,\zvec)}
\right)
\right]
\,=\,0.$$
Also, for any fixed ${\betavec} \in \mB$ it holds that 
$\nabla^2_{\betavec}\E[\bar{\ell}(e^*, \gvec^*, \betavec)] = \E[\Xvec^{\top} {\color{black}\nabla^2} b(\Xvec\betavec) \Xvec] \succ 0$. 
The last inequality follows because by assumption for any $\Xvec \in \Breve{\mX}$, $\lambda_{\min}\big({\color{black}\nabla^2} b(\Xvec\betavec)\big) >0$. 
For the purpose of contradiction, if  $\inf_{\Xvec \in \Breve{\mX}} \lambda_{\min}\big({\color{black}\nabla^2} b(\Xvec\betavec)\big) = 0$, there is a converging sequence $\{\Xvec_m\}_{m=1}^\infty$ with limit $\widetilde{\Xvec} \in \Breve{\mX}$ such that $\lambda_{\min}\big({\color{black}\nabla^2} b\big(\widetilde{\Xvec}\betavec\big)\big) = 0$ by continuity of ${\color{black}\nabla^2} b(\cdot)$. 
This is a contradiction. 
Therefore, $\nabla^2_{\betavec}\E[\bar{\ell}(e^*, \gvec^*, \betavec)]   \succ  \inf_{\Xvec \in \Breve{\mX}} \lambda_{\min}\big({\color{black}\nabla^2} b(\Xvec\betavec)\big)\cdot \E[\Xvec^{\top}  \Xvec] \succ 0$.
In other words, $\E[\bar{\ell}(e^*, \gvec^*, \betavec)]$ is strictly convex. 
Therefore, $\betavec^*$ is the unique minimum.
By compactness of $\Breve{\mX}$ and continuity of $b(\cdot)$, we have $\Pe_n\bar{\ell}(e^*, \gvec^*; \betavec) \to_P \E[\bar{\ell}(e^*, \gvec^*; \betavec)]$ as an application of the weak law of large numbers. 
By Theorem 2.7 of \cite{newey1994large}, 
it holds that there is a random sequence $\{\check{\betavec}_n\}^{\infty}_{n=1}$ that solves $\min_{\beta \in \mB} \Pe_n\bar{\ell}(e^*, \gvec^*; \betavec)$ and converges to $\betavec^*$ with probability one. 
This implies that with probability one there is a sequence $\{\check{\betavec}_n\}^{\infty}_{n=1}$ such that 
$\Pe_n\Psivec(e^*, \gvec^*; \check{\betavec}_n) =\bm{0}$.
By Lemma \ref{lem: orthogonality}, it holds that 
\begin{align}
    \label{eqn:transform}\sqrt{n}\Pe_n\Psivec(\hat{e}, \hat{\gvec}; \check{\betavec}_n)\,=\,\sqrt{n}\Pe_n\Psivec(e^*, \gvec^*; \check{\betavec}_n) + \sqrt{n}\left(\Pe_n\tauvec(e^*, \gvec^*) - \Pe_n\tauvec(\hat{e}, \hat{\gvec})\right)\,=\,o_P(1),
\end{align}
or $\inf_{\beta \in \mB}\big\lVert
\Pe_n \Psivec\big(\hat{e}, \hat{\gvec}; \betavec\big)
\big \rVert = o_P(1/\sqrt{n})$.
This implies $\Pe_n \Psivec\big(\hat{e}, \hat{\gvec}; \hat{\betavec}\big) = o_P(1/\sqrt{n})$.
Using the argument in \eqref{eqn:transform}
 again, we have $\Pe_n \Psivec\big({e}^*, {\gvec}^*; \hat{\betavec}\big)= o_P(1/\sqrt{n})$. 
This concludes the proof.

\subsection{Proof of Lemma \ref{lem: consistency}}

The notations and definitions are the same as in the proof of the Lemma \ref{lem: empirical-scores}. 
Fix arbitrary $\epsilon_1 < \epsilon_2$ such that $\overline{B(\betavec^*, \epsilon_2)} \subset \mB$.
On $\overline{B(\betavec^*, \epsilon_2)}\setminus B(\betavec^*, \epsilon_1)$, by a Taylor expansion, it must be that 
\begin{align*}
    \E[\bar{\ell}(e^*, \gvec^*; \betavec)]\,=\,& 
\E[\bar{\ell}(e^*, \gvec^*; \betavec^*)] + \frac{1}{2}(\betavec - \betavec^*)^{\top}\E[\Xvec^{\top} {\color{black}\nabla^2} b(\Xvec\tilde{\betavec}) \Xvec](\betavec - \betavec^*)
\,\stackrel{\text{(a)}}{\geq}\,\E[\bar{\ell}(e^*, \gvec^*; \betavec^*)] + \frac{c}{2}\lVert\betavec - \betavec^*\rVert^2
\end{align*}
which can be further lower bounded by $\E[\bar{\ell}(e^*, \gvec^*; \betavec^*)] + \frac{c\epsilon_1^2}{2}$,
for some $\tilde{\betavec} \in [\betavec^*, \betavec] \subset \overline{B(\betavec^*, \epsilon_2)}$ that depends on $\betavec$ and a universal $c > 0$.
We prove the inequality (a) as follows. It suffices to show that $\lambda_{\min}\left({\color{black}\nabla^2} b(\Xvec{\betavec})\right)$ is uniformly lower bounded by some positive number, when $\betavec$ ranges over $\overline{B(\betavec^*, \epsilon_2)}$ and $\Xvec$ ranges over ${\mX}$ . 
Suppose otherwise, there is then $\left\{(\Xvec_m, \betavec_m)\right\}_{m=1}^{\infty}$ such that 
$\lambda_{\min}\left({\color{black}\nabla^2} b(\Xvec_m{\betavec}_m)\right) \to 0$. 
By taking subsequence is necessary, we assume that $\Xvec_m \to \bar{\Xvec} \in \Breve{\mX}$ and ${\betavec}_m \to \bar{\betavec} \in \overline{B(\betavec^*, \epsilon_2)} \subset \mB$. 
By assumption $\bar{\Xvec}\bar{\betavec} \in \Theta$ and continuity of $\nabla^2 b(\cdot)$ implies $\lambda_{\min}\left({\color{black}\nabla^2} b\big(\bar{\Xvec}\bar{\betavec})\right) = 0$, which contradicts our assumption. 

By the convexity lemma \cite{pollard1991asymptotics}, convexity of $\bar{
\ell}(\cdot)$ and the compactness of $\overline{B(\betavec^*, \epsilon_2)}$, it holds that $\Pe_n \bar{\ell}\big({e}^*, {\gvec}^*; {\betavec}\big)$ converges to
$\E\big[\bar{\ell}\big({e}^*, {\gvec}^*; {\betavec}\big)\big]$ in probability uniformly on $\overline{B(\betavec^*, \epsilon_2)}$.
Consequently, with probability approaching one, it must be that $\forall \betavec \in \overline{B(\betavec^*, \epsilon_2)}\setminus B(\betavec^*, \epsilon_1)$, 
\begin{align*}
    \Pe_n\bar{\ell}(e^*, \gvec^*; \betavec) - \frac{c\epsilon^2_1}{4}
\,\geq\,&~ \Pe_n\bar{\ell}(e^*, \gvec^*; \betavec^*)\,\stackrel{\text{(b)}}{\geq}\,
\Pe_n\bar{\ell}(e^*, \gvec^*; \betavec) + 
(\betavec - \betavec^*)^{\top}\Pe_n{\Psivec}(e^*, \gvec^*; \betavec)\\
\,&=\,
\Pe_n\bar{\ell}(e^*, \gvec^*; \betavec) + 
\lVert \betavec - \betavec^*\rVert\cdot 
\frac{(\betavec^* - \betavec)^{\top}}{\lVert \betavec - \betavec^*\rVert}\Pe_n{\Psivec}(e^*, \gvec^*; \betavec),
\end{align*}
where (b) follows from subgradient inequality. 
This implies that 
\begin{align}
\label{eqn:gradient-bound}
    \frac{(\betavec - \betavec^*)^{\top}}{\lVert \betavec - \betavec^*\rVert}\Pe_n{\Psivec}(e^*, \gvec^*; \betavec)\,\geq\,\frac{c\epsilon^2_1}{4\epsilon_2}. 
\end{align}
Next, consider any $\betavec \in \mB \setminus \overline{B(\betavec^*, \epsilon_2)}$, there is $\tilde{\betavec} = \lambda\betavec+(1- \lambda)\betavec^*\in \overline{B(\betavec^*, \epsilon_2)}\setminus B(\betavec^*, \epsilon_1)$. 
By the integral form of the intermediate value theorem,
\begin{align*}
    &\frac{(\betavec - \betavec^*)^{\top}}{\lVert \betavec - \betavec^*\rVert}\big(\Pe_n{\Psivec}(e^*, \gvec^*; \betavec) -\Pe_n{\Psivec}(e^*, \gvec^*; \betavec^*
    ) \big)\\ 
\,&=\,\frac{(\betavec - \betavec^*)^{\top}}{\lVert \betavec - \betavec^*\rVert} \int^1_0 \Pe_n \Xvec^{\top} {\color{black}\nabla^2} b\big(\Xvec(\betavec^* + t(\betavec-\betavec^*))\big) \Xvec dt (\betavec - \betavec^*)\\
\,&=\,\frac{(\tilde{\betavec} - \betavec^*)^{\top}}{\lambda\lVert \tilde{\betavec} - \betavec^*\rVert} \int^1_0 \Pe_n \Xvec^{\top} {\color{black}\nabla^2} b\left(\Xvec\left(\betavec^* + \frac{t}{\lambda}(\tilde{\betavec}-\betavec^*)\right)\right)  \Xvec dt (\tilde{\betavec} - \betavec^*)\\
\,&\stackrel{\text{(c)}}{=}\,\frac{(\tilde{\betavec} - \betavec^*)^{\top}}{\lVert \tilde{\betavec} - \betavec^*\rVert} \int^{\frac{1}{\lambda}}_0 \Pe_n \Xvec^{\top} {\color{black}\nabla^2} b\big(\Xvec\left(\betavec^* + {t}(\tilde{\betavec}-\betavec^*)\right)\big)  \Xvec dt (\tilde{\betavec} - \betavec^*)\\
\,&\stackrel{\text{(d)}}{\geq}\,\frac{(\tilde{\betavec} - \betavec^*)^{\top}}{\lVert \tilde{\betavec} - \betavec^*\rVert} \int^{1}_0 \Pe_n \Xvec^{\top} {\color{black}\nabla^2} b\big(\Xvec\left(\betavec^* + {t}(\tilde{\betavec}-\betavec^*)\right)\big)  \Xvec dt (\tilde{\betavec} - \betavec^*)\cdot\\
\,&=\,\frac{(\tilde{\betavec} - \betavec^*)^{\top}}{\lVert \tilde{\betavec} - \betavec^*\rVert}\big(\Pe_n{\Psivec}(e^*, \gvec^*; \tilde{\betavec}) -\Pe_n{\Psivec}(e^*, \gvec^*; \betavec^*
    ) \big)\,\stackrel{\text{(e)}}{\geq}\,\frac{c\epsilon^2_1}{4\epsilon_2} - \frac{(\tilde{\betavec} - \betavec^*)^{\top}}{\lVert \tilde{\betavec} - \betavec^*\rVert}\Pe_n{\Psivec}(e^*, \gvec^*; \betavec^*
    ) ,
\end{align*}
where (c) follows from a change of variable in the integral, (d) follows because the ${\color{black}\nabla^2} b(\cdot)$ is positive definite, and (e) uses \eqref{eqn:gradient-bound}.
Combining this again with \eqref{eqn:gradient-bound}, we conclude that with probability converging to one, 
$$
\left \lVert \Pe_n{\Psivec}(e^*, \gvec^*; \betavec)\right \rVert\,\geq\,\frac{(\betavec - \betavec^*)^{\top}}{\lVert \betavec - \betavec^*\rVert}\Pe_n{\Psivec}(e^*, \gvec^*; \betavec)\,\geq\,\frac{c\epsilon^2_1}{4\epsilon_2}~~~~~\forall~ \betavec\,\in\,
\mB \setminus {B(\betavec^*, \epsilon_1)}.
$$
Finally, we note that by Lemma \ref{lem: empirical-scores}, $\big\lVert
\Pe_n \Psivec\big({e}^*, {\gvec}^*; \hat{\betavec}\big)
\big \rVert = o_P(1/\sqrt{n}).$
Therefore, it must be that $\hat{\betavec} \in B(\betavec^*, \epsilon_1)$ with probability converging to one.
We finish the proof.

\subsection{Proof of Corollary \ref{cor:variance_decomposition}}
\label{apdx: variance decomposition}

This appendix provides the proof of Corollary~\ref{cor:variance_decomposition}, which decomposes the variance gain of GAI into two interpretable components.

Recall that the primary-only estimator uses only labeled observations, while GAI incorporates auxiliary data through the conditional expectation $\E[\yvec \mid \Xvec, \zvec]$. To understand the variance gain, we decompose the GLM score $\nabla_{\betavec} \ell(\Xvec, \yvec; \betavec^*) = \Xvec^{\top}(\nabla b(\Xvec\betavec^*) - \yvec)$ by inserting intermediate quantities.

From the proof of Corollary~\ref{cor:dominance}, the asymptotic variance of the primary-only estimator minus that of the GAI estimator is proportional to:
\begin{equation}
\Sigmavec^{\text{\upshape P}} - \Sigmavec^{\text{\upshape GAI}}\,\propto\,(\frac{1}{\rho} - 1) \cdot \E\bigg[\Xvec^{\top}\Big(\nabla b(\Xvec\betavec^*) - \E[\yvec \mid \Xvec, \zvec]\Big)\Big(\nabla b(\Xvec\betavec^*) - \E[\yvec \mid \Xvec, \zvec]\Big)^{\top}\Xvec\bigg],
\label{eq:var_diff}
\end{equation}
where \rev{$\rho \in (0,1)$ is the labeling probability of Assumption~\ref{ass:random_labeling}, so that the labeled fraction $n_P/n$ converges to $\rho$}. This expression is always positive semidefinite (PSD), confirming that GAI never increases variance relative to primary-only estimation.

The variance reduction in \eqref{eq:var_diff} is strictly positive definite (i.e., GAI strictly dominates) whenever the inner expectation is positive definite. We decompose the gap $\nabla b(\Xvec\betavec^*) - \E[\yvec \mid \Xvec, \zvec]$ into interpretable components:
\begin{align}
\nabla b(\Xvec\betavec^*) - \E[\yvec \mid \Xvec, \zvec]
\,&=\,\underbrace{\Big(\nabla b(\Xvec\betavec^*) - \E[\yvec \mid \Xvec]\Big)}_{\text{(II) Model misspecification}}
+ \underbrace{\Big(\E[\yvec \mid \Xvec] - \E[\yvec \mid \Xvec, \zvec]\Big)}_{\text{(III) Extra information in } \zvec}.
\label{eq:decomp}
\end{align}

Since $\E[\yvec \mid \Xvec] = \E\big[\E[\yvec \mid \Xvec, \zvec] \,\big|\, \Xvec\big]$ by the law of iterated expectations, term (III) has conditional mean zero given $\Xvec$. This implies that the cross-product between terms (II) and (III) vanishes:
\[
\E\Big[\text{(II)} \cdot \text{(III)}^{\top}\Big]\,=\,\E\Big[\text{(II)} \cdot \E\big[\text{(III)}^{\top} \mid \Xvec\big]\Big]\,=\,0.
\]
Therefore, the variance gain decomposes additively into two non-negative components as stated in Corollary~\ref{cor:variance_decomposition}.

\subsection{An Example of Failure of Dominance}
\label{apdx: failure of dominance}

Generally, if $e(\Xvec, \zvec)$ is carefully constructed so that it is ``advantageous'' to ignore the data points with missing label, one can show the following negative result.

\begin{example}
{ \normalfont (\textbf{Failure of  Dominance})}
    Consider a setting with canonical GLMs such that the GLM density is correctly specified.
    Assume that $k =1$ and 
    $b(\theta) = \frac{1}{2}\theta^2$.
    In this case, we write $\Xvec = \xvec^{\top} \in \R^{1 \times d}$.
    We further assume that $\xvec$ is generated through a mixture distribution and there is $\tilde{w}$ such that $\tilde{w} = 1$ with probability $p$ and zero otherwise. 
    Also, $\text{\upshape E}[\xvec\xvec^{\top} \,|\, \tilde{w}] = \Ivec$ for all $\tilde{w} \in \{0,1\}$.
    Conditional on $\tilde{w}$, the supports of $\xvec$ are disjoint and we denote them by $\mX_{w}$.
    We assume that $\text{\upshape E}[y\,|\, \xvec] = \xvec^{\top}\betavec^*$, and $\text{\upshape Var}(y\,|\, \xvec) = \sigma^2_{w}$ whenever $\xvec \in \mX_w$. 
    Set $z = y - \xvec^{\top}\betavec$ so $\gvec^*(\xvec, z) = y$.
    $e(\xvec,z) = 1$ if $\xvec \in \mX_{1}$ and $e(\xvec,z) = \kappa$ if $\xvec \in \mX_{0}$. Conditional on $\tilde{w}$, $w$ is independent of $(y, z)$. In the setup, the GAI estimator leads to the same asymptotic variance as the case where $y$ is fully observed, i.e., $(p\sigma^2_1 + (1-p)\sigma^2_0)\Ivec$.
    It is easy to see that $\hat{\betavec}^{\text{\upshape{P}}}$ obtained under the score function $w_i\nabla_{\betavec}\ell(\xvec_i, \yvec_i; \betavec)$ is asymptotically normal with covariance given by
\begin{align*}
    &\text{\upshape AVar}\big(\hat{\betavec}^{\text{\upshape{P}}}\big)\,=\,
    \frac{p\sigma^2_1 + (1-p)\sigma^2_0\kappa}{[p+(1-p)\kappa]^2}\Ivec 
\,\prec\,
    \text{\upshape AVar}\big(\hat{\betavec}^{\text{\upshape{GAI}}}\big)\,=\,(p\sigma^2_1 + (1-p)\sigma^2_0)\Ivec, 
\end{align*}
    which holds when $\sigma_1, \kappa \downarrow 0$.
\end{example}

\section{Supporting Results for Section \ref{sec:covshift}}
\label{apdx: supporting covshift}

This appendix collects the proofs deferred from Section~\ref{sec:covshift} and from \rev{Section}~\ref{apdx: covariate shift theory}: the proof of the combined known-density-ratio result (Theorem~\ref{thm:normality_shift}) and the proof of the doubly robust result stated in the main body (Theorem~\ref{thm:normality_dr_body}).

\subsection{Proof of Theorem \ref{thm:normality_shift}}
\label{apdx: proof normality shift}
\label{apdx: proof dominance shift}

Theorem~\ref{thm:normality_shift} combines two claims: the asymptotic normality of the weighted GAI estimator $\hat{\betavec}_{\text{\upshape w}}$ and its dominance over the weighted primary-only estimator $\hat{\betavec}^{\text{\upshape P},\text{\upshape w}}$. We prove the normality claim first and then the dominance claim.

\medskip
\noindent\emph{Asymptotic normality of $\hat{\betavec}_{\text{\upshape w}}$.} The proof follows exactly the structure of the proof of Theorem~\ref{thm:normality} in Section~\ref{apdx: proof of normality}: we first establish the counterparts of Lemmas~\ref{lem: orthogonality}--\ref{lem: consistency} for the weighted score \eqref{eq:score_shift}, and then complete the argument with the same Donsker/Taylor step. For each step we indicate explicitly whether the existing argument applies---in which case we cite it and state the substitution---and we write out in full only the steps that are genuinely new.

As in the proof of Theorem~\ref{thm:normality}, we assume without loss of generality that the dataset is split into $I$ and $I^c$, each of size $n$; the nuisance $\hat{\gvec}$ is obtained on $I^c$ and $\hat{\betavec}_{\text{\upshape w}}$ is estimated on $I$. Here neither $\rho$ nor $r(\Xvec)$ is estimated, so $\hat{\gvec}$ is the only nuisance trained on $I^c$. We use the same shorthand $\Pe_n$ and $\Gn$ as in Section~\ref{apdx: proof of normality}, where now the population expectation is taken under the source data-generating process (including the labeling indicator $w$) and is written $\E_{\text{\upshape S}}[\,\cdot\,]$, and the same constant $C$ such that $\lVert \Xvec \rVert, \lVert \Xvec \rVert_F \leq C$ for all $\Xvec \in \Breve{\mX}$; by Assumption~\ref{ass:covariate_shift}(ii), the same bound applies under the target population. Mirroring the definition of $\tauvec(\Xi; e, \gvec)$ at the start of Section~\ref{apdx: proof of normality}, we define the $\betavec$-free part of the weighted score,
\begin{align*}
    \tauvec_{\text{\upshape w}}(\Xi; \gvec)\,:=\,r(\Xvec)\, \Xvec^{\top}\left[\gvec(\Xvec,\zvec) - \frac{w}{\rho}\big(\gvec(\Xvec,\zvec) - \yvec\big)\right]\,=\,r(\Xvec)\, \tauvec(\Xi; \rho, \gvec),
\end{align*}
so that $\Psivec_{\text{\upshape w}}(\Xi; \gvec; \betavec) = r(\Xvec)\,\Xvec^{\top}\nabla b(\Xvec\betavec) - \tauvec_{\text{\upshape w}}(\Xi; \gvec)$ for all $\betavec \in \mB$.

The first lemma is the counterpart of Lemma~\ref{lem: orthogonality}.

\begin{lemma}[\normalfont \textbf{Score Validity and Neyman Orthogonality under Covariate Shift}]
\label{lem: orthogonality shift} It holds that
$$\E_{\text{\upshape S}}[\Psivec_{\text{\upshape w}}\big(\gvec^*; \betavec^*_{\text{\upshape T}}\big)]\,=\,0 ~~\text{and}~~
\Pe_n \tauvec_{\text{\upshape w}}(\hat{\gvec}) - \Pe_n \tauvec_{\text{\upshape w}}(\gvec^*)\,=\,o_P(n^{-1/2}).
$$
\end{lemma}

\textit{Proof of Lemma~\ref{lem: orthogonality shift}.} For the first claim, by the definition of $\tauvec_{\text{\upshape w}}$,
\begin{align*}
    \E_{\text{\upshape S}}[\tauvec_{\text{\upshape w}}(\gvec^*)]\,&=\,\E_{\text{\upshape S}}\left[r(\Xvec)\, \Xvec^{\top}\left[\gvec^*(\Xvec,\zvec) - \frac{w}{\rho}\big(\gvec^*(\Xvec,\zvec) - \yvec\big)\right]\right]\\
\,&\stackrel{\text{(a)}}{=}\,\E_{\text{\upshape S}}\left[r(\Xvec)\, \Xvec^{\top}\left[\gvec^*(\Xvec,\zvec) - \frac{\E[w \mid \Xvec,\zvec]}{\rho}\big(\gvec^*(\Xvec,\zvec) - \E[\yvec \mid \Xvec,\zvec]\big)\right]\right]\\
\,&\stackrel{\text{(b)}}{=}\,\E_{\text{\upshape S}}\left[r(\Xvec)\, \Xvec^{\top}\E[\yvec \mid \Xvec,\zvec]\right]\,=\,\E_{\text{\upshape S}}\left[r(\Xvec)\, \Xvec^{\top}\yvec\right]\,\stackrel{\text{(c)}}{=}\,\E_{\text{\upshape T}}\left[\Xvec^{\top}\yvec\right],
\end{align*}
where (a) follows from the law of iterated expectations and the independence of $w$ and $(\Xvec, \yvec, \zvec)$ under Assumption~\ref{ass:covariate_shift}(iii) (note that $r(\Xvec)$ is a known function of $\Xvec$ and passes through the conditional expectation), (b) follows from $\E[w \mid \Xvec, \zvec] = \rho$ and $\gvec^*(\Xvec,\zvec) = \E[\yvec \mid \Xvec, \zvec]$, and (c) is the change-of-measure identity \eqref{eq:change_of_measure}. Steps (a)--(b) are identical to steps (a)--(b) in the proof of Lemma~\ref{lem: orthogonality}, with $e^*(\Xvec, \zvec) \equiv \rho$ and the bounded known factor $r(\Xvec)$ carried through; step (c) is new. Similarly, since $\Xvec^{\top}\nabla b\big(\Xvec\betavec^*_{\text{\upshape T}}\big)$ is a function of $\Xvec$ alone, \eqref{eq:change_of_measure} yields
$$
\E_{\text{\upshape S}}\left[r(\Xvec)\, \Xvec^{\top}\nabla b\big(\Xvec\betavec^*_{\text{\upshape T}}\big)\right]\,=\,\E_{\text{\upshape T}}\left[\Xvec^{\top}\nabla b\big(\Xvec\betavec^*_{\text{\upshape T}}\big)\right].
$$
Combining the two displays,
$$
\E_{\text{\upshape S}}[\Psivec_{\text{\upshape w}}\big(\gvec^*; \betavec^*_{\text{\upshape T}}\big)]\,=\,\E_{\text{\upshape T}}\left[\Xvec^{\top}\left(\nabla b\big(\Xvec\betavec^*_{\text{\upshape T}}\big) - \yvec\right)\right]\,=\,0
$$
by the target first-order condition \eqref{eq:foc_shift}. The two change-of-measure steps---which convert the weighted source moments into target moments, so that the target first-order condition rather than the source one closes the argument---are the only new ingredients relative to the proof of Lemma~\ref{lem: orthogonality}.

For the second claim, since the labeling probability is known and not estimated, we have $\hat{e} \equiv e^* \equiv \rho$, and therefore
\begin{align*}
    \sqrt{n}\left(\Pe_n \tauvec_{\text{\upshape w}}(\hat{\gvec}) - \Pe_n \tauvec_{\text{\upshape w}}(\gvec^*)\right)\,=\,\underbrace{\sqrt{n}\,\Pe_n\left[r(\Xvec)\left(1 - \frac{w}{\rho}\right)\Xvec^{\top}\Delta_{\gvec}\right]}_{\text{(I}_{\text{\upshape w}}\text{)}}.
\end{align*}
This is the exact analog of term (I) in the proof of Lemma~\ref{lem: orthogonality}, weighted by $r(\Xvec)$; in the notation of that proof, the analogs of terms (II) and (III)---both of which carry the factor $\Delta_e$---are \emph{identically zero} here because $\hat{e} \equiv e^* \equiv \rho$. This observation, although trivial, is what removes the need for any rate or sup-norm condition on the nuisances: only (I$_{\text{\upshape w}}$) must be controlled, which requires only the $L_2$-consistency of $\hat{\gvec}$.

We control (I$_{\text{\upshape w}}$) by the same argument used for term (I) there. First, conditional on $I^c$, it has zero mean: exactly as in \eqref{eqn:I-expectation}, $\E[w \mid \Xvec, \zvec, I^c] = \rho$ by Assumption~\ref{ass:covariate_shift}(iii) and the independence of $I$ and $I^c$, so
$$
\E\left[r(\Xvec)\left(1 - \frac{w}{\rho}\right)\Xvec^{\top}\Delta_{\gvec}\,\bigg|\, I^c\right]\,=\,\E\left[r(\Xvec)\left(1 - \frac{\E[w \mid \Xvec, \zvec, I^c]}{\rho}\right)\Xvec^{\top}\Delta_{\gvec}\,\bigg|\, I^c\right]\,=\,0.
$$
Second, by the same conditional Markov/variance computation as in \eqref{eqn:I},
$$
\Pe\left(\lVert \text{(I}_{\text{\upshape w}}\text{)} \rVert\,\geq\,\epsilon \,\big|\, I^c\right)\,\leq\,\E\left[\left\lVert r(\Xvec)\left(1 - \frac{w}{\rho}\right)\Xvec^{\top}\Delta_{\gvec}\right\rVert^2 \,\bigg|\, I^c\right]\bigg/\epsilon^2.
$$
The only change relative to the bound following \eqref{eqn:I} is the constant: since $0 \leq r(\Xvec) \leq M$ by Assumption~\ref{ass:covariate_shift}(ii) and $|1 - w/\rho| \leq 1 + 1/\rho \leq 2/\rho$ for $\rho \in (0,1]$,
$$
\left\lVert r(\Xvec)\left(1 - \frac{w}{\rho}\right)\Xvec^{\top}\Delta_{\gvec}\right\rVert_{Q_{\text{\upshape S}},2}^2\,\leq\,\frac{4M^2}{\rho^2}\, C^2\, \left\lVert \Delta_{\gvec}\right\rVert_{Q_{\text{\upshape S}},2}^2 ~\rightarrow~ 0,
$$
where the constant $4M^2/\rho^2$ replaces $4/\kappa^2$ in the proof of Lemma~\ref{lem: orthogonality}, and the convergence uses only the hypothesis $\lVert \Delta_{\gvec} \rVert_{Q_{\text{\upshape S}},2} = o(1)$ of Theorem~\ref{thm:normality_shift} in place of Assumption~\ref{ass:ml_rate}. The conclusion that (I$_{\text{\upshape w}}$) is $o_P(1)$ then follows by the same uniform-integrability argument as in the proof of Lemma~\ref{lem: orthogonality}. This concludes the proof of the lemma. \hfill $\blacksquare$

The next lemma is the counterpart of Lemma~\ref{lem: empirical-scores}.

\begin{lemma}[\normalfont \textbf{Rates of Empirical Scores under Covariate Shift}]
\label{lem: empirical-scores shift} It holds that
$$
\big\lVert \Pe_n \Psivec_{\text{\upshape w}}\big(\gvec^*; \hat{\betavec}_{\text{\upshape w}}\big) \big\rVert,\, \inf_{\betavec \in \mB}\big\lVert \Pe_n \Psivec_{\text{\upshape w}}\big(\hat{\gvec}; \betavec\big) \big\rVert\,=\,o_P(1/\sqrt{n}).
$$
\end{lemma}

\textit{Proof of Lemma~\ref{lem: empirical-scores shift}.} Mirroring $\bar{\ell}$ in the proof of Lemma~\ref{lem: empirical-scores}, define
$$
\bar{\ell}_{\text{\upshape w}}(\Xi; \gvec; \betavec)\,:=\,r(\Xvec)\, b(\Xvec\betavec) - \tauvec_{\text{\upshape w}}(\Xi; \gvec)^{\top}\betavec,
$$
written here through $\tauvec_{\text{\upshape w}}$ so that the gradient identity is immediate: $\nabla_{\betavec}\bar{\ell}_{\text{\upshape w}}(\Xi; \gvec; \betavec) = \Psivec_{\text{\upshape w}}(\Xi; \gvec; \betavec)$ and $\nabla^2_{\betavec}\bar{\ell}_{\text{\upshape w}}(\Xi; \gvec; \betavec) = r(\Xvec)\, \Xvec^{\top}\nabla^2 b(\Xvec\betavec)\Xvec \succeq 0$. The only new observation needed for convexity is that the known weight satisfies $r(\Xvec) \geq 0$, so the weighting preserves the positive semidefiniteness of the Hessian; hence $\bar{\ell}_{\text{\upshape w}}$ is convex in $\betavec$, and so is $\E_{\text{\upshape S}}[\bar{\ell}_{\text{\upshape w}}(\gvec^*; \betavec)]$. By Lemma~\ref{lem: orthogonality shift},
$$
\nabla_{\betavec} \E_{\text{\upshape S}}[\bar{\ell}_{\text{\upshape w}}\big(\gvec^*; \betavec^*_{\text{\upshape T}}\big)]\,=\,\E_{\text{\upshape S}}[\Psivec_{\text{\upshape w}}\big(\gvec^*; \betavec^*_{\text{\upshape T}}\big)]\,=\,0.
$$
The strict convexity of the population objective is the second genuinely new step: for any fixed $\betavec \in \mB$, since $\Xvec^{\top}\nabla^2 b(\Xvec\betavec)\Xvec$ is a function of $\Xvec$ alone, the change-of-measure identity \eqref{eq:change_of_measure} gives
$$
\nabla^2_{\betavec} \E_{\text{\upshape S}}[\bar{\ell}_{\text{\upshape w}}(\gvec^*; \betavec)]\,=\,\E_{\text{\upshape S}}\left[r(\Xvec)\, \Xvec^{\top}\nabla^2 b(\Xvec\betavec)\Xvec\right]\,=\,\E_{\text{\upshape T}}\left[\Xvec^{\top}\nabla^2 b(\Xvec\betavec)\Xvec\right]\,\succeq\,\inf_{\Xvec \in \Breve{\mX}}\lambda_{\min}\big(\nabla^2 b(\Xvec\betavec)\big)\cdot \E_{\text{\upshape T}}\left[\Xvec^{\top}\Xvec\right]\,\succ\,0,
$$
where $\inf_{\Xvec \in \Breve{\mX}}\lambda_{\min}\big(\nabla^2 b(\Xvec\betavec)\big) > 0$ by exactly the compactness/contradiction argument in the proof of Lemma~\ref{lem: empirical-scores}, which we cite rather than repeat, and $\E_{\text{\upshape T}}[\Xvec^{\top}\Xvec] \succ 0$ by Assumption~\ref{ass:covariate_shift}(iv). The positive-definiteness argument thus runs through the \emph{target} second-moment matrix $\E_{\text{\upshape T}}[\Xvec^{\top}\Xvec]$ rather than $\E_{\text{\upshape S}}[\Xvec^{\top}\Xvec]$. Consequently, $\E_{\text{\upshape S}}[\bar{\ell}_{\text{\upshape w}}(\gvec^*; \betavec)]$ is strictly convex and $\betavec^*_{\text{\upshape T}}$ is its unique minimizer.

The remainder of the proof is unchanged from the proof of Lemma~\ref{lem: empirical-scores}. Since $0 \leq r(\Xvec) \leq M$, the compactness of $\Breve{\mX}$ and the continuity of $b(\cdot)$ imply, by the weak law of large numbers, $\Pe_n \bar{\ell}_{\text{\upshape w}}(\gvec^*; \betavec) \to_P \E_{\text{\upshape S}}[\bar{\ell}_{\text{\upshape w}}(\gvec^*; \betavec)]$; the bounded weight does not affect the moment requirements. By Theorem 2.7 of \cite{newey1994large}, applied exactly as in the proof of Lemma~\ref{lem: empirical-scores}, there is a random sequence $\{\check{\betavec}_n\}_{n=1}^{\infty}$ that solves $\min_{\betavec \in \mB}\Pe_n \bar{\ell}_{\text{\upshape w}}(\gvec^*; \betavec)$ and converges to $\betavec^*_{\text{\upshape T}}$ with probability one; in particular, $\Pe_n \Psivec_{\text{\upshape w}}(\gvec^*; \check{\betavec}_n) = \bm{0}$. Noting that $\Psivec_{\text{\upshape w}}(\hat{\gvec}; \betavec) - \Psivec_{\text{\upshape w}}(\gvec^*; \betavec) = \tauvec_{\text{\upshape w}}(\gvec^*) - \tauvec_{\text{\upshape w}}(\hat{\gvec})$ does not depend on $\betavec$, Lemma~\ref{lem: orthogonality shift} yields the exact counterpart of \eqref{eqn:transform}:
\begin{align}
\label{eqn:transform-shift}
    \sqrt{n}\Pe_n\Psivec_{\text{\upshape w}}(\hat{\gvec}; \check{\betavec}_n)\,=\,\sqrt{n}\Pe_n\Psivec_{\text{\upshape w}}(\gvec^*; \check{\betavec}_n) + \sqrt{n}\left(\Pe_n\tauvec_{\text{\upshape w}}(\gvec^*) - \Pe_n\tauvec_{\text{\upshape w}}(\hat{\gvec})\right)\,=\,o_P(1).
\end{align}
Hence $\inf_{\betavec \in \mB}\big\lVert \Pe_n\Psivec_{\text{\upshape w}}(\hat{\gvec}; \betavec)\big\rVert = o_P(1/\sqrt{n})$, which implies $\Pe_n\Psivec_{\text{\upshape w}}\big(\hat{\gvec}; \hat{\betavec}_{\text{\upshape w}}\big) = o_P(1/\sqrt{n})$ by the definition of $\hat{\betavec}_{\text{\upshape w}}$. Using the argument in \eqref{eqn:transform-shift} again, $\Pe_n\Psivec_{\text{\upshape w}}\big(\gvec^*; \hat{\betavec}_{\text{\upshape w}}\big) = o_P(1/\sqrt{n})$. This concludes the proof of the lemma. \hfill $\blacksquare$

The final lemma is the counterpart of Lemma~\ref{lem: consistency}.

\begin{lemma}[\normalfont \textbf{Consistency under Covariate Shift}]
\label{lem: consistency shift} $\hat{\betavec}_{\text{\upshape w}} \to_P \betavec^*_{\text{\upshape T}}$.
\end{lemma}

\textit{Proof of Lemma~\ref{lem: consistency shift}.} The proof of Lemma~\ref{lem: consistency} carries over with $\bar{\ell}$, $\Psivec$, and $\betavec^*$ replaced by $\bar{\ell}_{\text{\upshape w}}$, $\Psivec_{\text{\upshape w}}$, and $\betavec^*_{\text{\upshape T}}$ throughout; we verify the two steps where the weighting enters and cite the rest. Fix arbitrary $\epsilon_1 < \epsilon_2$ such that $\overline{B(\betavec^*_{\text{\upshape T}}, \epsilon_2)} \subset \mB$. First, the uniform convergence step: by \cite{pollard1991asymptotics}, the convexity of $\bar{\ell}_{\text{\upshape w}}(\gvec^*; \cdot)$ established in the proof of Lemma~\ref{lem: empirical-scores shift} and the compactness of $\overline{B(\betavec^*_{\text{\upshape T}}, \epsilon_2)}$ imply that $\Pe_n\bar{\ell}_{\text{\upshape w}}(\gvec^*; \betavec)$ converges to $\E_{\text{\upshape S}}[\bar{\ell}_{\text{\upshape w}}(\gvec^*; \betavec)]$ in probability uniformly on $\overline{B(\betavec^*_{\text{\upshape T}}, \epsilon_2)}$; the argument is unchanged because $r(\Xvec) \leq M$ keeps all required moments finite. Second, the quadratic-growth lower bound (inequality (a) in the proof of Lemma~\ref{lem: consistency}): for $\betavec \in \overline{B(\betavec^*_{\text{\upshape T}}, \epsilon_2)}\setminus B(\betavec^*_{\text{\upshape T}}, \epsilon_1)$ and the corresponding intermediate point $\tilde{\betavec} \in [\betavec^*_{\text{\upshape T}}, \betavec]$, the change-of-measure identity \eqref{eq:change_of_measure} gives
$$
(\betavec - \betavec^*_{\text{\upshape T}})^{\top} \E_{\text{\upshape S}}\left[r(\Xvec)\,\Xvec^{\top}\nabla^2 b\big(\Xvec\tilde{\betavec}\big)\Xvec\right](\betavec - \betavec^*_{\text{\upshape T}})\,=\,(\betavec - \betavec^*_{\text{\upshape T}})^{\top} \E_{\text{\upshape T}}\left[\Xvec^{\top}\nabla^2 b\big(\Xvec\tilde{\betavec}\big)\Xvec\right](\betavec - \betavec^*_{\text{\upshape T}})\,\geq\,c_{\text{\upshape w}}\lVert \betavec - \betavec^*_{\text{\upshape T}}\rVert^2
$$
with the new constant $c_{\text{\upshape w}} := \inf\big\{\lambda_{\min}\big(\nabla^2 b(\Xvec\betavec)\big) : \Xvec \in \Breve{\mX},\, \betavec \in \overline{B(\betavec^*_{\text{\upshape T}}, \epsilon_2)}\big\} \cdot \lambda_{\min}\big(\E_{\text{\upshape T}}[\Xvec^{\top}\Xvec]\big) > 0$: the first factor is positive by the same compactness/contradiction argument given in the proof of Lemma~\ref{lem: consistency}, and the second is positive by Assumption~\ref{ass:covariate_shift}(iv); the lower-bound constant thus runs through $\E_{\text{\upshape T}}[\Xvec^{\top}\Xvec]$. All remaining steps---the subgradient inequality, the gradient lower bound \eqref{eqn:gradient-bound}, its extension from the annulus to $\mB \setminus \overline{B(\betavec^*_{\text{\upshape T}}, \epsilon_2)}$ via the integral form of the intermediate value theorem (which only uses $\nabla^2_{\betavec}\bar{\ell}_{\text{\upshape w}} \succeq 0$, valid here since $r(\Xvec) \geq 0$), and the conclusion via Lemma~\ref{lem: empirical-scores shift}---carry over verbatim, and we cite them rather than repeat. This concludes the proof of the lemma. \hfill $\blacksquare$

With these three lemmas, the completion of the proof is identical to the Donsker/Taylor step in Section~\ref{apdx: proof of normality}; we indicate the substitutions. From this point on we suppress $\gvec^*$ and write $\Psivec_{\text{\upshape w}}(\cdot; \betavec)$. First, the local Lipschitz property: for $\betavec_1, \betavec_2 \in B(\betavec^*_{\text{\upshape T}}, \epsilon)$,
$$
\left\lVert \Psivec_{\text{\upshape w}}(\Xi; \betavec_1) - \Psivec_{\text{\upshape w}}(\Xi; \betavec_2)\right\rVert\,=\,r(\Xvec)\left\lVert \Xvec^{\top}\nabla b(\Xvec\betavec_1) - \Xvec^{\top}\nabla b(\Xvec\betavec_2)\right\rVert\,\leq\,M C_1 \lVert\Xvec\rVert \lVert\betavec_1 - \betavec_2\rVert,
$$
where the bound on the unweighted difference is exactly the local Lipschitz display in the proof of Theorem~\ref{thm:normality}, with $C_1 := \sup_{\betavec \in \overline{B(\betavec^*_{\text{\upshape T}}, \epsilon)},\, \Xvec \in \Breve{\mX}}\lVert \nabla^2 b(\Xvec\betavec)\rVert < \infty$ defined as there (with $\betavec^*_{\text{\upshape T}}$ in place of $\betavec^*$); the Lipschitz constant simply picks up the factor $M$. Since $\E_{\text{\upshape S}}\big[M^2 C_1^2 \lVert\Xvec\rVert^2\big] < \infty$, by Example 19.7 and Theorem 19.5 in \cite{van2000asymptotic}, $\left\{\Psivec_{\text{\upshape w}}(\cdot; \betavec) : \betavec \in B(\betavec^*_{\text{\upshape T}}, \epsilon)\right\}$ forms a Donsker class. By Lemma~\ref{lem: consistency shift}, $\hat{\betavec}_{\text{\upshape w}} \to_P \betavec^*_{\text{\upshape T}}$, so by Theorem 19.9 in \cite{van2000asymptotic}, $\Gn\Psivec_{\text{\upshape w}}(\hat{\betavec}_{\text{\upshape w}}) - \Gn\Psivec_{\text{\upshape w}}(\betavec^*_{\text{\upshape T}}) \to_P 0$. Thus, exactly as in Section~\ref{apdx: proof of normality},
\begin{align*}
\sqrt{n}\left(\E_{\text{\upshape S}}[\Psivec_{\text{\upshape w}}(\hat{\betavec}_{\text{\upshape w}})] - \E_{\text{\upshape S}}[\Psivec_{\text{\upshape w}}(\betavec^*_{\text{\upshape T}})]\right)\,&=\,-\big(\Gn\Psivec_{\text{\upshape w}}(\hat{\betavec}_{\text{\upshape w}}) - \Gn\Psivec_{\text{\upshape w}}(\betavec^*_{\text{\upshape T}})\big) + \sqrt{n}\Pe_n\Psivec_{\text{\upshape w}}(\hat{\betavec}_{\text{\upshape w}}) - \sqrt{n}\Pe_n\Psivec_{\text{\upshape w}}(\betavec^*_{\text{\upshape T}})\,\\&=\,-\sqrt{n}\Pe_n\Psivec_{\text{\upshape w}}(\betavec^*_{\text{\upshape T}}) + o_P(1),
\end{align*}

where the last equality follows from Lemma~\ref{lem: empirical-scores shift}. Applying the Taylor expansion to the left-hand side, the Jacobian of the population map $\betavec \mapsto \E_{\text{\upshape S}}[\Psivec_{\text{\upshape w}}(\betavec)]$ at $\betavec^*_{\text{\upshape T}}$ is $\Jvec_{\text{\upshape T}} = \E_{\text{\upshape S}}\left[r(\Xvec)\,\Xvec^{\top}\nabla^2 b\big(\Xvec\betavec^*_{\text{\upshape T}}\big)\Xvec\right]$, and we obtain the counterpart of \eqref{eqn:expansion}:
\begin{align}
\label{eqn:expansion-shift}
    \sqrt{n}\left[\Jvec_{\text{\upshape T}}\big(\hat{\betavec}_{\text{\upshape w}} - \betavec^*_{\text{\upshape T}}\big) + o_P(1)\left\lVert \hat{\betavec}_{\text{\upshape w}} - \betavec^*_{\text{\upshape T}}\right\rVert\right]\,=\,-\sqrt{n}\Pe_n\Psivec_{\text{\upshape w}}(\betavec^*_{\text{\upshape T}}) + o_P(1)\,=\,O_P(1).
\end{align}
Here $\sqrt{n}\Pe_n\Psivec_{\text{\upshape w}}(\betavec^*_{\text{\upshape T}}) = O_P(1)$ by the central limit theorem: it is a scaled i.i.d.\ sum with mean zero by Lemma~\ref{lem: orthogonality shift}, and its covariance matrix $\bm{\Omega}_{\text{\upshape w}} = \E_{\text{\upshape S}}\big[\Psivec_{\text{\upshape w}}(\gvec^*; \betavec^*_{\text{\upshape T}})\Psivec_{\text{\upshape w}}(\gvec^*; \betavec^*_{\text{\upshape T}})^{\top}\big]$ is finite because $r(\Xvec) \leq M$, $\lVert\Xvec\rVert \leq C$, and $\lVert\Cov(\yvec \mid \Xvec, \zvec)\rVert \leq \tilde{\sigma}^2$---the same moment conditions as in the main text (together with the square-integrability of $\yvec$ in Assumption~\ref{ass:regularity}(iv), which bounds the $\zetavec_{\text{\upshape T}}$-block of $\bm{\Omega}_{\text{\upshape w}}$ via $\E_{\text{\upshape S}}\lVert\gvec^*(\Xvec,\zvec)\rVert^2 \leq \E_{\text{\upshape S}}\lVert\yvec\rVert^2$). Moreover, $\Jvec_{\text{\upshape T}} = \E_{\text{\upshape T}}\big[\Xvec^{\top}\nabla^2 b\big(\Xvec\betavec^*_{\text{\upshape T}}\big)\Xvec\big]$ is invertible by the argument in the proof of Lemma~\ref{lem: empirical-scores shift}. Therefore, \eqref{eqn:expansion-shift} implies $\sqrt{n}\big\lVert \hat{\betavec}_{\text{\upshape w}} - \betavec^*_{\text{\upshape T}}\big\rVert = O_P(1)$, and applying \eqref{eqn:expansion-shift} again,
$$
\sqrt{n}\big(\hat{\betavec}_{\text{\upshape w}} - \betavec^*_{\text{\upshape T}}\big)\,=\,-\Jvec_{\text{\upshape T}}^{-1}\sqrt{n}\Pe_n\Psivec_{\text{\upshape w}}(\betavec^*_{\text{\upshape T}}) + o_P(1)\,\rightsquigarrow\,N\big(\bm{0}, \Jvec_{\text{\upshape T}}^{-1}\bm{\Omega}_{\text{\upshape w}}\Jvec_{\text{\upshape T}}^{-1}\big)\,=\,N\big(\bm{0}, \Sigmavec_{\text{\upshape w}}\big).
$$

It remains to verify the decomposition \eqref{eqn:omega-w-decomposition}. By the definitions of $\zetavec_{\text{\upshape T}}$ and $\pivec$ in Section~\ref{apdx: cs setting},
$$
\Psivec_{\text{\upshape w}}\big(\Xi; \gvec^*; \betavec^*_{\text{\upshape T}}\big)\,=\,r(\Xvec)\,\Xvec^{\top}\left[\nabla b\big(\Xvec\betavec^*_{\text{\upshape T}}\big) - \gvec^*(\Xvec,\zvec) + \frac{w}{\rho}\big(\gvec^*(\Xvec,\zvec) - \yvec\big)\right]\,=\,r(\Xvec)\left(\zetavec_{\text{\upshape T}} + \frac{w}{\rho}\pivec\right),
$$
so that
$$
\bm{\Omega}_{\text{\upshape w}}\,=\,\E_{\text{\upshape S}}\left[r(\Xvec)^2\, \zetavec_{\text{\upshape T}}\zetavec_{\text{\upshape T}}^{\top}\right] + \E_{\text{\upshape S}}\left[r(\Xvec)^2\, \frac{w}{\rho}\big(\zetavec_{\text{\upshape T}}\pivec^{\top} + \pivec\zetavec_{\text{\upshape T}}^{\top}\big)\right] + \E_{\text{\upshape S}}\left[r(\Xvec)^2\, \frac{w^2}{\rho^2}\,\pivec\pivec^{\top}\right].
$$
For the cross terms, condition on $(\Xvec, \zvec)$: both $r(\Xvec)$ and $\zetavec_{\text{\upshape T}}$ are functions of $(\Xvec, \zvec)$, and since $w \perp (\Xvec, \yvec, \zvec)$ with $\E[w/\rho] = 1$ and $\E[\pivec \mid \Xvec, \zvec] = \Xvec^{\top}\big(\gvec^*(\Xvec,\zvec) - \E[\yvec \mid \Xvec, \zvec]\big) = 0$, we have $\E\left[(w/\rho)\,\pivec \mid \Xvec, \zvec\right] = \E[w/\rho]\cdot\E[\pivec \mid \Xvec, \zvec] = 0$; hence both cross terms vanish. For the last term, $w^2 = w$ and the independence of $w$ give $\E_{\text{\upshape S}}\left[r(\Xvec)^2 (w/\rho)^2 \pivec\pivec^{\top}\right] = \E[w/\rho^2]\cdot\E_{\text{\upshape S}}\left[r(\Xvec)^2\pivec\pivec^{\top}\right] = \frac{1}{\rho}\E_{\text{\upshape S}}\left[r(\Xvec)^2\pivec\pivec^{\top}\right]$. This proves \eqref{eqn:omega-w-decomposition}, paralleling the computation of \eqref{eqn:variance-gai-expansion}, and completes the proof of the normality claim.

\medskip
\noindent\emph{Dominance of $\hat{\betavec}_{\text{\upshape w}}$ over $\hat{\betavec}^{\text{\upshape P},\text{\upshape w}}$.} The dominance argument mirrors the proof of Corollary~\ref{cor:dominance} in Section~\ref{apdx: proof of dominance}. To start, as noted in Section~\ref{apdx: cs known}, the estimating equation \eqref{eq:primary_score_shift} is equivalent to $\sum_{i=1}^{n} w_i\, r(\Xvec_i)\, \nabla_{\betavec}\ell(\Xvec_i, \yvec_i; \betavec) = 0$, and its population validity at $\betavec^*_{\text{\upshape T}}$, i.e., $\E_{\text{\upshape S}}\big[w\, r(\Xvec)\, \nabla_{\betavec}\ell\big(\Xvec, \yvec; \betavec^*_{\text{\upshape T}}\big)\big] = 0$, was established in the display preceding the statement of Theorem~\ref{thm:normality_shift}; we do not repeat it. Thus, under the preset assumptions, with probability approaching one a solution $\hat{\betavec}^{\text{\upshape P},\text{\upshape w}}$ to this equation exists and $\sqrt{n}\big(\hat{\betavec}^{\text{\upshape P},\text{\upshape w}} - \betavec^*_{\text{\upshape T}}\big) \rightsquigarrow N\big(\bm{0}, \Sigmavec^{\text{\upshape P},\text{\upshape w}}\big)$: the proof follows the same M-estimation argument as the normality argument above, with the bounded known weight $w\, r(\Xvec)$ in place of $r(\Xvec)$ and no nuisance functions (so the argument only simplifies), exactly as the proof of Corollary~\ref{cor:dominance} treats the unweighted primary-only estimator. We skip the proof and record the two components of the sandwich covariance. For the Jacobian, by the independence of $w$ under Assumption~\ref{ass:covariate_shift}(iii),
$$
\E_{\text{\upshape S}}\left[w\, r(\Xvec)\, \Xvec^{\top}\nabla^2 b\big(\Xvec\betavec^*_{\text{\upshape T}}\big)\Xvec\right]\,=\,\rho\, \E_{\text{\upshape S}}\left[r(\Xvec)\, \Xvec^{\top}\nabla^2 b\big(\Xvec\betavec^*_{\text{\upshape T}}\big)\Xvec\right]\,=\,\rho\, \Jvec_{\text{\upshape T}}.
$$
For the score variance, write, as in the proof of Corollary~\ref{cor:dominance},
$$
\nabla_{\betavec}\ell\big(\Xvec, \yvec; \betavec^*_{\text{\upshape T}}\big)\,=\,\zetavec_{\text{\upshape T}}(\Xvec, \zvec) + \pivec(\Xvec, \yvec, \zvec),
$$
now with $\betavec^*_{\text{\upshape T}}$ in place of $\betavec^*$ in $\zetavec_{\text{\upshape T}}$. Because $\E\left[\gvec^*(\Xvec, \zvec) - \yvec \mid \Xvec, \zvec\right] = 0$, the cross term between $\zetavec_{\text{\upshape T}}$ and $\pivec$ vanishes by the same iterated-expectations argument used for \eqref{eqn:variance-primary-expansion}, and, using $w^2 = w$ together with the independence of $w$,
$$
\E_{\text{\upshape S}}\left[w\, r(\Xvec)^2\, \nabla_{\betavec}\ell\big(\Xvec, \yvec; \betavec^*_{\text{\upshape T}}\big)\nabla_{\betavec}\ell\big(\Xvec, \yvec; \betavec^*_{\text{\upshape T}}\big)^{\top}\right]\,=\,\rho\left(\E_{\text{\upshape S}}\left[r(\Xvec)^2\, \zetavec_{\text{\upshape T}}\zetavec_{\text{\upshape T}}^{\top}\right] + \E_{\text{\upshape S}}\left[r(\Xvec)^2\, \pivec\pivec^{\top}\right]\right).
$$
The sandwich formula then gives
$$
\Sigmavec^{\text{\upshape P},\text{\upshape w}}\,=\,(\rho\Jvec_{\text{\upshape T}})^{-1}\, \rho\left(\E_{\text{\upshape S}}\left[r(\Xvec)^2\, \zetavec_{\text{\upshape T}}\zetavec_{\text{\upshape T}}^{\top}\right] + \E_{\text{\upshape S}}\left[r(\Xvec)^2\, \pivec\pivec^{\top}\right]\right)(\rho\Jvec_{\text{\upshape T}})^{-1}\,=\,\frac{1}{\rho}\, \Jvec_{\text{\upshape T}}^{-1}\left(\E_{\text{\upshape S}}\left[r(\Xvec)^2\, \zetavec_{\text{\upshape T}}\zetavec_{\text{\upshape T}}^{\top}\right] + \E_{\text{\upshape S}}\left[r(\Xvec)^2\, \pivec\pivec^{\top}\right]\right)\Jvec_{\text{\upshape T}}^{-1},
$$
which is the expression for $\Sigmavec^{\text{\upshape P},\text{\upshape w}}$ in Theorem~\ref{thm:normality_shift} and parallels \eqref{eqn:variance-primary-expansion}. On the GAI side, as shown in the normality argument above, $\Psivec_{\text{\upshape w}}\big(\Xi; \gvec^*; \betavec^*_{\text{\upshape T}}\big) = r(\Xvec)\big(\zetavec_{\text{\upshape T}} + \frac{w}{\rho}\pivec\big)$ and $\bm{\Omega}_{\text{\upshape w}}$ admits the decomposition \eqref{eqn:omega-w-decomposition}; hence
$$
\Sigmavec_{\text{\upshape w}}\,=\,\Jvec_{\text{\upshape T}}^{-1}\E_{\text{\upshape S}}\left[r(\Xvec)^2\, \zetavec_{\text{\upshape T}}\zetavec_{\text{\upshape T}}^{\top}\right]\Jvec_{\text{\upshape T}}^{-1} + \frac{1}{\rho}\,\Jvec_{\text{\upshape T}}^{-1}\E_{\text{\upshape S}}\left[r(\Xvec)^2\, \pivec\pivec^{\top}\right]\Jvec_{\text{\upshape T}}^{-1},
$$
paralleling \eqref{eqn:variance-gai-expansion}. Subtracting the two displays yields \eqref{eqn:variance-gap-shift}, which is positive semidefinite because $\rho \leq 1$ and $\E_{\text{\upshape S}}\left[r(\Xvec)^2\, \zetavec_{\text{\upshape T}}\zetavec_{\text{\upshape T}}^{\top}\right] \succeq 0$. Moreover, since $\Jvec_{\text{\upshape T}}$ is invertible, the gap \eqref{eqn:variance-gap-shift} is positive definite if and only if $\rho < 1$ and $\E_{\text{\upshape S}}\left[r(\Xvec)^2\, \zetavec_{\text{\upshape T}}\zetavec_{\text{\upshape T}}^{\top}\right] \succ 0$; substituting $\zetavec_{\text{\upshape T}} = \Xvec^{\top}\big(\nabla b\big(\Xvec\betavec^*_{\text{\upshape T}}\big) - \E[\yvec \mid \Xvec, \zvec]\big)$ gives exactly the dominance condition displayed in Theorem~\ref{thm:normality_shift}. This completes the proof of the dominance claim, and with it the proof of Theorem~\ref{thm:normality_shift}.

\subsection{Proof of Theorem \ref{thm:normality_dr_body}}
\label{apdx: proof normality dr}
\label{apdx: cs estimated}

The proof follows exactly the structure of the proof of Theorem~\ref{thm:normality} in Section~\ref{apdx: proof of normality} and of its known-ratio counterpart in Section~\ref{apdx: proof normality shift}: we first establish the counterparts of Lemmas~\ref{lem: orthogonality}--\ref{lem: consistency} for the doubly robust score \eqref{eq:score_dr_body}, and then complete the argument with the same Donsker/Taylor step, now with a two-sample central limit theorem. For each step we indicate explicitly whether the existing argument applies---in which case we cite it and state the substitution---and we write out in full only the steps that are genuinely new. The genuinely new ingredients are all consequences of the two-sample structure.

We first fix notation and the independence convention. We write $\Pe_{n_{\text{\upshape T}}}$ and $\Pe_{n_{\text{\upshape S}}}$ for the empirical means over the target sample $\mD_{\text{\upshape T}}$ and the source sample $\mD_{\text{\upshape S}}$, respectively, e.g., $\Pe_{n_{\text{\upshape T}}} f := \frac{1}{n_{\text{\upshape T}}}\sum_{j=1}^{n_{\text{\upshape T}}} f(\Xvec_j, \zvec_j)$, and $\mathbb{G}_{n_{\text{\upshape T}}} f := \sqrt{n_{\text{\upshape T}}}\left(\Pe_{n_{\text{\upshape T}}} f - \E_{\text{\upshape T}} f\right)$ for the target empirical process. As in Section~\ref{apdx: proof of normality}, we assume without loss of generality (the split-sample device) that the nuisances are trained on auxiliary data, collectively denoted $I^c$, with the independence structure stated in Section~\ref{apdx: cs estimated}: $\hat{\gvec}$ is trained on an auxiliary source fold, hence independent of all observations entering \eqref{eq:score_dr_body}---independence from the target sample is automatic, since $\hat{\gvec}$ uses only source data and the two samples are independent; $\hat{r}$ is trained on data independent of the source observations entering the bias-correction term, and is allowed to depend on the target covariates. The latter is innocuous: $\hat{r}$ multiplies only source rows in \eqref{eq:score_dr_body}, and the only term of the decomposition below that involves the target-sample average, term (I$'$), involves the \emph{true} ratio $r^*$ rather than $\hat{r}$, so a possible dependence of $\hat{r}$ on $\mD_{\text{\upshape T}}$ never interacts with the target-sample average. Accordingly, in the conditional Markov arguments below we condition, exactly as in the proof of Lemma~\ref{lem: orthogonality}, on the training data of the nuisances entering the term at hand: on all of $I^c$ for the terms built from source evaluation observations, which are independent of $I^c$ in its entirety, and on the auxiliary source fold alone for the target-sample part of term (I$'$), which involves only $\hat{\gvec}$. For brevity we write all such conditionings as conditioning on $I^c$. We use the same constant $C$ with $\lVert \Xvec \rVert, \lVert \Xvec \rVert_F \leq C$ for all $\Xvec \in \Breve{\mX}$; by Assumption~\ref{ass:covariate_shift}(ii) the same bound applies under the target population. We abbreviate $\Delta_{\gvec} := \hat{\gvec} - \gvec^*$ and $\Delta_r := \hat{r} - r^*$.

Mirroring the definition of $\tauvec(\Xi; e, \gvec)$ at the start of Section~\ref{apdx: proof of normality}, we collect the $\betavec$-free part of the doubly robust score:
$$
\tauvec_{\text{\upshape DR}}(\gvec, r)\,:=\,\Pe_{n_{\text{\upshape T}}}\left[\Xvec^{\top}\gvec(\Xvec, \zvec)\right] - \Pe_{n_{\text{\upshape S}}}\left[r(\Xvec)\,\Xvec^{\top}\big(\gvec(\Xvec, \zvec) - \yvec\big)\right],
$$
so that
$$
\bm{\Psi}_{\text{\upshape DR}}(\gvec, r; \betavec)\,=\,\Pe_{n_{\text{\upshape T}}}\left[\Xvec^{\top}\nabla b(\Xvec\betavec)\right] - \tauvec_{\text{\upshape DR}}(\gvec, r) \quad \text{for all } \betavec\,\in\,\mB.
$$
Two structural facts drive the entire proof and we highlight them at the outset: (a) only the target-sample term of \eqref{eq:score_dr_body} depends on $\betavec$---the bias-correction term is a constant vector in $\betavec$; and (b) the two samples are independent. Fact (a) means that the convexity machinery of Lemmas~\ref{lem: empirical-scores} and~\ref{lem: consistency} operates entirely on the target sample, with the bias correction entering only as a $\betavec$-free random shift; fact (b) is what permits the two-sample mean cancellation in term (I$'$) below and the two-sample central limit theorem at the end.

The first lemma is the counterpart of Lemma~\ref{lem: orthogonality}.

\begin{lemma}[\normalfont \textbf{Score Validity and Neyman Orthogonality for the Doubly Robust Score}]
\label{lem: orthogonality dr} It holds that
$$
\E_{\text{\upshape T}}\left[\zetavec_{\text{\upshape T}}\right]\,=\,0, \qquad \E_{\text{\upshape S}}\left[r^*(\Xvec)\,\pivec\right]\,=\,0, \qquad\text{and}\qquad
\tauvec_{\text{\upshape DR}}(\hat{\gvec}, \hat{r}) - \tauvec_{\text{\upshape DR}}(\gvec^*, r^*)\,=\,o_P\big(n_{\text{\upshape S}}^{-1/2}\big).
$$
\end{lemma}

Throughout the proof we use the prediction residual $\zetavec_{\text{\upshape T}}(\Xvec, \zvec) = \Xvec^{\top}\big(\nabla b\big(\Xvec\betavec^*_{\text{\upshape T}}\big) - \gvec^*(\Xvec, \zvec)\big)$ and the label residual $\pivec(\Xvec, \yvec, \zvec) = \Xvec^{\top}\big(\gvec^*(\Xvec, \zvec) - \yvec\big)$, both defined in Section~\ref{apdx: cs setting}. The first two claims state that the population means of the two sample-terms of $\bm{\Psi}_{\text{\upshape DR}}\big(\gvec^*, r^*; \betavec^*_{\text{\upshape T}}\big)$---namely $\E_{\text{\upshape T}}\big[\Xvec^{\top}\big(\nabla b\big(\Xvec\betavec^*_{\text{\upshape T}}\big) - \gvec^*\big)\big] = \E_{\text{\upshape T}}[\zetavec_{\text{\upshape T}}]$ and $\E_{\text{\upshape S}}\big[r^*(\Xvec)\pivec\big]$---are \emph{each} zero, not merely that their sum is zero. This is a genuinely new requirement relative to Lemma~\ref{lem: orthogonality}, where validity of the score only requires the single population mean $\E[\Psivec(e^*, \gvec^*; \betavec^*)]$ to vanish: here the two terms are averages over two different samples with different sizes, so the two-sample central limit theorem at the end of the proof requires each sample-term to be centered separately.

\textit{Proof of Lemma~\ref{lem: orthogonality dr}.} We begin with the two validity claims; both are new but short. For the second claim,
$$
\E_{\text{\upshape S}}\left[r^*(\Xvec)\,\pivec\right]
\,\stackrel{\text{(a)}}{=}\,\E_{\text{\upshape S}}\left[r^*(\Xvec)\,\E[\pivec \mid \Xvec, \zvec]\right]
\,\stackrel{\text{(b)}}{=}\,\E_{\text{\upshape T}}\left[\pivec\right]\,\stackrel{\text{(c)}}{=}\,0;
$$
where (a) follows from the law of iterated expectations, conditioning on $(\Xvec, \zvec)$ (note that $r^*(\Xvec)$ is a function of $\Xvec$ and passes through the conditional expectation), (b) is the change-of-measure identity \eqref{eq:change_of_measure}, and (c) uses $\E[\pivec \mid \Xvec, \zvec] = \Xvec^{\top}\big(\gvec^*(\Xvec, \zvec) - \E[\yvec \mid \Xvec, \zvec]\big) = \bm{0}$, since $\gvec^*(\Xvec, \zvec) = \E[\yvec \mid \Xvec, \zvec]$. The first expectation vanishes already because $\E[\pivec \mid \Xvec, \zvec] = \bm{0}$, and we record the identity $\E_{\text{\upshape T}}[\pivec] = 0$ separately because we use it next. For the first claim,
$$
\E_{\text{\upshape T}}\left[\zetavec_{\text{\upshape T}}\right]\,\stackrel{\text{(d)}}{=}\,\E_{\text{\upshape T}}\left[\Xvec^{\top}\left(\nabla b\big(\Xvec\betavec^*_{\text{\upshape T}}\big) - \yvec\right)\right] - \E_{\text{\upshape T}}\left[\pivec\right]\,\stackrel{\text{(e)}}{=}\,0 - 0\,=\,0,
$$
where (d) adds and subtracts $\E_{\text{\upshape T}}[\yvec]$ to write $\zetavec_{\text{\upshape T}} = \Xvec^{\top}\big(\nabla b(\Xvec\betavec^*_{\text{\upshape T}}) - \yvec\big) - \pivec$, and (e) follows because the first expectation is zero by the target first-order condition \eqref{eq:foc_shift} and the second is zero by claim two just established. Note that no change of measure is needed for the first claim: the prediction component of \eqref{eq:score_dr_body} is evaluated directly under the target population, which is precisely the point of the two-sample construction.

For the third claim, decompose $\hat{r}(\hat{\gvec} - \yvec) - r^*(\gvec^* - \yvec) = r^*\Delta_{\gvec} + \Delta_r(\gvec^* - \yvec) + \Delta_r\Delta_{\gvec}$ inside the bias-correction term to obtain
\begin{align}
\label{eqn:decomposition-dr}
    \sqrt{n_{\text{\upshape S}}}\left(\tauvec_{\text{\upshape DR}}(\hat{\gvec}, \hat{r}) - \tauvec_{\text{\upshape DR}}(\gvec^*, r^*)\right)
\,=\,\text{(I$'$)} - \text{(II$'$)} - \text{(III$'$)},
\end{align}
where
\begin{align*}
    \text{(I$'$)}\,&:=\,\sqrt{n_{\text{\upshape S}}}\,\Pe_{n_{\text{\upshape T}}}\left[\Xvec^{\top}\Delta_{\gvec}\right] - \sqrt{n_{\text{\upshape S}}}\,\Pe_{n_{\text{\upshape S}}}\left[r^*(\Xvec)\,\Xvec^{\top}\Delta_{\gvec}\right],\\
    \text{(II$'$)}\,&:=\,\sqrt{n_{\text{\upshape S}}}\,\Pe_{n_{\text{\upshape S}}}\left[\Delta_r(\Xvec)\,\Xvec^{\top}\big(\gvec^*(\Xvec,\zvec) - \yvec\big)\right],\\
    \text{(III$'$)}\,&:=\,\sqrt{n_{\text{\upshape S}}}\,\Pe_{n_{\text{\upshape S}}}\left[\Delta_r(\Xvec)\,\Xvec^{\top}\Delta_{\gvec}\right].
\end{align*}
This decomposition is the two-sample counterpart of the decomposition (I)--(III) in the proof of Lemma~\ref{lem: orthogonality}: (II$'$) and (III$'$) mirror (II) and (III) with $\Delta_r$ in place of $\hat{e} - e^*$, while (I$'$) replaces the single-sample term (I) by a \emph{difference of averages over the two samples}, which is the genuinely new structure. We treat each term separately, conditioning on $I^c$ throughout.

\emph{Term (I$'$).} Unlike term (I) in the proof of Lemma~\ref{lem: orthogonality}, neither average in (I$'$) has zero conditional mean by itself. Instead, the two conditional means cancel exactly across the two samples: conditional on $I^c$,
$$
\E\left[r^*(\Xvec)\,\Xvec^{\top}\Delta_{\gvec} \,\bigg|\, I^c\right]\,\stackrel{\text{(a)}}{=}\,\E_{\text{\upshape S}}\left[r^*(\Xvec)\,\Xvec^{\top}\Delta_{\gvec} \,\big|\, I^c\right]\,\stackrel{\text{(b)}}{=}\,\E_{\text{\upshape T}}\left[\Xvec^{\top}\Delta_{\gvec} \,\big|\, I^c\right],
$$
where (a) records that the source observations entering the bias-correction term have law $Q_{\text{\upshape S}}$ given $I^c$, and (b) is the change-of-measure identity \eqref{eq:change_of_measure}; the right-hand side is exactly the conditional mean of $\Pe_{n_{\text{\upshape T}}}\left[\Xvec^{\top}\Delta_{\gvec}\right]$ (here we use that $\hat{\gvec}$ is independent of both evaluation samples). Therefore
$$
\text{(I$'$)}\,=\,\sqrt{n_{\text{\upshape S}}}\left(\Pe_{n_{\text{\upshape T}}} - \E_{\text{\upshape T}}\right)\left[\Xvec^{\top}\Delta_{\gvec}\right] - \sqrt{n_{\text{\upshape S}}}\left(\Pe_{n_{\text{\upshape S}}} - \E_{\text{\upshape S}}\right)\left[r^*(\Xvec)\,\Xvec^{\top}\Delta_{\gvec}\right],
$$
a difference of two mean-deviation terms that are independent conditional on the auxiliary source fold on which $\hat{\gvec}$ is trained (fact (b)). Each is controlled by the conditional Markov/variance argument of \eqref{eqn:I}, with new constants. For the target part, the summands are i.i.d.\ conditional on $I^c$ (by the convention above, the conditioning here is on the auxiliary source fold that trains $\hat{\gvec}$, of which the target sample is independent), so
$$
\E\left[\left\lVert \sqrt{n_{\text{\upshape S}}}\left(\Pe_{n_{\text{\upshape T}}} - \E_{\text{\upshape T}}\right)\left[\Xvec^{\top}\Delta_{\gvec}\right]\right\rVert^2 \,\bigg|\, I^c\right]
\,\stackrel{\text{(c)}}{\leq}\,\frac{n_{\text{\upshape S}}}{n_{\text{\upshape T}}}\, \E_{\text{\upshape T}}\left[\left\lVert \Xvec^{\top}\Delta_{\gvec}\right\rVert^2 \,\big|\, I^c\right]
\,\stackrel{\text{(d)}}{\leq}\,\frac{n_{\text{\upshape S}}}{n_{\text{\upshape T}}}\, C^2 \left\lVert \Delta_{\gvec} \right\rVert_{Q_{\text{\upshape T}},2}^2
\,\stackrel{\text{(e)}}{\leq}\,\frac{n_{\text{\upshape S}}}{n_{\text{\upshape T}}}\, C^2 M \left\lVert \Delta_{\gvec} \right\rVert_{Q_{\text{\upshape S}},2}^2 \rightarrow 0,
$$
where (c) bounds the variance of the mean-deviation term by $1/n_{\text{\upshape T}}$ times the summand second moment (the summands being i.i.d.\ conditional on $I^c$), (d) uses $\lVert \Xvec \rVert \leq C$, and (e) is the $Q_{\text{\upshape T}}$-norm conversion derived after Assumption~\ref{ass:ml_rate_shift}; the convergence holds because $n_{\text{\upshape S}}/n_{\text{\upshape T}} \to \gamma < \infty$ is bounded and $\lVert \Delta_{\gvec} \rVert_{Q_{\text{\upshape S}},2} \to 0$ by Assumption~\ref{ass:ml_rate_shift}. The requirement $n_{\text{\upshape S}}/n_{\text{\upshape T}} \to \gamma < \infty$ enters the proof exactly here (and again in the Donsker and CLT steps below). For the source part,
$$
\E\left[\left\lVert \sqrt{n_{\text{\upshape S}}}\left(\Pe_{n_{\text{\upshape S}}} - \E_{\text{\upshape S}}\right)\left[r^*(\Xvec)\,\Xvec^{\top}\Delta_{\gvec}\right]\right\rVert^2 \,\bigg|\, I^c\right]
\,\stackrel{\text{(f)}}{\leq}\,\E_{\text{\upshape S}}\left[r^*(\Xvec)^2\left\lVert \Xvec^{\top}\Delta_{\gvec}\right\rVert^2 \,\bigg|\, I^c\right]
\,\stackrel{\text{(g)}}{\leq}\,M^2 C^2 \left\lVert \Delta_{\gvec} \right\rVert_{Q_{\text{\upshape S}},2}^2 \rightarrow 0,
$$
where (f) bounds the variance of the mean-deviation term by the summand second moment, and (g) uses $r^*(\Xvec) \leq M$ and $\lVert \Xvec \rVert \leq C$.
By the conditional Markov inequality applied as in \eqref{eqn:I} to each of the two parts, followed by the same uniform-integrability argument as in the proof of Lemma~\ref{lem: orthogonality} to remove the conditioning, (I$'$) is $o_P(1)$. The two-sample mean cancellation and the $Q_{\text{\upshape T}}$-norm conversion are the new ingredients; the Markov/variance mechanics are unchanged.

\emph{Term (II$'$).} This term mirrors term (II) in the proof of Lemma~\ref{lem: orthogonality} with $\Delta_r$ in place of $\hat{e} - e^*$, and is in fact simpler: because $\hat{r}$ enters multiplicatively rather than through an estimated denominator, no analog of the factor $1/\kappa^2$ appears and no sup-norm condition on $\hat{r}$ is invoked. Conditional on $I^c$, (II$'$) has zero mean: by the argument of \eqref{eqn:I-expectation}, $\E\big[\gvec^*(\Xvec,\zvec) - \yvec \mid \Xvec, \zvec, I^c\big] = \E\big[\gvec^*(\Xvec,\zvec) - \yvec \mid \Xvec, \zvec\big] = 0$, and $\Delta_r(\Xvec)$ is a function of $(\Xvec, I^c)$ only (recall that $\hat{r}$ is trained on data independent of the source evaluation observations; its possible dependence on the target covariates is part of $I^c$ and is irrelevant here, since (II$'$) involves only source observations). For the conditional variance, by the same computation as in the bound for (II) in the proof of Lemma~\ref{lem: orthogonality},
$$
\E\left[\left\lVert \text{(II$'$)} \right\rVert^2 \,\big|\, I^c\right]
\,\stackrel{\text{(a)}}{\leq}\,\E_{\text{\upshape S}}\left[\Delta_r(\Xvec)^2 \left\lVert \Xvec^{\top}\big(\gvec^*(\Xvec,\zvec) - \yvec\big)\right\rVert^2 \,\bigg|\, I^c\right]
\,\stackrel{\text{(b)}}{\leq}\,\tilde{\sigma}^2 C^2 \left\lVert \Delta_r \right\rVert_{Q_{\text{\upshape S}},2}^2 \rightarrow 0,
$$
where (a) bounds the variance of the zero-mean i.i.d.\ average by the summand second moment, and (b) uses $\Cov(\yvec \mid \Xvec, \zvec, I^c) = \Cov(\yvec \mid \Xvec, \zvec)$ with $\lVert \Cov(\yvec \mid \Xvec, \zvec)\rVert \leq \tilde{\sigma}^2$ and $\lVert \Xvec \rVert_F \leq C$; the convergence holds by Assumption~\ref{ass:ml_rate_shift}. The conditional Markov inequality and the uniform-integrability argument then give (II$'$) $= o_P(1)$, exactly as for (II).

\emph{Term (III$'$).} This term mirrors term (III) in the proof of Lemma~\ref{lem: orthogonality}, with the product $\Delta_r \Delta_{\gvec}$ in place of $(\hat{e} - e^*)(\hat{\gvec} - \gvec^*)$ and no constant from an estimated denominator (the factor $2/\kappa^2$ obtained there is absent because $\hat{r}$ enters multiplicatively). Conditional on $I^c$,
$$
\Pe\left(\left\lVert \text{(III$'$)} \right\rVert\,\geq\,\epsilon \,\big|\, I^c\right)
\,\stackrel{\text{(a)}}{\leq}\,\frac{\sqrt{n_{\text{\upshape S}}}}{\epsilon}\, \E_{\text{\upshape S}}\left[\left|\Delta_r(\Xvec)\right| \left\lVert \Xvec^{\top}\Delta_{\gvec} \right\rVert \,\bigg|\, I^c\right]
\,\stackrel{\text{(b)}}{\leq}\,\frac{C\sqrt{n_{\text{\upshape S}}}}{\epsilon}\, \left\lVert \Delta_r \right\rVert_{Q_{\text{\upshape S}},2} \left\lVert \Delta_{\gvec} \right\rVert_{Q_{\text{\upshape S}},2}
\,\stackrel{\text{(c)}}{\leq}\,\frac{C}{\epsilon}\, \alpha(n_{\text{\upshape S}})^2\, n_{\text{\upshape S}}^{1/2 - r_g - r_r}
$$
which converges to zero, 
where (a) is the Markov inequality applied to the first moment, (b) uses $\lVert \Xvec \rVert \leq C$ together with the Cauchy--Schwarz inequality, and (c) uses Assumption~\ref{ass:ml_rate_shift} and $r_g + r_r \geq 1/2$; this is exactly the role played by the product-rate condition $r_1 + r_2 \geq 1/2$ for term (III) in the proof of Lemma~\ref{lem: orthogonality}. Hence (III$'$) $= o_P(1)$, and combining the three terms in \eqref{eqn:decomposition-dr} concludes the proof of the lemma. \hfill $\blacksquare$

The next lemma is the counterpart of Lemma~\ref{lem: empirical-scores}.

\begin{lemma}[\normalfont \textbf{Rates of Empirical Scores for the Doubly Robust Estimator}]
\label{lem: empirical-scores dr} It holds that
$$
\big\lVert \bm{\Psi}_{\text{\upshape DR}}\big(\gvec^*, r^*; \hat{\betavec}_{\text{\upshape DR}}\big) \big\rVert,\, \inf_{\betavec \in \mB}\big\lVert \bm{\Psi}_{\text{\upshape DR}}\big(\hat{\gvec}, \hat{r}; \betavec\big) \big\rVert\,=\,o_P\big(1/\sqrt{n_{\text{\upshape S}}}\big).
$$
\end{lemma}

\textit{Proof of Lemma~\ref{lem: empirical-scores dr}.} Mirroring $\bar{\ell}$ in the proof of Lemma~\ref{lem: empirical-scores}, define, for nuisances $(\gvec, r)$,
$$
\bar{\ell}_{\text{\upshape DR}}(\gvec, r; \betavec)\,:=\,\Pe_{n_{\text{\upshape T}}}\left[b(\Xvec\betavec) - \gvec(\Xvec, \zvec)^{\top}\Xvec\betavec\right] + \bm{\xi}_{n_{\text{\upshape S}}}(\gvec, r)^{\top}\betavec,
\qquad
\bm{\xi}_{n_{\text{\upshape S}}}(\gvec, r)\,:=\,\Pe_{n_{\text{\upshape S}}}\left[r(\Xvec)\,\Xvec^{\top}\big(\gvec(\Xvec, \zvec) - \yvec\big)\right],
$$
so that $\nabla_{\betavec}\bar{\ell}_{\text{\upshape DR}}(\gvec, r; \betavec) = \bm{\Psi}_{\text{\upshape DR}}(\gvec, r; \betavec)$. The correction term $\bm{\xi}_{n_{\text{\upshape S}}}(\gvec, r)$ is a \emph{constant vector in $\betavec$} (fact (a)), so it contributes a linear term to the objective and drops out of the Hessian:
$$
\nabla^2_{\betavec}\bar{\ell}_{\text{\upshape DR}}(\gvec, r; \betavec)\,=\,\Pe_{n_{\text{\upshape T}}}\left[\Xvec^{\top}\nabla^2 b(\Xvec\betavec)\Xvec\right]\,\succeq\,0,
$$
so $\bar{\ell}_{\text{\upshape DR}}$ is convex in $\betavec$ regardless of the nuisances; this one-line observation is the only new step needed for convexity. The population (true-nuisance) objective is
$$
L(\betavec)\,:=\,\E_{\text{\upshape T}}\left[b(\Xvec\betavec) - \gvec^*(\Xvec, \zvec)^{\top}\Xvec\betavec\right],
$$
with no linear correction, because the population mean of $\bm{\xi}_{n_{\text{\upshape S}}}(\gvec^*, r^*)$ is $\E_{\text{\upshape S}}\big[r^*(\Xvec)\pivec\big] = 0$ by Lemma~\ref{lem: orthogonality dr}. By Lemma~\ref{lem: orthogonality dr} again, $\nabla L\big(\betavec^*_{\text{\upshape T}}\big) = \E_{\text{\upshape T}}\big[\zetavec_{\text{\upshape T}}\big] = 0$, and for any fixed $\betavec \in \mB$,
$$
\nabla^2 L(\betavec)\,=\,\E_{\text{\upshape T}}\left[\Xvec^{\top}\nabla^2 b(\Xvec\betavec)\Xvec\right]\,\stackrel{\text{(a)}}{\succeq}\,\inf_{\Xvec \in \Breve{\mX}}\lambda_{\min}\big(\nabla^2 b(\Xvec\betavec)\big)\cdot \E_{\text{\upshape T}}\left[\Xvec^{\top}\Xvec\right]\,\stackrel{\text{(b)}}{\succ}\,0,
$$
where (a) bounds $\nabla^2 b(\Xvec\betavec)$ below by its smallest eigenvalue over $\Breve{\mX}$, and (b) holds because the positivity of the infimum is the compactness/contradiction argument in the proof of Lemma~\ref{lem: empirical-scores}, which we cite rather than repeat, and $\E_{\text{\upshape T}}[\Xvec^{\top}\Xvec] \succ 0$ by Assumption~\ref{ass:covariate_shift}(iv); this is the same substitution---the target second-moment matrix in place of $\E[\Xvec^{\top}\Xvec]$---already used in the proof of Lemma~\ref{lem: empirical-scores shift}. Hence $L$ is strictly convex and $\betavec^*_{\text{\upshape T}}$ is its unique minimizer.

We next verify that the empirical objective at the true nuisances converges pointwise to $L$, which is where the $\betavec$-free correction requires a (short) new argument. The target-sample part $\Pe_{n_{\text{\upshape T}}}\big[b(\Xvec\betavec) - \gvec^{*\top}\Xvec\betavec\big]$ converges in probability to $\E_{\text{\upshape T}}\big[b(\Xvec\betavec) - \gvec^{*\top}\Xvec\betavec\big]$ for each $\betavec$ by the weak law of large numbers, exactly as in the proof of Lemma~\ref{lem: empirical-scores}, with the required moments finite under $Q_{\text{\upshape T}}$ by the boundedness of $\Xvec$ and $\E_{\text{\upshape T}}\lVert\gvec^*\rVert^2 \stackrel{\text{(a)}}{\leq} \E_{\text{\upshape T}}\lVert\yvec\rVert^2 \stackrel{\text{(b)}}{=} \E_{\text{\upshape S}}\big[r^*(\Xvec)\lVert\yvec\rVert^2\big] \stackrel{\text{(c)}}{\leq} M\,\E_{\text{\upshape S}}\lVert\yvec\rVert^2 < \infty$, where (a) is the conditional Jensen inequality applied to $\gvec^*(\Xvec,\zvec) = \E[\yvec\mid\Xvec,\zvec]$, (b) is the change-of-measure identity \eqref{eq:change_of_measure}, and (c) uses $r^*(\Xvec) \leq M$ (the square-integrability of $\yvec$ holding by Assumption~\ref{ass:regularity}(iv)). For the correction term at the true nuisances, $\bm{\xi}_{n_{\text{\upshape S}}}(\gvec^*, r^*)$ is an i.i.d.\ average with mean zero (Lemma~\ref{lem: orthogonality dr}) and bounded summand variance,
$$
\E_{\text{\upshape S}}\left[\left\lVert r^*(\Xvec)\,\pivec \right\rVert^2\right]\,=\,\E_{\text{\upshape S}}\left[r^*(\Xvec)^2 \left\lVert \pivec \right\rVert^2\right]\,\stackrel{\text{(a)}}{\leq}\,M^2 \tilde{\sigma}^2 C^2\,<\,\infty,
$$
where (a) uses $r^*(\Xvec) \leq M$, $\lVert\Cov(\yvec\mid\Xvec,\zvec)\rVert \leq \tilde{\sigma}^2$, and $\lVert\Xvec\rVert \leq C$,
so $\bm{\xi}_{n_{\text{\upshape S}}}(\gvec^*, r^*) = O_P\big(n_{\text{\upshape S}}^{-1/2}\big) = o_P(1)$; consequently the linear term $\bm{\xi}_{n_{\text{\upshape S}}}(\gvec^*, r^*)^{\top}\betavec$ vanishes in probability for each fixed $\betavec \in \mB$, and $\bar{\ell}_{\text{\upshape DR}}(\gvec^*, r^*; \betavec) \to_P L(\betavec)$ for each $\betavec$. (At the estimated nuisances, Lemma~\ref{lem: orthogonality dr} adds only an $o_P\big(n_{\text{\upshape S}}^{-1/2}\big)$ perturbation to $\tauvec_{\text{\upshape DR}}$, so the corresponding statements hold there as well; we only need the true-nuisance version in what follows.) By Theorem 2.7 of \cite{newey1994large}, applied exactly as in the proof of Lemma~\ref{lem: empirical-scores} to the convex objective $\bar{\ell}_{\text{\upshape DR}}(\gvec^*, r^*; \cdot)$ with unique population minimizer $\betavec^*_{\text{\upshape T}}$, there is a random sequence $\{\check{\betavec}_n\}$ (the subscript $n$ indexing the joint asymptotic regime in which $n_{\text{\upshape S}}, n_{\text{\upshape T}} \to \infty$) that solves $\min_{\betavec \in \mB}\bar{\ell}_{\text{\upshape DR}}(\gvec^*, r^*; \betavec)$ and converges to $\betavec^*_{\text{\upshape T}}$ with probability one; in particular, $\bm{\Psi}_{\text{\upshape DR}}(\gvec^*, r^*; \check{\betavec}_n) = \bm{0}$.

The transformation between the $(\hat{\gvec}, \hat{r})$-score and the $(\gvec^*, r^*)$-score is identical in logic to \eqref{eqn:transform} and \eqref{eqn:transform-shift}: since $\bm{\Psi}_{\text{\upshape DR}}(\hat{\gvec}, \hat{r}; \betavec) - \bm{\Psi}_{\text{\upshape DR}}(\gvec^*, r^*; \betavec) = \tauvec_{\text{\upshape DR}}(\gvec^*, r^*) - \tauvec_{\text{\upshape DR}}(\hat{\gvec}, \hat{r})$ does not depend on $\betavec$, Lemma~\ref{lem: orthogonality dr} yields
\begin{align}
\label{eqn:transform-dr}
    \sqrt{n_{\text{\upshape S}}}\,\bm{\Psi}_{\text{\upshape DR}}(\hat{\gvec}, \hat{r}; \check{\betavec}_n)\,=\,\sqrt{n_{\text{\upshape S}}}\,\bm{\Psi}_{\text{\upshape DR}}(\gvec^*, r^*; \check{\betavec}_n) + \sqrt{n_{\text{\upshape S}}}\left(\tauvec_{\text{\upshape DR}}(\gvec^*, r^*) - \tauvec_{\text{\upshape DR}}(\hat{\gvec}, \hat{r})\right)\,=\,o_P(1).
\end{align}
Hence $\inf_{\betavec \in \mB}\big\lVert \bm{\Psi}_{\text{\upshape DR}}(\hat{\gvec}, \hat{r}; \betavec)\big\rVert = o_P\big(1/\sqrt{n_{\text{\upshape S}}}\big)$, which implies $\bm{\Psi}_{\text{\upshape DR}}\big(\hat{\gvec}, \hat{r}; \hat{\betavec}_{\text{\upshape DR}}\big) = o_P\big(1/\sqrt{n_{\text{\upshape S}}}\big)$ by the definition of $\hat{\betavec}_{\text{\upshape DR}}$. Using the argument in \eqref{eqn:transform-dr} again, $\bm{\Psi}_{\text{\upshape DR}}\big(\gvec^*, r^*; \hat{\betavec}_{\text{\upshape DR}}\big) = o_P\big(1/\sqrt{n_{\text{\upshape S}}}\big)$. This concludes the proof of the lemma. \hfill $\blacksquare$

The final lemma is the counterpart of Lemma~\ref{lem: consistency}.

\begin{lemma}[\normalfont \textbf{Consistency of the Doubly Robust Estimator}]
\label{lem: consistency dr} $\hat{\betavec}_{\text{\upshape DR}} \to_P \betavec^*_{\text{\upshape T}}$.
\end{lemma}

\textit{Proof of Lemma~\ref{lem: consistency dr}.} The proof of Lemma~\ref{lem: consistency} carries over with $\bar{\ell}$, $\Psivec$, $\Pe_n$, and $\betavec^*$ replaced by $\bar{\ell}_{\text{\upshape DR}}(\gvec^*, r^*; \cdot)$, $\bm{\Psi}_{\text{\upshape DR}}(\gvec^*, r^*; \cdot)$, $\Pe_{n_{\text{\upshape T}}}$, and $\betavec^*_{\text{\upshape T}}$; we verify the three places where the two-sample structure enters and cite the rest. Fix arbitrary $\epsilon_1 < \epsilon_2$ such that $\overline{B(\betavec^*_{\text{\upshape T}}, \epsilon_2)} \subset \mB$. First, the uniform convergence step: by \cite{pollard1991asymptotics}, the convexity of the target-sample objective $\Pe_{n_{\text{\upshape T}}}\big[b(\Xvec\betavec) - \gvec^*(\Xvec,\zvec)^{\top}\Xvec\betavec\big]$ (established in the proof of Lemma~\ref{lem: empirical-scores dr}) and the compactness of $\overline{B(\betavec^*_{\text{\upshape T}}, \epsilon_2)}$ imply that it converges in probability, uniformly on $\overline{B(\betavec^*_{\text{\upshape T}}, \epsilon_2)}$, to $L(\betavec) = \E_{\text{\upshape T}}\big[b(\Xvec\betavec) - \gvec^*(\Xvec,\zvec)^{\top}\Xvec\betavec\big]$, exactly as in the proof of Lemma~\ref{lem: consistency}, now under $Q_{\text{\upshape T}}$ and with the moment bound on $\gvec^*$ from the proof of Lemma~\ref{lem: empirical-scores dr}. The full empirical objective $\bar{\ell}_{\text{\upshape DR}}(\gvec^*, r^*; \betavec)$ differs from this target-sample objective by the $\betavec$-free linear term $\bm{\xi}_{n_{\text{\upshape S}}}(\gvec^*, r^*)^{\top}\betavec$, which is bounded in norm by $\lVert \bm{\xi}_{n_{\text{\upshape S}}}(\gvec^*, r^*)\rVert\big(\lVert\betavec^*_{\text{\upshape T}}\rVert + \epsilon_2\big) = o_P(1)$ uniformly over the ball since $\bm{\xi}_{n_{\text{\upshape S}}}(\gvec^*, r^*) = O_P\big(n_{\text{\upshape S}}^{-1/2}\big)$; hence $\bar{\ell}_{\text{\upshape DR}}(\gvec^*, r^*; \cdot)$ also converges to $L$ uniformly. This $\betavec$-free perturbation is the only new ingredient in the uniform-convergence step. Second, the quadratic-growth lower bound (inequality (a) in the proof of Lemma~\ref{lem: consistency}): for $\betavec$ in the annulus and the corresponding intermediate point $\tilde{\betavec}$,
$$
(\betavec - \betavec^*_{\text{\upshape T}})^{\top}\, \E_{\text{\upshape T}}\left[\Xvec^{\top}\nabla^2 b\big(\Xvec\tilde{\betavec}\big)\Xvec\right](\betavec - \betavec^*_{\text{\upshape T}})\,\geq\,c_{\text{\upshape DR}}\left\lVert \betavec - \betavec^*_{\text{\upshape T}}\right\rVert^2,
$$
with $c_{\text{\upshape DR}} := \inf\big\{\lambda_{\min}\big(\nabla^2 b(\Xvec\betavec)\big) : \Xvec \in \Breve{\mX},\, \betavec \in \overline{B(\betavec^*_{\text{\upshape T}}, \epsilon_2)}\big\} \cdot \lambda_{\min}\big(\E_{\text{\upshape T}}[\Xvec^{\top}\Xvec]\big) > 0$ by the compactness/contradiction argument cited in the proof of Lemma~\ref{lem: consistency} and Assumption~\ref{ass:covariate_shift}(iv); as in the proof of Lemma~\ref{lem: consistency shift}, the constant runs through the target second-moment matrix $\E_{\text{\upshape T}}[\Xvec^{\top}\Xvec]$. Third, the extension from the annulus to $\mB \setminus \overline{B(\betavec^*_{\text{\upshape T}}, \epsilon_2)}$ via the integral form of the intermediate value theorem uses only the positive semidefiniteness of the empirical Hessian, which here is $\Pe_{n_{\text{\upshape T}}}\big[\Xvec^{\top}\nabla^2 b(\Xvec\betavec)\Xvec\big] \succeq 0$---unaffected by the linear correction term (fact (a)). All remaining steps---the subgradient inequality, the gradient lower bound \eqref{eqn:gradient-bound}, and the conclusion---carry over verbatim with the stated substitutions, the conclusion now invoking $\big\lVert \bm{\Psi}_{\text{\upshape DR}}\big(\gvec^*, r^*; \hat{\betavec}_{\text{\upshape DR}}\big)\big\rVert = o_P\big(1/\sqrt{n_{\text{\upshape S}}}\big)$ from Lemma~\ref{lem: empirical-scores dr}. This concludes the proof of the lemma. \hfill $\blacksquare$

With these three lemmas, the completion of the proof follows the Donsker/Taylor step in Section~\ref{apdx: proof of normality}; we indicate the substitutions and write out in full the two-sample scaling and the two-sample central limit theorem, which are new. Write $\mathbf{f}_{\betavec}(\Xvec) := \Xvec^{\top}\nabla b(\Xvec\betavec)$, the only $\betavec$-dependent part of the score (fact (a)). The local Lipschitz bound in the proof of Theorem~\ref{thm:normality} applies verbatim: for $\betavec_1, \betavec_2 \in B(\betavec^*_{\text{\upshape T}}, \epsilon)$, $\lVert \mathbf{f}_{\betavec_1}(\Xvec) - \mathbf{f}_{\betavec_2}(\Xvec)\rVert \leq C_1 \lVert\Xvec\rVert \lVert\betavec_1 - \betavec_2\rVert$ with $C_1 := \sup_{\betavec \in \overline{B(\betavec^*_{\text{\upshape T}}, \epsilon)},\, \Xvec \in \Breve{\mX}}\lVert \nabla^2 b(\Xvec\betavec)\rVert < \infty$; unlike the weighted score of Section~\ref{apdx: proof normality shift}, no factor $M$ arises, because the target-sample term of \eqref{eq:score_dr_body} carries no weight. Since $\E_{\text{\upshape T}}\big[C_1^2\lVert\Xvec\rVert^2\big] < \infty$ by the boundedness of $\Xvec$ under the target population, Example 19.7 and Theorem 19.5 in \cite{van2000asymptotic} imply that $\left\{\mathbf{f}_{\betavec} : \betavec \in B(\betavec^*_{\text{\upshape T}}, \epsilon)\right\}$ is a Donsker class under $Q_{\text{\upshape T}}$. By Lemma~\ref{lem: consistency dr}, $\hat{\betavec}_{\text{\upshape DR}} \to_P \betavec^*_{\text{\upshape T}}$, so Theorem 19.9 in \cite{van2000asymptotic} gives $\mathbb{G}_{n_{\text{\upshape T}}}\mathbf{f}_{\hat{\betavec}_{\text{\upshape DR}}} - \mathbb{G}_{n_{\text{\upshape T}}}\mathbf{f}_{\betavec^*_{\text{\upshape T}}} \to_P 0$. The new two-sample wrinkle is the scaling: the asymptotic equicontinuity statement is in the $\sqrt{n_{\text{\upshape T}}}$ scale, while the theorem is normalized by $\sqrt{n_{\text{\upshape S}}}$, so we multiply by $\sqrt{n_{\text{\upshape S}}/n_{\text{\upshape T}}}$, which converges to $\sqrt{\gamma} < \infty$; hence
$$
\sqrt{n_{\text{\upshape S}}}\left(\Pe_{n_{\text{\upshape T}}} - \E_{\text{\upshape T}}\right)\left(\mathbf{f}_{\hat{\betavec}_{\text{\upshape DR}}} - \mathbf{f}_{\betavec^*_{\text{\upshape T}}}\right)\,=\,\sqrt{\frac{n_{\text{\upshape S}}}{n_{\text{\upshape T}}}}\left(\mathbb{G}_{n_{\text{\upshape T}}}\mathbf{f}_{\hat{\betavec}_{\text{\upshape DR}}} - \mathbb{G}_{n_{\text{\upshape T}}}\mathbf{f}_{\betavec^*_{\text{\upshape T}}}\right)\,=\,o_P(1).
$$
Define the population score map $\bar{\bm{\Psi}}(\betavec) := \E_{\text{\upshape T}}\big[\Xvec^{\top}\big(\nabla b(\Xvec\betavec) - \gvec^*(\Xvec,\zvec)\big)\big] = \nabla L(\betavec)$; note that differences of $\bm{\Psi}_{\text{\upshape DR}}(\gvec^*, r^*; \cdot)$ across two values of $\betavec$ equal $\Pe_{n_{\text{\upshape T}}}$-differences of $\mathbf{f}_{\betavec}$, since the $\betavec$-free parts cancel. Thus, exactly as in Section~\ref{apdx: proof of normality},
\begin{align*}
\sqrt{n_{\text{\upshape S}}}\left(\bar{\bm{\Psi}}\big(\hat{\betavec}_{\text{\upshape DR}}\big) - \bar{\bm{\Psi}}\big(\betavec^*_{\text{\upshape T}}\big)\right)\,&=\,-\sqrt{\frac{n_{\text{\upshape S}}}{n_{\text{\upshape T}}}}\left(\mathbb{G}_{n_{\text{\upshape T}}}\mathbf{f}_{\hat{\betavec}_{\text{\upshape DR}}} - \mathbb{G}_{n_{\text{\upshape T}}}\mathbf{f}_{\betavec^*_{\text{\upshape T}}}\right) + \sqrt{n_{\text{\upshape S}}}\,\bm{\Psi}_{\text{\upshape DR}}\big(\gvec^*, r^*; \hat{\betavec}_{\text{\upshape DR}}\big) - \sqrt{n_{\text{\upshape S}}}\,\bm{\Psi}_{\text{\upshape DR}}\big(\gvec^*, r^*; \betavec^*_{\text{\upshape T}}\big)\\
\,&=\,-\sqrt{n_{\text{\upshape S}}}\,\bm{\Psi}_{\text{\upshape DR}}\big(\gvec^*, r^*; \betavec^*_{\text{\upshape T}}\big) + o_P(1),
\end{align*}
where the last equality follows from Lemma~\ref{lem: empirical-scores dr}. Applying the Taylor expansion to the left-hand side, by continuous mapping, the Jacobian of $\bar{\bm{\Psi}}$ at $\betavec^*_{\text{\upshape T}}$ is $\Jvec_{\text{\upshape T}} = \E_{\text{\upshape T}}\big[\Xvec^{\top}\nabla^2 b\big(\Xvec\betavec^*_{\text{\upshape T}}\big)\Xvec\big]$, and we obtain the counterpart of \eqref{eqn:expansion}:
\begin{align}
\label{eqn:expansion-dr}
    \sqrt{n_{\text{\upshape S}}}\left[\Jvec_{\text{\upshape T}}\big(\hat{\betavec}_{\text{\upshape DR}} - \betavec^*_{\text{\upshape T}}\big) + o_P(1)\left\lVert \hat{\betavec}_{\text{\upshape DR}} - \betavec^*_{\text{\upshape T}}\right\rVert\right]\,=\,-\sqrt{n_{\text{\upshape S}}}\,\bm{\Psi}_{\text{\upshape DR}}\big(\gvec^*, r^*; \betavec^*_{\text{\upshape T}}\big) + o_P(1)\,=\,O_P(1).
\end{align}
Here $\Jvec_{\text{\upshape T}} = \nabla^2 L\big(\betavec^*_{\text{\upshape T}}\big)$ is invertible by the argument in the proof of Lemma~\ref{lem: empirical-scores dr}, and the $O_P(1)$ claim follows from the two-sample central limit theorem that we now establish; it is the final new step.

By the definitions of $\zetavec_{\text{\upshape T}}$ and $\pivec$ and fact (b),
\begin{align}
\label{eqn:clt-dr}
    \sqrt{n_{\text{\upshape S}}}\,\bm{\Psi}_{\text{\upshape DR}}\big(\gvec^*, r^*; \betavec^*_{\text{\upshape T}}\big)
\,=\,\sqrt{\frac{n_{\text{\upshape S}}}{n_{\text{\upshape T}}}}\cdot\sqrt{n_{\text{\upshape T}}}\,\Pe_{n_{\text{\upshape T}}}\left[\zetavec_{\text{\upshape T}}\right] \;+\; \sqrt{n_{\text{\upshape S}}}\,\Pe_{n_{\text{\upshape S}}}\left[r^*(\Xvec)\,\pivec\right],
\end{align}
a sum of two \emph{independent} terms, each an i.i.d.\ average centered at zero by Lemma~\ref{lem: orthogonality dr}---this is where the separate centering of the two components is indispensable. For the source term, the summands have finite covariance matrix
$$
\E_{\text{\upshape S}}\left[r^*(\Xvec)^2\,\pivec\pivec^{\top}\right],
$$
which is the $\rho = 1$ specialization of the last term of \eqref{eqn:omega-w-decomposition} and is finite since $r^*(\Xvec) \leq M$, $\lVert\Xvec\rVert \leq C$, and $\lVert\Cov(\yvec \mid \Xvec, \zvec)\rVert \leq \tilde{\sigma}^2$; hence by the central limit theorem,
$$
\sqrt{n_{\text{\upshape S}}}\,\Pe_{n_{\text{\upshape S}}}\left[r^*(\Xvec)\,\pivec\right]\,\rightsquigarrow\,N\left(\bm{0},\, \E_{\text{\upshape S}}\left[r^*(\Xvec)^2\,\pivec\pivec^{\top}\right]\right).
$$
For the target term, $\E_{\text{\upshape T}}\big[\zetavec_{\text{\upshape T}}\zetavec_{\text{\upshape T}}^{\top}\big]$ is finite---$\lVert\Xvec\rVert \leq C$ under the target population and $\E_{\text{\upshape T}}\lVert\gvec^*\rVert^2 \leq M\,\E_{\text{\upshape S}}\lVert\yvec\rVert^2 < \infty$ as noted in the proof of Lemma~\ref{lem: empirical-scores dr}---so $\sqrt{n_{\text{\upshape T}}}\,\Pe_{n_{\text{\upshape T}}}\left[\zetavec_{\text{\upshape T}}\right] \rightsquigarrow N\big(\bm{0}, \E_{\text{\upshape T}}\big[\zetavec_{\text{\upshape T}}\zetavec_{\text{\upshape T}}^{\top}\big]\big)$ by the central limit theorem under $Q_{\text{\upshape T}}$. If $\gamma > 0$, then $\sqrt{n_{\text{\upshape S}}/n_{\text{\upshape T}}} \to \sqrt{\gamma}$ and, by Slutsky's theorem, the first term of \eqref{eqn:clt-dr} converges weakly to $N\big(\bm{0}, \gamma\,\E_{\text{\upshape T}}\big[\zetavec_{\text{\upshape T}}\zetavec_{\text{\upshape T}}^{\top}\big]\big)$. If $\gamma = 0$, the first term is $\sqrt{n_{\text{\upshape S}}/n_{\text{\upshape T}}}\cdot O_P(1) \to_P 0$, i.e., it degenerates to zero in probability; this case is still covered, the limit below holding with the first term of $\bm{\Omega}_{\gamma}$ equal to zero. Because the two terms of \eqref{eqn:clt-dr} are independent for every $(n_{\text{\upshape S}}, n_{\text{\upshape T}})$, they converge jointly to independent limits, and their sum satisfies
$$
\sqrt{n_{\text{\upshape S}}}\,\bm{\Psi}_{\text{\upshape DR}}\big(\gvec^*, r^*; \betavec^*_{\text{\upshape T}}\big)\,\rightsquigarrow\,N\left(\bm{0},\, \gamma\,\E_{\text{\upshape T}}\left[\zetavec_{\text{\upshape T}}\zetavec_{\text{\upshape T}}^{\top}\right] + \E_{\text{\upshape S}}\left[r^*(\Xvec)^2\,\pivec\pivec^{\top}\right]\right)\,=\,N\big(\bm{0}, \bm{\Omega}_{\gamma}\big),
$$
for all $\gamma \in [0, \infty)$; in particular it is $O_P(1)$, as claimed in \eqref{eqn:expansion-dr}.

Since $\Jvec_{\text{\upshape T}}$ is invertible, \eqref{eqn:expansion-dr} implies $\sqrt{n_{\text{\upshape S}}}\big\lVert \hat{\betavec}_{\text{\upshape DR}} - \betavec^*_{\text{\upshape T}}\big\rVert = O_P(1)$, and applying \eqref{eqn:expansion-dr} again,
$$
\sqrt{n_{\text{\upshape S}}}\big(\hat{\betavec}_{\text{\upshape DR}} - \betavec^*_{\text{\upshape T}}\big)\,=\,-\Jvec_{\text{\upshape T}}^{-1}\sqrt{n_{\text{\upshape S}}}\,\bm{\Psi}_{\text{\upshape DR}}\big(\gvec^*, r^*; \betavec^*_{\text{\upshape T}}\big) + o_P(1)\,\rightsquigarrow\,N\big(\bm{0},\, \Jvec_{\text{\upshape T}}^{-1}\bm{\Omega}_{\gamma}\Jvec_{\text{\upshape T}}^{-1}\big)\,=\,N\big(\bm{0}, \Sigmavec_{\text{\upshape DR}}\big),
$$
as desired. This completes the proof of Theorem~\ref{thm:normality_dr_body}.

\end{APPENDICES}

\end{document}